\pdfoutput=1

\documentclass[11pt]{article}

\usepackage[preprint]{for_arxiv}

\usepackage{times}
\usepackage{latexsym}

\usepackage[OT2,T1]{fontenc}
\DeclareSymbolFont{cyrletters}{OT2}{wncyr}{m}{n}
\DeclareMathSymbol{\Sha}{\mathalpha}{cyrletters}{"58}
\DeclareMathSymbol{\sha}{\mathalpha}{cyrletters}{"78}

\usepackage[utf8]{inputenc}

\usepackage{microtype}

\usepackage{inconsolata}

\usepackage{graphicx}

\usepackage{booktabs}
\usepackage{comment}
\usepackage{amsmath}
\usepackage{amssymb}
\usepackage{amsfonts}
\usepackage{amsthm}
\usepackage{mathtools}
\usepackage{makecell}

\newtheorem{theorem}{Theorem}
\newtheorem{lemma}{Lemma}
\newtheorem{proposition}{Proposition}
\newtheorem{definition}{Definition}
\newtheorem{corollary}{Corollary}

\title{Theoretical Analysis of Hierarchical Language Recognition and Generation by Transformers without Positional Encoding}

\author{
    Daichi Hayakawa \quad Issei Sato\\
    The University of Tokyo \\
    \texttt{\{hayakawadaichi001, sato\}@g.ecc.u-tokyo.ac.jp} 
}

\begin{document}
\twocolumn[\begin{@twocolumnfalse}
\maketitle
\begin{abstract}
In this study, we provide constructive proof that Transformers can recognize and generate hierarchical language efficiently with respect to model size, even without the need for a specific positional encoding.
Specifically, we show that causal masking and a starting token enable Transformers to compute positional information and depth within hierarchical structures.
We demonstrate that Transformers without positional encoding can generate hierarchical languages. 
Furthermore, we suggest that explicit positional encoding might have a detrimental effect on generalization with respect to sequence length.
\end{abstract}
\end{@twocolumnfalse}]

\section{Introduction}\label{sec:introduction}
Transformer-based models have achieved significant success in natural language processing. The empirical success of Transformers has drawn attention to the theoretical understanding of the problem classes that Transformers can solve. Although natural languages and programming languages possess hierarchical structures, \citet{tran-etal-2018-importance}, \citet{petty2021transformersgeneralizelinearly}, and \citet{mueller-etal-2022-coloring} stated that Transformers have different inductive biases than humans and often face difficulties processing such hierarchical languages. On the other hand, \citet{murty-etal-2023-grokking} empirically showed that Transformers have the potential to learn hierarchical structures through grokking.

To investigate the expressive capacity of Transformers, several studies (\citealp{hahn-2020-theoretical}, \citealp{bhattamishra-etal-2020-ability}, \citealp{ebrahimi-etal-2020-self}, \citealp{yao-etal-2021-self}, \citealp{chiang-cholak-2022-overcoming}, and \citealp{NEURIPS2023_79ba1b82}) have formulated such challenges as recognition/generation tasks of formal languages such as the parity language, the $\texttt{Dyck}_k$ language, and the $\texttt{Shuffle-Dyck}_k$ language.
The parity language is a language over an alphabet consisting of only two characters, $0$ and $1$. A string belongs to the parity language when it has an odd number of $1$s. 
While the parity language is quite simple, it has a Kleene closure that characterizes regular languages.
In contrast, the $\texttt{Dyck}_k$ language is a language over an alphabet consisting of $k$ types of brackets. Intuitively, it includes properly balanced strings. The $\texttt{Shuffle-Dyck}_k$ language is a shuffle of multiple $\texttt{Dyck}_1$ languages defined over different bracket-pairs. For example, regarding $\texttt{Dyck}_2$ and $\texttt{Shuffle-Dyck}_2$ over $\{\texttt{``(", ``)", ``[", ``]"}\}$, $\texttt{``[]()"}$ and $\texttt{``([()])"}$ belong to $\texttt{Dyck}_2$, while $\texttt{``[(])"}$ and $\texttt{``([(]))}"$ belong to $\texttt{Shuffle-Dyck}_2$ not to $\texttt{Dyck}_2$. Despite their simplicity, these languages are important because they provide a simplified framework for investigating the ability to comprehend hierarchical structures, which are found in both natural and programming languages, as well as the capability to process these structures in parallel.

\citet{hahn-2020-theoretical} pointed out that Lipschitz-bounded Transformers cannot solve recognition and generation tasks of Dyck languages for arbitrary lengths, implying that Transformers do not have the ability to grasp hierarchical structures. However, \citet{yao-etal-2021-self} provided a proof that Transformer with specific absolute positional encoding can generate $\texttt{Dyck}_{k}$ and recognize $\texttt{Dyck}_{k, D}$, where $\texttt{Dyck}_{k, D}$ is a subset of $\texttt{Dyck}_{k}$ but the maximum nesting depth is bounded to $D$. Furthermore, \citet{NEURIPS2023_79ba1b82} provided an existence proof that a $2$-layer $O(k^2D^2)$-width Transformer can process $\texttt{Dyck}_{k, D}$. 
These theoretical results raise the question: \begin{quote} Why can Transformers with smaller widths and without specific absolute positional encoding experimentally perform well on processing the $\texttt{Dyck}_k$ language?\end{quote}

\begin{table*}
  \centering
  \begin{tabular}{llccc}
    \hline 
    \textbf{Method} & \textbf{Language} & \textbf{Width} & \textbf{Positional Encoding} \\
    \hline \hline
    \textbf{Recognition task} \\
    \hline
    \citet{bhattamishra-etal-2020-ability}  & $\texttt{Shuffle-Dyck}_{k}$     & $O(k)$        & None \\
    \citet{yao-etal-2021-self}              & $\texttt{Dyck}_{k, D}$          & $O(\log k)$   & $i/n$ \\
    Ours                                    & $\texttt{Dyck}_{k}$             & $O(\log k)$   & None \\
                                            & $\texttt{Shuffle-Dyck}_{k}$       & $O(\log k)$   & None \\
    \hline
    \textbf{Generation task} \\
    \hline
    \citet{yao-etal-2021-self}              & $\texttt{Dyck}_{k, D}$             & $O(\log k)$   & $i/n$ \\
        & $\texttt{Dyck}_{k}$             & $O(\log k)$   & $i/n, i/n^3, n$ \\
    \citet{NEURIPS2023_79ba1b82}             & $\texttt{Dyck}_{k, D}$      & $O(k^2D^2)$   & None \\
    Ours                                    & $\texttt{Dyck}_{k}$         & $O(\log k)$   & None \\
                                            & $\texttt{Shuffle-Dyck}_{k}$      & $O(k)$        & None \\
    \hline
  \end{tabular}
  \caption{Comparison of proposed method to previous studies. Note that \textbf{Width} represents the width of the Transformer block, excluding the width of the heads for each task. Since a task-specific head is a mapping from $\mathbb{R}^{d_{\mathrm{model}}}$ to $\mathbb{R}$ for a recognition task and to $\mathbb{R}^K$ for a generation task, each has an  $\Omega({d_{\mathrm{model}}})$-width and an $\Omega(\max({d_{\mathrm{model}}}, K))$-width head, respectively. In all cases in this table, the width of a task-specific head is $O({d_{\mathrm{model}}})$ for a recognition task and $O(\max({d_{\mathrm{model}}}, K))$ for a generation task.}
  \label{tab: diff of results}
\end{table*}

In contrast to the approaches of \citet{bhattamishra-etal-2020-ability}, \citet{yao-etal-2021-self}, and \citet{NEURIPS2023_79ba1b82}, our theoretical analysis offers two advantages: (i) it reduces the linear or super-linear dependency of the number of bracket types $k$ and the maximum depth $D$ on the network width, and (ii) it does not rely on specific absolute positional encoding. Table \ref{tab: diff of results} outlines these differences and highlights the strengths of our approach in comparison.

Our contributions are summarized as follows.
\begin{enumerate}
\item We provide constructive proofs that with a starting token, causal Transformers with a constant number of layers and $O(\log k)$ width have the ability to recognize the $\texttt{Dyck}_k$ and $\texttt{Shuffle-Dyck}_k$ languages and to generate the $\texttt{Dyck}_k$ language. Moreover, we also present a proof that those with a constant number of layers and $O(k)$ width have the ability to generate the $\texttt{Shuffle-Dyck}_k$ language. Note that the network is followed by a fully-connected layer whose output dimension is $\mathbb{R}$ for a recognition task and $K$ for a generation task, where $\mathbb{R}^K$ is the vocabulary size.

\item We give a constructive proof that Transformers can still create a signal that can serve similarly to a starting token by only leveraging causal masking under an additional assumption.
\end{enumerate}

\section{Related Work}\label{sec:related work}

Since the emergence of Transformer \citep{NIPS2017_3f5ee243}, a wide range of theoretical analyses have been conducted on its expressive capacity. Some of these analyses have focused on formal language recognition and generation tasks, particularly for the $\texttt{Dyck}_k$ language and the $\texttt{Shuffle-Dyck}_k$ language.

\citet{bhattamishra-etal-2020-ability} theoretically showed that a Transformer with a width of $O(k)$ can recognize the $\texttt{Shuffle-Dyck}_k$ language. In addition, \citet{yao-etal-2021-self} provided a constructive proof that by using specific absolute positional encoding $i/n$, where $n$ is the maximum length of the input string, and $i$ is the position of characters, a $(D+1)$-layer causal Transformer can recognize the $\texttt{Dyck}_{k, D}$ language. \citet{yao-etal-2021-self} also proved that using positional encoding $i/n, i/n^3, i$, a $2$-layer causal Transformer can generate the $\texttt{Dyck}_{k}$ language. Furthermore, \citet{NEURIPS2023_79ba1b82} proved that a $2$-layer Transformer network with a width of $O(k^2D^2)$ can generate $\texttt{Dyck}_{k, D}$.

\section{Preliminaries}\label{sec:Preliminaries}

\subsection{Dyck Languages}\label{subsec: preliminaries/Dyck Languages}

The $\texttt{Dyck}_k$ language is a context-free language over an alphabet consisting solely of $k$ types of bracket pairs $\{\langle_t, \rangle_t\}_{t=1}^k$ and includes strings with correctly nested brackets.

Despite its simplicity, \citet{CHOMSKY1959118} showed that any context-free language can be expressed as a homomorphism of the intersection of the Dyck language and a regular language, suggesting that the Dyck language has an essence of context-free languages. Therefore, we aim to analyze the recognition and generation capacity of Transformers with respect to $\texttt{Dyck}_k$ and its variant, $\texttt{Shuffle-Dyck}_k$.

In this paper, we consider languages with two special tokens, $\texttt{<bos>}$ and $\texttt{<eos>}$, which stand for $``\text{beginning-of-sentence}"$ and $``\text{end-of-sentence}"$ respectively. 
In language models, $\texttt{<bos>}$ is typically inserted at the start, and $\texttt{<eos>}$ is used as a signal to stop generating output. Therefore, we define the $\texttt{Dyck}_k$ and  $\texttt{Shuffle-Dyck}_k$ languages for language models as follows:

\begin{definition}[$\texttt{Dyck}_k$ language for language models]\label{def: dyck_k for language models}
The $\texttt{Dyck}_k$ language for language models is a context-free language over an alphabet $\Sigma = \{\langle_t, \rangle_t\}_{t=1}^k \cup \{\texttt{<bos>}, \texttt{<eos>}\}$. The following context-free grammar generates $\texttt{Dyck}_k$ language:
\begin{align}
S &\rightarrow \texttt{<bos>} \, X \, \texttt{<eos>} ,\\
X &\rightarrow \varepsilon \, \mid \, \langle_1 \, X \, \rangle_1\,  X \, \mid \, \cdots \, \mid \langle_k \, X \, \rangle_k\,  X,
\end{align}
where $S$ and $\varepsilon$ are the starting symbol and empty string, respectively.
\end{definition}

\begin{definition}[$\texttt{Shuffle-Dyck}_k$ language for language models (informal)]\label{def: shuffle dyck_k for language models}
The $\texttt{Shuffle-Dyck}_k$ language for language models is defined as follows:
\begin{equation}
\left\{\texttt{<bos>}w\texttt{<eos>} | w \in \texttt{Shuffle-Dyck}_k\right\},
\end{equation}
where $\texttt{Shuffle-Dyck}_k$ is a language over an alphabet $\Sigma = \{\langle_t, \rangle_t\}_{t=1}^k$ and is defined as a shuffle of $k$ multiple $\texttt{Dyck}_1$ --- $\texttt{Dyck}_1^1, \cdots, \texttt{Dyck}_1^k$ ---, where $\texttt{Dyck}_1^t$ is the $\texttt{Dyck}_1$ language over an alphabet $\{\langle_t, \rangle_t\}$. 
A formal definition of $\texttt{Shuffle-Dyck}_k$ is provided in Appendix \ref{app: preliminaries, shuffle-dyck}.

\end{definition}

For example, $``\langle_1 \langle_2 \rangle_1 \rangle_2"$ does not belong to $\texttt{Dyck}_2$ but to $\texttt{Shuffle-Dyck}_2$. $\texttt{Shuffle-Dyck}_k$ can be recognized by $k$-counter machines; thus, this language provides insights into the ability to process $k$ hierarchical structures in parallel.

We also define a prefix for languages and the depth of a prefix in the $\texttt{Dyck}_k$ language as follows:

\begin{definition}[Prefix for language]\label{def: prefix for language}
A string $w \in \Sigma^*$ is a prefix for language $\mathcal{L}$ if there exists $u \in \Sigma^*$ such that $wu \in \mathcal{L}$. In addition, denote by $w \in \operatorname{Pre}(\mathcal{L})$ that $w$ is a prefix for $\mathcal{L}$. 
\end{definition}

Hereafter, denote an input string of length $n+1$ by $w_{0:n}$ and the prefix of length $i+1$ by $w_{0:i}$.

\begin{definition}[Depth of string]\label{def: depth of prefixes}
The depth of a prefix $w_{0:i} (=\texttt{<bos>} w_{1:i})$ in  $\texttt{Dyck}_k$ is defined as follows: 
\begin{equation}
\operatorname{d}(w_{0:i})=\#_{\langle}(w_{0:i}) - \#_{\rangle}(w_{0:i}) ,
\end{equation}
where $\#_\langle(w_{0:i})$ and $\#_\rangle(w_{0:i})$ represent the number of open brackets and closed brackets in $w_{0:i}$, respectively. Here, the differences in bracket types are ignored.
\end{definition}

Note that for any prefix $w_{0:i}$ for $\texttt{Dyck}_k$, the following three statements hold: (i) $w_{0:i}\,\langle$ is always a prefix, (ii) if $\operatorname{d}(w_{0:i}) = 0$, $w_{0:i}\,\rangle$ cannot be a prefix, and (iii) if $\operatorname{d}(w_{0:i}) \geq 1$, there exists only one type $t_{\mathrm{valid}}$ such that $w_{0:i} \,\rangle_{t_{\mathrm{valid}}}$ is a prefix. With respect to (iii), although such a closed bracket depends on $w_{0:i}$, denote it by $\rangle_{t_\mathrm{valid}}$ in an abusive manner. In addition, there can be more than one $\rangle_{t_\mathrm{valid}}$ in $\texttt{Shuffle-Dyck}_k$.

\subsection{Transformer architecture}\label{subsec: preliminaries/Transformer Architecture}
Transformer architecture takes an input string of length $n$ and converts each character into a $d_\mathrm{model}$-dimensional vector. Then, by applying Transformer blocks $\left(\mathbb{R}^{n\times d_{\mathrm{model}}} \rightarrow \mathbb{R}^{n\times d_{\mathrm{model}}}\right)$ for multiple times, an output of dimension $\mathbb{R}^{n\times d_{\mathrm{model}}}$ is obtained. Since $n$ is not fixed, we represent a Transformer as $\mathcal{T}: \Sigma^* \rightarrow \mathbb{R}^{*\times {d_{\mathrm{model}}}}$.

In this paper, we largely follow the Transformer architecture adopted in \citet{yao-etal-2021-self}; namely, we consider a Transformer architecture composed of multiple single-head Transformer blocks, each of which incorporates a self-attention layer and a feed-forward network layer. The major differences of the architecture adopted in \citet{yao-etal-2021-self} from the model proposed by \citet{NIPS2017_3f5ee243} are (i) \citet{yao-etal-2021-self} uses single-head attention instead of multi-head attention and (ii) \citet{yao-etal-2021-self} incorporates layer normalization \citep{ba2016layernormalization} right after the first linear transformation in the feed-forward network layer instead of after the attention layer and feed-forward network layer.

We adopt the architecture in \citet{yao-etal-2021-self} with a slight modification: we replace the standard layer normalization \citep{ba2016layernormalization} with the RMS layer normalization \citep{NEURIPS2019_1e8a1942}.  \citet{NEURIPS2019_1e8a1942} empirically showed that the RMS layer normalization reduces the training time compared to the conventional layer normalization while maintaining the same performance. The RMS layer normalization has been adopted in recent models such as Llama \citep{touvron2023llamaopenefficientfoundation} and Llama 2 \citep{touvron2023llama2openfoundation}. The details of the Transformer architecture are provided in Appendix \ref{app: preliminaries, transformer architecture}.

\subsection{Language recognition and generation}\label{subsec: preliminaries/Language Recognition and Generation}
In this paper, we mainly focus on two tasks: language recognition and generation. Here, we define language recognition and generation by Transformers. For each task, a fully-connected layer follows the network, and the output dimension is $\mathbb{R}$ for recognition tasks and $\mathbb{R}^K$ for generation tasks, which we call the recognizer head and generator head, respectively.

\begin{definition}[Language recognition by Transformers]\label{def:Language recognition by transformers}
A Transformer $\mathcal{T}: \Sigma^* \rightarrow \mathbb{R}^{*\times {d_{\mathrm{model}}}}$ recognizes a language $\mathcal{L}\subseteq \Sigma^*$ if there exists a fully-connected layer $f_{\mathrm{rec}} : \mathbb{R}^{d_{\mathrm{model}}} \rightarrow \mathbb{R}$ such that 
\begin{equation}
\operatorname{sgn}(f_{\mathrm{rec}}(\mathcal{T}(w_{0:n})_n)) = 
\begin{cases}
1 & \text{ if } w_{0:n} \in \mathcal{L} \\
-1 & \text{ if } w_{0:n} \notin \mathcal{L}
\end{cases},
\end{equation}
where $\operatorname{sgn}(\cdot)$ is a sign function.
\end{definition}

It is impossible to define language generation by Transformers by simply setting a threshold on the output probability of each string in a language because formal languages are typically infinite string sets.
Therefore, we first define language generation process and then define language generation by Transformers. This approach is similar to the methods in \citet{yao-etal-2021-self}, \citet{NEURIPS2023_79ba1b82} and \citet{svete-cotterell-2024-transformers}. Specifically, we define language generation process using the conditional categorical distribution as follows.

\begin{definition}[Language generation process]\label{def: language generation process}
A language generation process over an alphabet $\Sigma$ is a categorical distribution over $\Sigma$ conditioned by a string $w_{0:i}$. Specifically, denote the language generation process of a language $\mathcal{L}$ by $p_{\mathcal{L}}(w_{i+1} \mid w_{0:i})$.
\end{definition}

Note that language generation processes are well-defined: the following proposition holds. 
\begin{proposition}\label{prop:existence of language process}
For any language $\mathcal{L} \subset \Sigma^*$ over a finite alphabet $\Sigma$ and any probability distribution $p$ over $\mathcal{L}$, there exists a language generation process that produces the given probability distribution $p$. In other words, there exists a language generation process $p_\mathcal{L}(w_{i+1} \mid \texttt{<bos>} w_{1:i})$ such that for any string $w_{1:n} \in \mathcal{L}$, 
\begin{equation}
p(w_{1:n}) = p_\mathcal{L}(\texttt{<bos>} w_{1:n}\texttt{<eos>}),
\end{equation}
where
\begin{equation}
\begin{aligned}
p_\mathcal{L}&(\texttt{<bos>} w_{1:n}\texttt{<eos>}) \\
&= p_\mathcal{L}(\texttt{<bos>}) \\
& \quad \cdot \left(\prod_{i=1}^n p_\mathcal{L}(w_{i}\mid \texttt{<bos>} w_{1:i-1})\right) \\
& \quad \cdot p_\mathcal{L}(\texttt{<eos>}\mid \texttt{<bos>} w_{1:n}).
\end{aligned}
\end{equation}
\begin{proof}
The proof is provided in Appendix \ref{app: proof of prop 1}.
\end{proof}
\end{proposition}

Then, we define the language generation by Transformers. We largely follow the definition in \citet{yao-etal-2021-self}, which defines it as whether the probability $p(w_i\mid w_{1:i-1})$ exceeds a certain threshold for any string $w_{1:n} \in \mathcal{L}$ and $i \in [n] (= \{1, \cdots, n\})$. However, we make this definition more stringent: we assume the existence of a true distribution and define it as the ability to output this distribution. This is because one of the most important properties of language models is the ability to generate diverse but natural sentences by assigning appropriate probability to consistent sequences. This approach is similar to \citet{NEURIPS2023_79ba1b82} and \citet{svete-cotterell-2024-transformers}.

In general, Transformer-based language models transform the last token output with a fully-connected layer $f_{\mathrm{gen}} : \mathbb{R}^{d_{\mathrm{model}}} \rightarrow \mathbb{R}^K$. Then, the vector is converted into a probability vector with softmax function $: \mathbb{R}^K \rightarrow \Delta^{K-1}$, where $\Delta^{K-1}\left(\subset \mathbb{R}^K\right)$ is a probability simplex. Here, the softmax function transforms each element into a value in the range of $(0, 1)$, which makes it impossible to represent a probability of $0$ or $1$ exactly. Therefore, we define the realization of the language generation process by Transformers as the ability to approximate the language generation process with arbitrary precision as follows.

\begin{definition}[Realization of language generation process by Transformers]
A Transformer $\mathcal{T}: \Sigma^* \rightarrow \mathbb{R}^{*\times d_{\mathrm{model}}}$ realizes a language generation process $p_\mathcal{L}(w_{i+1} \mid w_{0:i})$ if for any $\epsilon>0$ there exists a fully-connected layer $f_{\mathrm{gen}} : \mathbb{R}^{d_{\mathrm{model}}} \rightarrow \mathbb{R}^K$ such that if $p_\mathcal{L}(w_{0:i}) > 0$ then 
\begin{equation}
\operatorname{TV}\left(p_\mathcal{T}(w_{i+1} \,|\, w_{0:i}),  p_\mathcal{L}(w_{i+1}\, |\, w_{0:i})\right) < \epsilon, 
\end{equation}
where $p_\mathcal{T}(w_{i+1} \,|\, w_{0:i})$ is the categorical distribution based on the output of Transformer and $\operatorname{TV}(\cdot, \cdot)$ is the total variation distance. Specifically,
\begin{equation}
p_\mathcal{T}(w_{i+1} \,|\, w_{0:i}) = \mathbb{S}(f_{\mathrm{gen}}(\mathcal{T}(w_{0:i})_i)),
\end{equation}
where $\mathbb{S}(\cdot)$ is a softmax function and the total variation distance between two $K$-dimensional categorical distributions $p = (p_1, \cdots, p_K)$ and $p^\prime = (p_1^\prime, \cdots, p_K^\prime)$ is expressed as follows:
\begin{equation}
\operatorname{TV}(p, p^\prime) = \frac{1}{2}\sum_{l=1}^K \left|p_l - p^\prime_l\right|
\end{equation}
\end{definition}

Next, we define the language generation process of $\texttt{Dyck}_k$. Note that this definition generalizes the definition in \citet{hewitt-etal-2020-rnns} and \citet{NEURIPS2023_79ba1b82}: they treat all types of brackets in a symmetric way, while we slightly generalize the approach to be able to assign different probabilities.

\begin{definition}[$\texttt{Dyck}_k$ language generation process]\label{def: dyck_k language generation process}
A language generation process $p(w_{i+1} \mid w_{0:i})$ over an alphabet $\Sigma = \{\langle_t, \rangle_t\}_{t=1}^k \cup \{\texttt{<bos>}, \texttt{<eos>}\}$ is called the $\texttt{Dyck}_k$ language generation process if
\begin{align}
&p(w_0 = ``\texttt{<bos>}" \mid \varepsilon) = 1 , \\
&\begin{aligned}
& p(w_{i+1} \mid w_{0:i} ) \\
&\quad =  \begin{cases}
p_{0}(w_{i+1}) & \text{if } \operatorname{d}(w_{0:i}) = 0 \\
p_{1}(w_{i+1}) & \text{if } \operatorname{d}(w_{0:i}) \geq 1 \\
\end{cases},
\end{aligned}
\end{align}
where 
\begin{equation}
\begin{aligned}
&p_{0}(w_{i+1}) = 
\begin{cases}
r \pi_t & \text{if } w_{i+1} = ``\langle_t" \\
1-r & \text{if } w_{i+1} = ``\texttt{<eos>}" \\
0 & \text{otherwise}
\end{cases}, \\
&p_{1}(w_{i+1}) = 
\begin{cases}
q \pi_t & \text{if } w_{i+1} = ``\langle_t" \\
1-q & \text{if } w_{i+1} = ``\rangle_{t_{\mathrm{valid}}}" \\
0 & \text{otherwise}
\end{cases} ,
\end{aligned}
\end{equation}
$q, r \in (0, 1)$, $\boldsymbol{\pi} \in \Delta^{k-1}$.

Hereafter, we explicitly write the Dyck language generation process parameterized by $q, r, \boldsymbol{\pi}$ as $p_{\texttt{Dyck}_k}(\cdot; q, r, \boldsymbol{\pi})$.
\end{definition}

Note that the $\texttt{Dyck}_k$ language generation process defined above corresponds appropriately with the $\texttt{Dyck}_k$ language as described below.

\begin{proposition}\label{prop: correspondence of dyck_k language and language generation process}
For any length $n$ and $\texttt{Dyck}_k$ language generation process $p_{\texttt{Dyck}_k}(\cdot; q, r, \boldsymbol{\pi})$, there exists $\epsilon_n$ such that if $\boldsymbol{\pi} > 0$ then 
\begin{equation}
\begin{aligned}
&p_{\texttt{Dyck}_k}(\texttt{<bos>} w_{1:n} \texttt{<eos>};q, r, \boldsymbol{\pi}) \\
&\begin{cases}
\geq \epsilon_n & \text{if } w_{1:n} \in \texttt{Dyck}_k \\
= 0 & \text{if } w_{1:n} \notin \texttt{Dyck}_k \\
\end{cases}
\end{aligned}
\end{equation}
holds. 

\begin{proof}
The Proof is provided in Appendix \ref{app: proof of prop 2}.
\end{proof}
\end{proposition}

We also define the language generation process of $\texttt{Shuffle-Dyck}_k$ in a similar way. The details are provided in Appendix \ref{app: preliminaries, shuffle-dyck}.

\section{Theoretical Results}\label{sec:Theoretical Results}
In this section, we show our theoretical results. 

\begin{theorem}[Transformers with starting token, $\texttt{Dyck}_k$ recognition]\label{theorem: transformers with bos recognize dyck_k}

For all $k$, there exists a 5-layer $O(\log k)$-width causal Transformer without positional encoding that recognizes the $\texttt{Dyck}_k$ language. Each layer incorporates both the residual connection and the layer normalization.
This network is followed by a fully-connected layer and a sign function to output an acceptance signal.

\begin{proof}[Proof sketch]
A Transformer network that recognizes the $\texttt{Dyck}_k$ language can be constructed by performing the following operations in each layer. Note that $w_{0:i}$ corresponds to $\texttt{<bos>} w_{1:i}$. 
First, we compute positional and depth information using the BOS token. Then, using the information, we check whether the following two conditions are simultaneously satisfied: (i) $w_{1:i}$ is a prefix of the $\texttt{Dyck}_k$ language and (ii) the depth of $w_{0:i}$ is $0$.

\begin{description}
    \item[First layer] creates pseudo positional encoding $(\cos\phi(i), \sin\phi(i))$ at position $i$, where $\phi(i) = \tan^{-1}(i/\exp(a))$ and $a$ is an attention score on $\texttt{<bos>}$.
    
    \item[Second and third layers] count depth $\operatorname{d}(w_{0:i})$ and $\operatorname{d}(w_{0:i})+1$, respectively. This is because the depth of the closed bracket is smaller by $1$ than the corresponding open bracket. For instance, the depths calculated for $``\langle_1\rangle_1"$ are $1$ for $``\langle_1"$ and $0$ for $``\rangle_1"$. These computations are achieved by constructing a value matrix that outputs $1$ for open brackets and $-1$ for closed brackets in a specific dimension. 

    \item[Fourth layer] makes each closed bracket assign attention to the nearest depth-matched open bracket, using the positional and depth information calculated in the first, second, and third layers. Then the following  propositional variable $Q(w_{0:i})$ is computed: 
    \begin{equation}
    \begin{aligned}
    &Q(w_{0:i}) =\begin{cases}
        \texttt{True} & \text{if } w_{1:i} \in\operatorname{Pre}(\texttt{Dyck}_k) \\
        \texttt{False} & \text{otherwise}
    \end{cases}
    \end{aligned}.
    \end{equation}
    Note that $Q(w_{0:i})$ is guaranteed to return the correct value only when $i=0$ or $w_{1:{i-1}}$ is a prefix for $\texttt{Dyck}_k$. 
    
    \item[Fifth layer] calculates (i) whether $w_{1:n}$ is a prefix for $\texttt{Dyck}_k$ with $\bigwedge_{i=1}^n Q(w_{0:i})$ and (ii) whether $\operatorname{d}(w_{0:i}) = 0$ or not.
\end{description}

The subsequent fully-connected layer determines whether the string $w_{1:n}$ belongs to $\texttt{Dyck}_k$ by examining whether the two conditions calculated in the fifth layer are simultaneously satisfied.

The full proof is provided in Appendix \ref{app: dyck recognition with bos}.
\end{proof}
\end{theorem}

\begin{theorem}[Transformers with starting token, $\texttt{Dyck}_k$ generation]\label{theorem: transformers with bos generate dyck}
For all $k$, there exists a $3$-layer $O(\log k)$-width causal Transformer network without positional encoding that generates the $\texttt{Dyck}_k$ language. Each layer incorporates both the residual connection and the layer normalization.
This network is followed by a fully-connected layer and softmax layer to output the probability distribution. 

\begin{proof}[Proof sketch]
A Transformer network that generates the $\texttt{Dyck}_k$ language can be constructed by performing the following operations in each layer. The first and second layers do the same operations as those used in Theorem \ref{theorem: transformers with bos recognize dyck_k}.

\begin{description}
    \item[First layer] creates pseudo positional encoding $(\cos\phi(i), \sin\phi(i))$.
    
    \item[Second layer] counts depth $\operatorname{d}(w_{0:i})$.

    \item[Third layer] fetches a valid closed bracket if one exists; otherwise, a zero vector is fetched. This operation is achieved by placing attention on the largest $j$ among $\{0\} \cup \{j \mid \operatorname{d}(w_{0:j}) = \operatorname{d}(w_{0:i})\}$.
    
\end{description}

Then, the subsequent fully-connected layer and softmax operation output the next token distribution using the vector calculated in the third layer.

The full proof is provided in Appendix \ref{app: dyck language generation with bos}.
\end{proof}
\end{theorem}

\begin{proposition}[Transformers with starting token, $\texttt{Shuffle-Dyck}_k$ recognition]\label{proposition: transformers with bos recognize shuffle dyck_k}
For all $k$, there exists a $3$-layer $O(\log k)$-width causal Transformer without positional encoding that recognizes the $\texttt{Shuffle-Dyck}_k$ language. Each layer incorporates both the residual connection and the layer normalization. 
This network is followed by a fully-connected layer and a sign function to output an acceptance signal.
\begin{proof}
The proof is provided in Appendix \ref{app:proof of shuffle dyck recognition}.
\end{proof}
\end{proposition}

\begin{proposition}[Transformers with starting token, $\texttt{Shuffle-Dyck}_k$ generation]\label{proposition: transformers with bos generate shuffle dyck_k}
For all $k$, there exists a $3$-layer $O(k)$-width causal Transformer without positional encoding that generates the $\texttt{Shuffle-Dyck}_k$ language. Each layer incorporates both the residual connection and the layer normalization. 
This network is followed by a fully-connected layer and softmax layer to output the probability distribution. 

\begin{proof}
The proof is provided in Appendix \ref{app: proof of shuffle dyck generation}.
\end{proof}
\end{proposition}

\begin{proposition}\label{proposition: Networks with sub-polynomial width, cannot generate Shuffle Dyck}
There is no network whose width grows strictly slower than $k/\log k$ that generates $\texttt{Shuffle-Dyck}_k$; that is, if 
\begin{equation}
\lim_{k\rightarrow\infty} \frac{d_{\mathrm{model}}}{k/ \log k} = 0
\end{equation}
holds, then there exists $k_0$ such that for any $k \geq k_0$, networks with $d_\mathrm{model}$-width cannot generate $\texttt{Shuffle-Dyck}_k$. For example, $\sqrt{k}$ grows strictly slower than $k/\log k$.

\begin{proof}
The proof is provided in Appendix \ref{app: proof of prop 3}.
\end{proof}
\end{proposition}

Next, we show that even without $\texttt{<bos>}$, Transformers can recognize and generate $\texttt{Dyck}_k$ languages under certain conditions. The following proposition states that under relatively weak conditions, Transformers can generate a signal that serves a similar role to $\texttt{<bos>}$ in Theorems \ref{theorem: transformers with bos recognize dyck_k}, \ref{theorem: transformers with bos generate dyck}.

\begin{proposition}\label{proposition: create pseudo starting signal without bos}
Assume that there exists a linear subspace such that the embeddings are distinct from each other and have a constant $2$-norm. Then, there exists a Transformer block without a starting token that creates a pseudo starting signal $\hat{s}_i$ for any string $w_{1:n}$ whose first two tokens are different, where
\begin{equation}
\hat{s}_i = \begin{cases}
1 & \text{if } i = 1 \\
0 & \text{otherwise }
\end{cases}.
\end{equation}

Specifically, this block transforms the constants-padded vector $\hat{\mathbf{x}}_i$ as follows:
\begin{equation}
\hat{\mathbf{x}}_{i}  = 
\begin{bmatrix}
\mathbf{x}_{i} \\
\vdots \\
0 
\end{bmatrix} \mapsto \begin{bmatrix}
\mathbf{x}_{i} \\
\vdots \\
\hat{s}_i
\end{bmatrix}.
\end{equation}
\begin{proof}
The proof is provided in Appendix \ref{app: proof of lemma create pseudo starting signal without bos}.
\end{proof}
\end{proposition}

By leveraging Proposition \ref{proposition: create pseudo starting signal without bos}, we also show that Transformers without $\texttt{<bos>}$ can recognize and generate under the assumption that there exists a subspace of the input representation with a constant $2$-norm.

\begin{corollary}[Transformers without starting token, $\texttt{Dyck}_k$ probabilistic recognition]\label{corollary: transformers without bos recognize dyck_k}
Assume the same condition as in Proposition \ref{proposition: create pseudo starting signal without bos}. There exists a $9$-layer causal Transformer without a starting token that recognizes the $\texttt{Dyck}_k$ language with probability at least $1-1/k$.
\begin{proof}
The proof is provided in Appendix \ref{app: dyck recognition without bos}.
\end{proof}
\end{corollary}

\begin{corollary}[Transformers without starting token, $\texttt{Dyck}_k$ subset generation]\label{corollary: transformers without bos generate dyck_k}
Assume the same condition as in Proposition \ref{proposition: create pseudo starting signal without bos}. There exists a $7$-layer causal Transformer without a starting token that can generate a subset of $\texttt{Dyck}_k$ where the first two characters are different; that is, the Transformer can generate all possible subsequent sequences when there is an input string whose first two characters are different.
\begin{proof}
The proof is provided in Appendix \ref{app: dyck language generation without bos}.
\end{proof}
    
\end{corollary}

\section{Experiments}\label{sec:experiments}
The constructive proofs in the previous section show that single-head Transformers with a starting token have the ability to recognize and generate $\texttt{Dyck}_k$ and $\texttt{Shuffle-Dyck}_k$, and that even without a starting token, Transformers can recognize and generate $\texttt{Dyck}_k$. In this section, we examined the theoretical results by conducting experiments on the generation ability for $\texttt{Dyck}_k$ with/without a starting token (Theorem \ref{theorem: transformers with bos generate dyck} and Corollary \ref{corollary: transformers without bos generate dyck_k}). We also investigated the generation ability on $\texttt{Shuffle-Dyck}_k$ (Proposition \ref{proposition: transformers with bos generate shuffle dyck_k}). 

In addition, we empirically investigated the effect of the layer normalization position on model performance using natural language datasets because the Transformer architecture used in this paper differs from common architectures regarding the layer normalization position.

\subsection{Evaluation on $\texttt{Dyck}_k$ and $\texttt{Shuffle-Dyck}_k$}

\begin{figure*}[t]
  \includegraphics[width=0.44\linewidth]{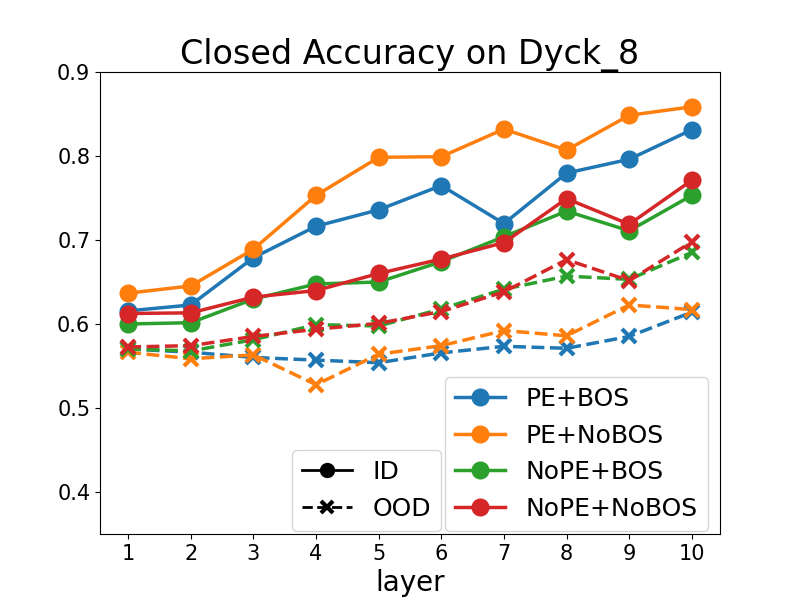}
  \hfill
  \includegraphics[width=0.44\linewidth]{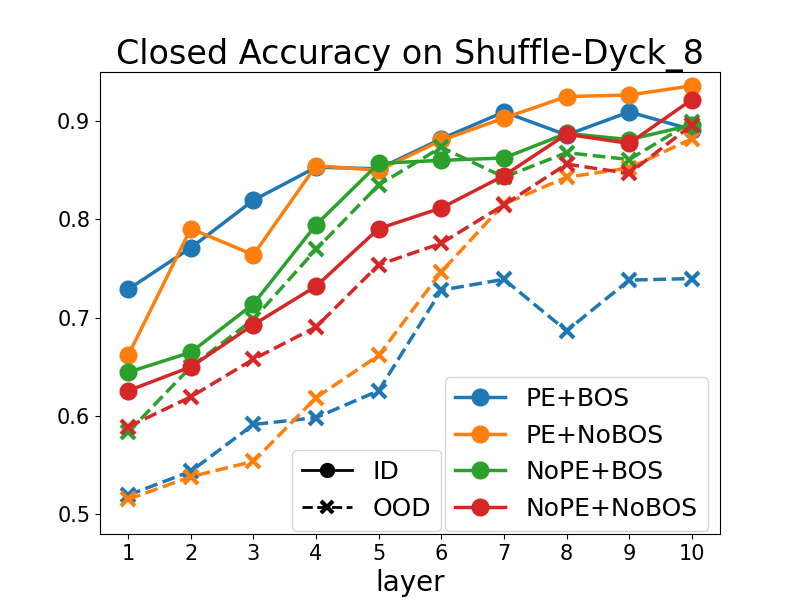}
  \caption{(Left) Test accuracy of generating the correct closed brackets on $\texttt{Dyck}_8$. (Right) Test accuracy of generating the correct closed bracket on $\texttt{Shuffle-Dyck}_8$. The solid lines represent the results for in-distribution data ($n \leq 700$), while the dashed lines represent the results for out-of-distribution data ($700 < n \leq 840$). In both experiments, results are averaged over $5$ runs with different random seeds.}
  \label{fig:closed_acc_dyck_8}
\end{figure*}

Our constructive proofs show that Transformers are capable of recognizing and generating $\texttt{Dyck}_k$ and $\texttt{Shuffle-Dyck}_k$.
In this section, we experimentally investigated whether such networks can actually be learned. Here, we provide a brief explanation of the experimental setup and the results for $\texttt{Dyck}_8$ and $\texttt{Shuffle-Dyck}_8$, while the detailed explanation and other results are provided in Appendix \ref{app: experiments}.

\paragraph{Setup}
Following \citet{yao-etal-2021-self}, we generated training and validation sets with the maximum input length of $700$ according to a language generation process. We trained Transformers with causal masking by having them solve a next-token prediction task. 

We compared four types of models: (i) with positional encoding and a starting token ($\texttt{PE+BOS}$), (ii) with positional encoding but without a starting token ($\texttt{PE+NoBOS}$), (iii) without positional encoding but with a starting token ($\texttt{NoPE+BOS}$), and (iv) without positional encoding and a starting token ($\texttt{NoPE+NoBOS}$). We reported the average accuracy of generating correct closed brackets, which is described below, separately for in-distribution (ID) data ($n \leq 700$) and out-of-distribution (OOD) data ($700 < n \leq 840$).

\paragraph{Metric}
Following \citet{hewitt-etal-2020-rnns}, \citet{yao-etal-2021-self}, we report the conditioned probability to output the correct closed bracket(s): 
\begin{equation}
\operatorname{Acc}_{\mathrm{closed}} = \mathbb{E} \left[ p(\rangle_{t_\mathrm{valid}} \mid \, \rangle_\cdot)\right],
\end{equation}
where 
\begin{equation}
\begin{aligned}
&p(\rangle_{t_\mathrm{valid}} \mid \, \rangle_\cdot) \\
&=\begin{cases}
\begin{aligned}\frac{p(\rangle_{t_\mathrm{valid}})}{\sum_{t=1}^k p(\rangle_t)}\end{aligned} & \text{for } \texttt{Dyck}_k\\
\begin{aligned} \frac{\sum_{t_{\mathrm{valid}}} p(\rangle_{t_\mathrm{valid}})}{\sum_{t=1}^k p(\rangle_t)}\end{aligned}\vphantom{\frac{\vdots}{\vdots}}& \text{for } \texttt{Shuffle-Dyck}_k
\end{cases}.
\end{aligned}
\end{equation}
This metric indicates how accurately the models can generate the sequence.

Figure \ref{fig:closed_acc_dyck_8} shows the test accuracy of generating the correct closed bracket on $\texttt{Dyck}_8$ and $\texttt{Shuffle-Dyck}_8$, while the results for other values of $k$ are provided in Appendix \ref{app: experiments}.

\subsection{Evaluation on natural language datasets}\label{subsec:Experiment/Evaluation on natural language datasets} 
In the previous section, we derived theoretical results using the architecture that differs from common ones with respect to the position of the layer normalization. In this section, we experimentally investigated the performance differences arising from the positions of the layer normalization.

Generally, there are two types regarding the position of the layer normalization used in Transformer architectures. One is $\texttt{Post-LN}$, used in models such as GPT \citep{radford2018improving}, and the other is $\texttt{Pre-LN}$, used in models such as GPT-2 \citep{radford2019language}, GPT-3 \citep{NEURIPS2020_1457c0d6}, Llama \citep{touvron2023llamaopenefficientfoundation}, and Llama 2 \citep{touvron2023llama2openfoundation}.
Specifically, in $\texttt{Post-LN}$, the layer normalization is applied after the attention layer and feed-forward network layer, whereas in $\texttt{Pre-LN}$, it is applied before these layers. In contrast, the architecture used in our proofs incorporates the layer normalization after the first linear transformation in the feed-forward network layer.

We investigated how the positions of the layer normalization affect the model performance using the two natural language datasets, WikiText-103 \citep{merity2016pointer} and OpenWebText \citep{Gokaslan2019OpenWeb}.

\begin{table}
    \begin{tabular}{c|cc}
        \textbf{Architecture} & \textbf{WikiText-103}  & \textbf{OpenWebText} \\ 
        \hline 
        $\texttt{Post-LN}$      & 19.11 & 20.82 \\
        $\texttt{Pre-LN}$       & 19.44 & 20.83 \\
        $\texttt{No-LN}$      & 21.25 & 22.72 \\
        $\texttt{FFN-LN}$     & 19.17 & 21.32 
    \end{tabular}
    \caption{Test perplexity on two natural language datasets with different positions of the layer normalization. }
    \label{tab: test perplexity}
\end{table}

In addition to the two positions, $\texttt{Post-LN}$ and $\texttt{Pre-LN}$, we considered two other settings, $\texttt{No-LN}$ and $\texttt{FFN-LN}$, where $\texttt{No-LN}$ represents an architecture without the layer normalization and $\texttt{FFN-LN}$ represents an architecture that incorporates the layer normalization right after the first linear transformation in feed-forward network layer. We trained four 124M models ($\texttt{Post-LN}$, $\texttt{Pre-LN}$, $\texttt{No-LN}$, and $\texttt{FFN-LN}$) once each from scratch. The test perplexities of the models that achieve the best validation losses are described in Table \ref{tab: test perplexity}. The Appendix \ref{app: experiments} provides detailed information about the training process and other results.

\section{Discussion}\label{sec:discussion}

\subsection{Experiments on $\texttt{Dyck}_8$ and $\texttt{Shuffle-Dyck}_8$ }
From the results in Figure \ref{fig:closed_acc_dyck_8}, $\texttt{PE}$ lets models achieve higher accuracy on ID data compared to $\texttt{NoPE}$. However, the performance drops significantly on OOD data. On the other hand, for $\texttt{NoPE}$, the performance on OOD data drops slightly compared to that on ID data. This suggests that $\texttt{NoPE}$ might let models obtain a better inductive bias with respect to capturing hierarchical structure and generalizing with respect to sequence length. In addition, we did not observe a noticeable difference between $\texttt{BOS}$ and $\texttt{NoBOS}$. This correlates with Corollary \ref{corollary: transformers without bos generate dyck_k}.

\subsection{Experiments on natural language datasets}
Here, we discuss the optimal position of the layer normalization. \citet{wang-etal-2019-learning-deep} and \citet{pmlr-v119-xiong20b} showed that $\texttt{Pre-LN}$ leads to stable training and training time reduction compared to $\texttt{Post-LN}$, while \citet{nguyen-salazar-2019-transformers} and \citet{mao-etal-2023-exploring} demonstrated that under certain conditions, such as machine translation, $\texttt{Post-LN}$ outperforms $\texttt{Pre-LN}$.
Furthermore, \citet{shleifer2021normformerimprovedtransformerpretraining} demonstrated that incorporating the layer normalization before the second linear layer of the feed-forward network layer can effectively mitigate gradient explosion and vanishing, which are commonly observed issues in both $\texttt{Pre-LN}$ and $\texttt{Post-LN}$ setups. 

In this way, although the optimal position remains unclear, we conclude that our modified architecture is competitive to $\texttt{Pre-LN}$ and $\texttt{Post-LN}$ because the architecture used in our proof effectively benefits from the layer normalization in the experiments on WikiText-103 and OpenWebText.
Further discussion on the layer normalization position is provided in Appendix \ref{app: further discussion}.

\section{Conclusion}\label{sec: conclusion}
We theoretically showed that Transformers can efficiently process hierarchical languages. Our theoretical and empirical results might alleviate the existing concern that Transformers, unlike RNNs and LSTMs, often face difficulties in capturing hierarchical structures.

\subsection*{Limitations}
We adopt the layer normalization position that differs from the commonly used positions, but it remains unclear whether this specific position is essential for our proofs. We also assume real numbers with infinite precision, occasionally involving operations with large real values, which leads to a question as to whether it is possible to realize such operations with finite-bit floating point representation. This issue is particularly important in light of recent trends towards quantization for reducing model sizes, where $16$-bit or even $4$ or $8$-bit floating-point representations are frequently used.

In addition, we trained $124$M models using two natural language datasets and empirically demonstrated the validity of the architecture we adopted. However, it remains unclear whether the adopted architecture is competitive with $\texttt{Pre-LN}$ and $\texttt{Post-LN}$ when applied to larger models or different datasets.

\subsection*{Ethics Statement}
This paper consists solely of theoretical results and supporting experiments. While we conducted experiments using natural language datasets, we have presented only sufficiently aggregated results. To the best of our knowledge, there are no ethical concerns or potential risks associated with this study.


\clearpage
\appendix
\setlength{\arraycolsep}{3.5pt}
\allowdisplaybreaks

\section{Additional Related Work}
Since the emergence of Transformer \citep{NIPS2017_3f5ee243}, a wide range of theoretical analyses have been conducted on its expressive capacity. Some of these analyses focus on language recognition and generation tasks. These analyses can be broadly classified into two categories: (i) studies on the expressive capacity using circuit complexity and (ii) studies on the expressive power by examining specific languages.

\subsection{Theoretical analyses based on circuit complexity}

There have been studies that try to identify the language classes that Transformers can process from the perspective of circuit complexity. 
\citet{hao-etal-2022-formal} established the relationship between unique hard attention Transformers (UHAT) and circuits and showed that UHAT can only recognize languages in the circuit class $\texttt{AC}^0$. $\texttt{AC}^0$ is a circuit class that circuits consisting of constant depth and polynomial size $\texttt{AND}$ and $\texttt{OR}$ gates belong to.
In addition, \citet{barcelo2024logical} showed that UHAT cannot recognize all languages in $\texttt{AC}^0$.
Furthermore, \citet{hao-etal-2022-formal}, \citet{merrill-etal-2022-saturated}, \citet{merrill-sabharwal-2023-parallelism}, \citet{strobl2023averagehardattentiontransformersconstantdepth}, and \citet{barcelo2024logical} provided theoretical results on saturated attention, or average hard attention (AHAT), which extends hardmax attention to be able to refer more than one token. \citet{hao-etal-2022-formal} showed that AHAT has strictly higher expressive power compared to UHAT. In addition, \citet{merrill-etal-2022-saturated} provided a proof that AHAT can only recognize languages in the circuit class $\texttt{TC}^0$, where $\texttt{TC}^0$ is an extended circuit class of $\texttt{AC}^0$ by adding majority gates to $\texttt{AND}$ and $\texttt{OR}$ gates. \citet{barcelo2024logical} showed that AHAT can recognize languages within the linear temporal logic extended to require counting. Furthermore, \citet{merrill-sabharwal-2023-parallelism} showed that log-precision Transformers can only recognize the languages within the class of uniform $\texttt{TC}^0$. Furthermore, \citet{strobl2023averagehardattentiontransformersconstantdepth} showed that AHAT can also recognize the languages within the class of uniform $\texttt{TC}^0$.

\subsection{Theoretical analyses on specific languages}

On the other hand, some studies have focused on specific languages to examine the expressive power of Transformers,  particularly for the parity language within regular languages, the $\texttt{Dyck}_k$ language within context-free languages, and the $\texttt{Shuffle-Dyck}_k$ language. \citet{hahn-2020-theoretical} pointed out that Lipschitz-continuous Transformers cannot solve the parity task, $\texttt{Dyck}_1$, and $\texttt{Dyck}_2$ for arbitrary lengths. This is because when one character out of an input string of length $n$ is changed, the change in the output decays at $O(1/n)$, indicating that the performance of Transformers with restricted a Lipschitz constant approaches random guessing as the input length increases. Meanwhile, \citet{yao-etal-2021-self} and \citet{chiang-cholak-2022-overcoming} showed that the theoretical limitations presented by \citet{hahn-2020-theoretical} can be overcome by incorporating layer normalization because the Lipschitz constant of layer normalization can be $O(n)$. \citet{chiang-cholak-2022-overcoming} also showed that Transformers with layer normalization can solve the PARITY task by incorporating task-specific positional encoding $i/n, (-1)^i$.
Furthermore, \citet{bhattamishra-etal-2020-ability} theoretically showed that $O(k)$-width Transformers can recognize the $\texttt{Shuffle-Dyck}_k$ language \footnote{Intuitively, the $\texttt{Shuffle-Dyck}_k$ language is a set of strings composed of $k$ types of brackets, where all of the $k$ types of substrings are well-balanced. For instance, $\texttt{``([)]"}$ belongs to $\texttt{Shuffle-Dyck}_2$ not to $\texttt{Dyck}_2$. }, suggesting that  $O(k)$-width Transformers can process $k$ hierarchical structures in parallel.

In addition, there have also been studies that focus on how Transformers handle such hierarchical structures. \citet{ebrahimi-etal-2020-self} focused on the Dyck language and demonstrated that the stack states appear in the attention patterns, suggesting that the self-attention networks learn hierarchical structures within the attention layers. However, \citet{NEURIPS2023_79ba1b82} indicated that such attention patterns cannot be fully reliable. Moreover, \citet{NEURIPS2023_79ba1b82} provided a proof that a two-layer Transformer network with a width of $O(k^2D^2)$ can recognize $\texttt{Dyck}_k$. Furthermore, \citet{yao-etal-2021-self} provided a constructive proof that by using specific absolute positional encoding $i/n$, $3$-layer causal Transformers can recognize the $\texttt{Dyck}_{k, D}$ language. \citet{yao-etal-2021-self} also showed that $2$-layer causal Transformers with absolute positional encoding $i/n, i/n^3, n$ can generate the $\texttt{Dyck}_{k}$ language.

\subsection{Analyses on the role of uninformative tokens}
Moreover, there have been studies focusing on the importance of uninformative tokens --- the BOS token in GPT \citep{radford2018improving} and the CLS and SEP tokens in BERT \citep{devlin-etal-2019-bert}. \citet{clark-etal-2019-bert}, \citet{devlin-etal-2019-bert}, and \citet{kovaleva-etal-2019-revealing} observed that BERTs place relatively large attention on the CLS and SEP tokens. \citet{clark-etal-2019-bert} speculated that this phenomenon is for achieving $``\text{no-operations}"$. In addition, \citet{nye2021workscratchpadsintermediatecomputation} observed that in algorithmic tasks, special tokens such as the CLS token serve as scratchpads, contributing to performance improvement.

In contrast, although the BOS token cannot refer to other tokens under causal masking, \citet{ebrahimi-etal-2020-self} empirically showed that the presence of a starting token significantly improves the performance in recognizing the Dyck language. In addition, \citet{weiss2021thinking} showed that with a starting token, it is possible to determine how many tokens each head focuses on. Moreover, \citet{NEURIPS2023_4e85362c} showed that with the BOS token, Transformers can create specific absolute and relative positional encoding. Furthermore, \citet{xiao2024efficient} demonstrated that by slightly modifying the Transformer architecture with a fixed window size so that every token can refer to a starting token, Transformers perform significantly better. 
In light of these theoretical and empirical results, it has become evident that even tokens that do not have meaning independently are significant to enhance the performance of Transformers.

\section{Detailed Preliminaries}\label{app: preliminaries}
We provide preliminaries for the proofs in the following sections and detailed definitions that are omitted due to the lack of space.

\subsection{$\texttt{Shuffle-Dyck}_k$}\label{app: preliminaries, shuffle-dyck}

Following \citet{suzgun-etal-2019-lstm}, before defining the $\texttt{Shuffle-Dyck}_k$ language, we first define the shuffling operation over two strings $\sha :\Sigma^* \times \Sigma^* \rightarrow 2^{\Sigma^*}$ as follows:
\begin{align}
& u_1\sha\varepsilon=\varepsilon\sha u_1=\{u_1\}, \\
& \beta_1 u_1 \sha \beta_2 u_2 = \beta_1(u_1 \sha \beta_2 u_2) \cup \beta_2(\beta_1 u_1 \sha u_2)
\end{align}
for any $\beta_1, \beta_2 \in \Sigma$ and $u_1, u_2 \in \Sigma^*$. For instance, 
\begin{equation}
\begin{aligned}
\langle_1\rangle_1 \sha \langle_2\rangle_2 = \{&\langle_1\rangle_1\langle_2\rangle_2, \langle_1\langle_2\rangle_1\rangle_2,\langle_1\langle_2\rangle_2\rangle_1, \\
&\langle_2\rangle_2\langle_1\rangle_1, \langle_2\langle_1\rangle_2\rangle_1,\langle_2\langle_1\rangle_1\rangle_2\}.
\end{aligned}
\end{equation}

Moreover, we define the shuffling operation over $k$ strings $u_1, \cdots, u_k \in \Sigma^*$ and over $k$ languages $\mathcal{L}_1, \cdots, \mathcal{L}_k \subset \Sigma^*$ as follows:
\begin{align}
&\overset{k}{\underset{t=1}{\Sha}} u_t = \bigcup_{u \in {\Sha}_{t=1}^{k-1} u_t} u_k \sha u, \\
&\overset{k}{\underset{t=1}{\Sha}} \mathcal{L}_t = \bigcup_{u_1 \in \mathcal{L}_1, \cdots, u_k \in \mathcal{L}_k} \overset{k}{\underset{t=1}{\Sha}} u_t,
\end{align}
where $\overset{1}{\underset{t=1}{\Sha}} u_t = \{ u_1 \}$.

\begin{definition}[$\texttt{Shuffle-Dyck}_k$ language for language models]\label{def: shuffle dyck_k for language models formal}
The $\texttt{Shuffle-Dyck}_k$ language for language models is a language over an alphabet $\Sigma = \{\langle_t, \rangle_t\}_{t=1}^k \cup \{\texttt{<bos>}, \texttt{<eos>}\}$. 

Given $k$ $\texttt{Dyck}_1$ --- $\texttt{Dyck}_1^1, \cdots, \texttt{Dyck}_1^k$, where $\texttt{Dyck}_1^t$ is the $\texttt{Dyck}_1$ language over an alphabet $\{\langle_t, \rangle_t\}$ ---, the $\texttt{Shuffle-Dyck}_k$ language for language models is defined as follows:

\begin{equation}
\left\{\texttt{<bos>} w \texttt{<eos>}\,\left|\, w \in \overset{k}{\underset{t=1}{\Sha}} \texttt{Dyck}_1^t\right.\right\}
\end{equation}

Intuitively, the $\texttt{Shuffle-Dyck}_k$ language is a mixture of the $k$ $\texttt{Dyck}_1$ languages, and the ability to process the $\texttt{Shuffle-Dyck}_k$ language suggests that $k$ hierarchical structures can be processed in parallel. Figure \ref{fig: shuffle-dyck example} shows an example string that belongs to $\texttt{Shuffle-Dyck}_3$.
\end{definition}

\begin{definition}[$\texttt{Shuffle-Dyck}_k$ language generation process]\label{def: shuffle-dyck_k language generation process}
A language generation process $p(w_{i+1} \mid w_{0:i})$ over an alphabet $\Sigma = \{\langle_t, \rangle_t\}_{t=1}^k \cup \{\texttt{<bos>}, \texttt{<eos>}\}$ is called the $\texttt{Shuffle-Dyck}_k$ language generation process if
\begin{align}
&p(w_0 = ``\texttt{<bos>}" \mid \varepsilon) = 1 , \\
&\begin{aligned}
& p(w_{i+1} \mid w_{0:i} ) \\
& =\begin{cases}
p_{0}(w_{i+1}) & \text{if } \forall t \in [k]. \operatorname{d}(w_{0:i} \mid  t) = 0\\
p_{1}(w_{i+1}) & \text{otherwise }
\end{cases},
\end{aligned}
\end{align}
where $\operatorname{d}(w_{0:i} |  t)$ represents the depth of the substring of type $t$ extracted from $w_{0:i}$, and $p_{0}(w_{i+1}),p_{1}(w_{i+1})$ are defined as follows:
\begin{align}
&p_{0}(w_{i+1}) = 
\begin{cases}
r \pi_t & \text{if } w_{i+1} = ``\langle_t" \\
1-r & \text{if } w_{i+1} = ``\texttt{<eos>}" \\
0 & \text{otherwise}
\end{cases}, \\
&p_{1}(w_{i+1}) = 
\begin{cases}
\frac{q\pi_t}{Z} & \text{if } w_{i+1} = ``\langle_t" \\
\frac{(1-q)\overline{\pi}_t}{Z} & \begin{aligned}
&\text{if } w_{i+1} = ``\rangle_t" \\
&\,  \wedge\, \operatorname{d}(w_{0:i} \mid  t) > 0 \end{aligned} \\
0 & \text{otherwise}
\end{cases} ,
\end{align}
where $\boldsymbol{\pi}, \boldsymbol{\overline{\pi}} \in \Delta^{k-1}$, and 
\begin{equation}
Z = \sum_{t'=1}^k q \pi_{t'} + \sum_{t'\in\{t \mid  \operatorname{d}(w_{0:i}\mid t) > 0\}} (1-q)\overline{\pi}_{t'}.
\end{equation}

Hereafter, we explicitly write the $\texttt{Shuffle-Dyck}_k$ language generation process parameterized by $q, r, \boldsymbol{\pi}, \boldsymbol{\overline{\pi}}$ as $p_{\texttt{Shuffle-Dyck}_k}(\cdot; q, r, \boldsymbol{\pi}, \boldsymbol{\overline{\pi}})$.
\end{definition}

\begin{figure}
    \includegraphics[width=\linewidth]{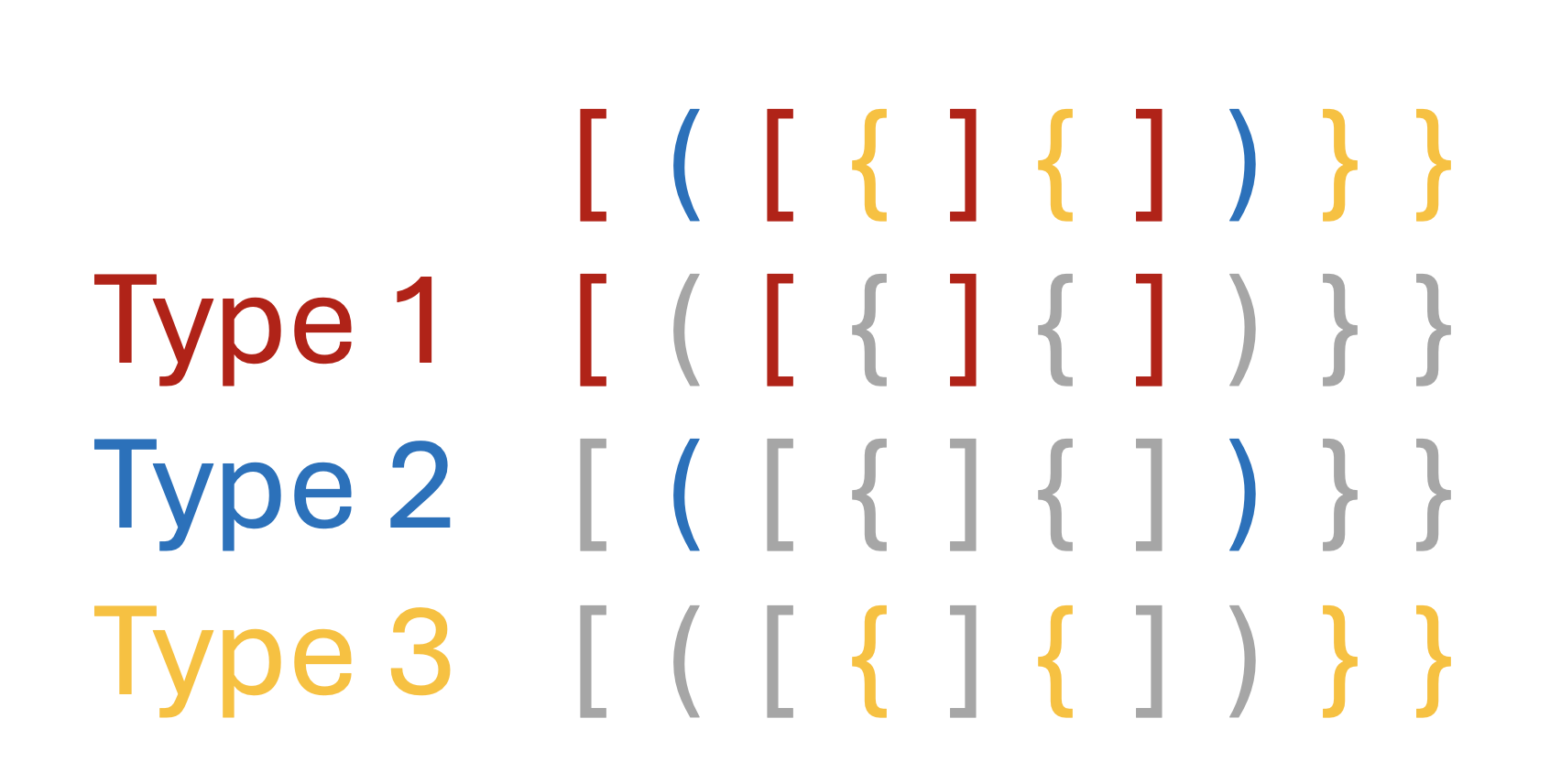}
    \caption{An example of string that belongs to $\texttt{Shuffle-Dyck}_3$ not to $\texttt{Dyck}_3$. Each substring of type $t \in \{1, 2, 3\}$ is properly balanced.}
    \label{fig: shuffle-dyck example}
\end{figure}

\subsection{Transformer Architecture}\label{app: preliminaries, transformer architecture}

We largely follow the Transformer architecture adopted in \citet{yao-etal-2021-self}; that is, we consider Transformer architecture composed of multiple Transformer blocks, each of which incorporates a self-attention layer and a feed-forward network layer. 

Let $L$ be the number of Transformer blocks, $d_{\mathrm{model}}$ be the dimension of the embedding vectors and hidden representations, $\Sigma$ be the vocabulary set, and $K$ be the vocabulary size.

Given an input string $w_{0:n} (= \texttt{<bos>} w_{1:n} )\in \Sigma^*$, which we identify with the sequence of one-hot vectors $\begin{bmatrix}\mathbf{e}_{w_0} & \cdots & \mathbf{e}_{w_n}\end{bmatrix} \in \mathbb{R}^{K\times (n+1)}$, the architecture process the string as follows:

\begin{align}
& \mathbf{x}_i^{(1)} = W_{\mathrm{emb}}\mathbf{e}_{w_i} + \mathbf{p}_i, \\
& \mathbf{h}_{i}^{(\ell)} = \operatorname{Att}\left(W_Q^{(\ell)}\mathbf{x}_{i}^{(\ell)}, W_K^{(\ell)}\mathbf{x}_{0:i}^{(\ell)}, W_V^{(\ell)}\mathbf{x}_{0:i}^{(\ell)}\right), \\
& \mathbf{x}_{i}^{(\ell+1)} = \operatorname{FFN}\left(\mathbf{h}_{i}^{(\ell)}; W_1^{(\ell)}, W_2^{(\ell)}, \boldsymbol{\beta}^{(\ell)}, \boldsymbol{\gamma}^{(\ell)}\right),
\end{align}
where 
\begin{itemize}
    \item $\mathbf{x}^{(\ell)}_i\in \mathbb{R}^{d_{\mathrm{model}}}$ is the $i$-th input representation to the $\ell$-th layer,
    \item $W_{\mathrm{emb}}\in\mathbb{R}^{d_{\mathrm{model}}\times K}$ is a linear embedding function,
    \item $\mathbf{p}_i \in \mathbb{R}^{d_{\mathrm{model}}}$ is the positional encoding at the position $i$,
    \item $\operatorname{Att}(\cdot)$ is an attention layer, which is parameterized by three matrices $W_Q^{(\ell)}, W_K^{(\ell)}, W_V^{(\ell)} \in \mathbb{R}^{{d_{\mathrm{model}}}\times {d_{\mathrm{model}}}}$,
    \item $\operatorname{FFN}(\cdot)$ is a feed-forward network layer, which is parameterized by two matrices $W_1^{(\ell)}, W_2^{(\ell)} \in \mathbb{R}^{{d_{\mathrm{model}}} \times {d_{\mathrm{model}}}}$ and $\boldsymbol{\beta}^{(\ell)}, \boldsymbol{\gamma}^{(\ell)} \in \mathbb{R}^{d_{\mathrm{model}}}$. 
\end{itemize}

Next, we describe the details of the attention and feed-forward network layers.

\subsubsection*{Attention layer}\label{subsubsec: preliminaries/Transformer Architecture/Attention layer}
We consider attention layers with causal masking and the residual connection \citep{he2015deepresiduallearningimage}. 
Specifically, the input sequence of length $n+1$ --- $\mathbf{x}_0^{(\ell)}, \cdots, \mathbf{x}_n^{(\ell)}$ --- is first processed with three token-wise linear transformations $W_Q^{(\ell)}, W_K^{(\ell)},W_V^{(\ell)}$, which create $3(n+1)$ vectors $\left\{W_Q^{(\ell)} \mathbf{x}_{i}^{(\ell)}, W_K^{(\ell)} \mathbf{x}_{i}^{(\ell)}, W_V^{(\ell)} \mathbf{x}_{i}^{(\ell)}\right\}_{i=0}^n$. Then, the $i$-th output $\mathbf{h}_i^{(\ell)}$ is calculated as follows:
\begin{align}
&\begin{aligned}
\boldsymbol{\alpha}_i^{(\ell)} &= \mathbb{S}\left(\left\langle W_Q^{(\ell)}\mathbf{x}_i^{(\ell)}, W_K^{(\ell)}\mathbf{x}_0^{(\ell)} \right\rangle, \right.\\
& \left. \quad \qquad \cdots, \left\langle W_Q^{(\ell)}\mathbf{x}_i^{(\ell)}, W_K^{(\ell)}\mathbf{x}_i^{(\ell)} \right\rangle\right),
\end{aligned} \\
&\begin{aligned} 
\mathbf{a}_i^{(\ell)}&= \sum_{j=0}^i \alpha_{i, j}^{(\ell)} W_V^{(\ell)}\mathbf{x}_j^{(\ell)} ,
\end{aligned} \\
&\begin{aligned}\mathbf{h}_i^{(\ell)} &= \operatorname{Att}\left(W_Q^{(\ell)}\mathbf{x}_{i}^{(\ell)}, W_K^{(\ell)}\mathbf{x}_{0:i}^{(\ell)}, W_V^{(\ell)}\mathbf{x}_{0:i}^{(\ell))}\right)
\\
& = \mathbf{x}_i^{(\ell)} + \mathbf{a}_i^{(\ell)}, 
\end{aligned}\end{align}
where $\langle \cdot, \cdot \rangle$ is a dot-product and $\mathbb{S}(\cdot)$ is a softmax operation.

\subsubsection*{Feed-forward network layer}\label{subsubsec: preliminaries/Transformer Architecture/Feed-forward network layer}

A feed-forward network layer is a token-wise transformation that maps $\mathbf{h}_i \mapsto \operatorname{FFN}(\mathbf{h}_i)$. In this paper, we implement a feed-forward network as two linear transformations with the ReLU activations. We adopt the residual connection \citep{he2015deepresiduallearningimage} and the RMS layer normalization \citep{NEURIPS2019_1e8a1942}. This architecture largely follows that proposed in \citet{yao-etal-2021-self} with a slight modification: we replace the standard layer normalization \citep{ba2016layernormalization} with the RMS layer normalization \citep{NEURIPS2019_1e8a1942}. Specifically, the feed-forward network transforms the vector $\mathbf{h}_i^{(\ell)}$ as follows:
\begin{equation}
\begin{aligned}
&\operatorname{FFN}\left(\mathbf{h}_{i}^{(\ell)}; W_1^{(\ell)}, W_2^{(\ell)},\boldsymbol{\beta}^{(\ell)}, \boldsymbol{\gamma}^{(\ell)}\right) \\
&\,\, = \mathbf{h}_{i}^{(\ell)} + W_2^{(\ell)} \left[\left( \operatorname{LN}_\mathrm{RMS} \left(W_1^{(\ell)} \mathbf{h}_{i}^{(\ell)}\right)\right)\right]_+,
\end{aligned}
\end{equation}
where $[\cdot]_{+}$ is a ReLU activation and $\operatorname{LN}_\mathrm{RMS}(\cdot)$ is the RMS layer normalization \citep{NEURIPS2019_1e8a1942} parameterized by $\boldsymbol{\beta}^{(\ell)}, \boldsymbol{\gamma}^{(\ell)} \in \mathbb{R}^{d_\mathrm{model}}$. \citet{NEURIPS2019_1e8a1942} empirically showed that the RMS layer normalization reduces the training time compared to the conventional layer normalization while maintaining their performances.
Specifically,
\begin{equation}
\operatorname{LN}_\mathrm{RMS}(\mathbf{y}) = \boldsymbol{\gamma}^{(\ell)} \odot \frac{\mathbf{y}}{\operatorname{RMS}(\mathbf{y})}+\boldsymbol{\beta}^{(\ell)},
\end{equation}
where $\odot$ is an element-wise multiplication and 
\begin{equation}
\operatorname{RMS}(\mathbf{y}) = \sqrt{\frac{1}{{d_{\mathrm{model}}}}\sum_{d=1}^{d_{\mathrm{model}}} y_d^2}.
\end{equation}
The RMS layer normalization has been adopted in recent models such as Llama \citep{touvron2023llamaopenefficientfoundation} and Llama 2 \citep{touvron2023llama2openfoundation}.

\section{Notation}
The notations used in this paper are summarized in Table \ref{tab: table of notations}.
\begin{table*}
  \centering
  \begin{tabular}{ll}
    \hline
    \textbf{Variable} & \textbf{Definition} \\
    \hline
    $k$ & Number of bracket types\\
    $\texttt{Dyck}_k \, /\, \texttt{Shuffle-Dyck}_k$     & The Dyck / Shuffle-Dyck languages with $k$ types of bracket pairs \\
    $D$     & Maximum depth of the Dyck language \\
    $\texttt{Dyck}_{k, D}$     & The $\texttt{Dyck}_k$ language with bounded depth $D$ \\
    $\Sigma \, / \, K$ & Vocaburary set / Vocaburary size\\
    $t$     & Bracket type \\
    $``\langle_t"\, / \,``\rangle_t"$ & Open / Closed bracket of type $t$ \\
    $``\texttt{<bos>}" \, / \, ``\texttt{<eos>}"\,/\, \varepsilon$ & BOS / EOS token / Empty string \\
    $\sha\,/\,\Sha$ & Shuffling operation over two strings / multiple strings or languages \\
    $n \, / \, n_\mathrm{max}$     & Length of input string / Maximum length of the training dataset\\
    $i, j$     & Index of position \\
    $w_{0:i} (=\texttt{<bos>} w_{1:i})$   & Prefix of string $w_{0:n}$ with a length of $i+1$\\
    $\mathcal{L}$ & Language \\
    $\operatorname{d}(\cdot)$ & Depth function \\
    $p_\mathcal{L}(\cdot)$ & Language generation process of language $\mathcal{L}$ (Definition \ref{def: language generation process}) \\
    $q,r,\boldsymbol{\pi}\,/\, q,r,\boldsymbol{\pi},\boldsymbol{\overline{\pi}}$ & Parameters of the $\texttt{Dyck}_k$ / $\texttt{Shuffle-Dyck}_k$ language generation process \\
    $L$     & Number of Transformer blocks \\
    $d_{\mathrm{model}}$     & Dimension of token representation \\
    $W_{\mathrm{emb}}$ & Token embedding matrix \\
    $\mathbf{p}_i$  & Positional encoding at position $i$ \\
    $\mathbf{x}_i^{(\ell)}$ & Input vector to the $\ell$-th layer at position $i$ \\
    $\mathbf{h}_i^{(\ell)} \left(=\mathbf{x}_i^{(\ell)} + \mathbf{a}_i^{(\ell)}\right)$ & Output vector of the $\ell$-th attention layer at position $i$ \\
    $\mathbf{t}_i\,/\, o_i\,/\, s_i$ & Bracket-type embedding / Openness of bracket / Starting signal \\
    $\hat{s}$ & Pseudo starting signal \\
    $\operatorname{Att}(\cdot) \,/\, \operatorname{FFN}(\cdot)$ & Self-attention layer / Feed-forward network layer\\
    $W_Q^{(\ell)} \, / \, W_K^{(\ell)} \,/  \,W_V^{(\ell)} $ & Query / key / value matrices that parameterize $\operatorname{Att}(\cdot)$ in $\ell$-th layer\\
    $W_1^{(\ell)} \, / \, W_2^{(\ell)} $ & Weights of the first / second linear transformation in $\operatorname{FFN}(\cdot)$ in $\ell$-th layer\\
    $\boldsymbol{\alpha}_i^{(\ell)}$ & Attention weights of query at position $i$ in $\ell$-th attention layer \\
    $\operatorname{LN}(\cdot)\, /\, \operatorname{LN}_{\mathrm{RMS}}(\cdot)$ & The layer normalization / The RMS layer normalization\\
    $\operatorname{RMS}(\cdot)$ & Root mean square \\
    $\boldsymbol{\beta}^{(\ell)}, \boldsymbol{\gamma}^{(\ell)}$ & Parameters of the RMS layer normalization in $\ell$-th attention layer\\
    $\langle \cdot, \cdot \rangle\,/\,\mathbb{S}(\cdot)\,/\,\Delta^{K-1}$ &  Dot product / Softmax function / $(K{-1})$-dimensional probability simplex \\
    $\mathcal{T}$ & Transformer $\Sigma^* \rightarrow \mathbb{R}^{* \times d_\mathrm{model}}$, where $*$ represents an arbitrary length. \\
    $f_{\mathrm{rec}}\,/\,f_{\mathrm{gen}}$ & Recognizer head $\mathbb{R}^{d_\mathrm{model}} \rightarrow \mathbb{R}$ / Generator head $\mathbb{R}^{d_\mathrm{model}} \rightarrow \mathbb{R}^{K}$ \\
    $\operatorname{sgn}(\cdot)$ & Sign function $\mathbb{R} \rightarrow \{1, -1\}$ \\
    $p$ & Probability distribution over strings \\
    $(\Sigma^*, \mathcal{F}, P) \, /\, (\Sigma^*, \mathcal{F}', P') $ & Probability space / Complete extension of $(\Sigma^*, \mathcal{F}', P')$ \\
    $\{w_{1:n}\}$ & Singleton set of a string $w_{1:n}$ \\
    $a$ & attention score on a starting token $\texttt{<bos>}$ \\
    $\phi(\cdot)\, / \, \theta(\cdot)$  &  Function that converts position / depth to the angle \\
    $Q(w_{0:i})$ & Propositional variable that indicates $w_{0:i}$ is a prefix for the $\texttt{Dyck}_k$ language \\
    $\operatorname{q}(w_{0:i})$ & Variable associated with the propositional variable $Q(w_{0:i})$ \\
    $\epsilon$ & Small value \\
    $I$ & Identity matrix \\
    \hline
  \end{tabular}
  \caption{Table of notations. }
  \label{tab: table of notations}
\end{table*}

\section{Vector Representation}\label{app: vector representation}
We define the vector representation that is used in the following sections. Specifically, the vector representation of the alphabet $\Sigma = \{\langle_i, \rangle_i\}_{i=1}^k \cup \{\texttt{<bos>}, \texttt{<eos>}\} $ takes the following form:

\begin{equation}
\begin{aligned}
&\mathbf{x}_i \left(\in \mathbb{R}^{d_{\mathrm{model}}}\right)\\
&=\begin{bmatrix}
\mathbf{t}_i \\
o_i \\
s_i \\
1 \\
\mathbf{0}
\end{bmatrix}
\begin{matrix*}[l]
\} \lceil \log_2 k \rceil \text{ dim.} \\
\} 1 \text{ dim.}\\
\} 1 \text{ dim.}\\
\} 1 \text{ dim.}\\
\} (d_{\mathrm{model}} - \lceil \log_2 k \rceil  - 3) \text{ dim.}
\end{matrix*}.
\end{aligned}
\end{equation}

where 
\begin{itemize}
    \item $\mathbf{t}_i\in\{-1, 1\}^{\lceil \log_2 k \rceil}$ represents a bracket-type embedding. Here, bracket types are encoded by $\pm 1$ binary encoding; that is, for $k=4$, $4$ types are encoded into $\begin{bmatrix}-1 \\ -1\end{bmatrix}, \begin{bmatrix}-1 \\ 1\end{bmatrix}, \begin{bmatrix}1 \\ -1\end{bmatrix}, \begin{bmatrix}1 \\1\end{bmatrix}$. Note that the $\mathbf{t}_i$ of the two special tokens $\{``\texttt{<bos>}", ``\texttt{<eos>}"\}$ are defined as zero vectors.
    
    \item $o_i \in \{-1, 0, 1\}$ represents the openness: $o_i = 1$ for open brackets, $o_i = -1$ for closed brackets, and $o_i = 0$ for two special tokens $``\texttt{<bos>}"$ and $``\texttt{<eos>}"$.
    
    \item $s_i \in \{0, 1\}$ is a starting signal that indicates whether the token is the starting token $``\texttt{<bos>}"$ or not. This value is set to $1$ for $``\texttt{<bos>}"$ and $0$ for the other tokens.
    \item $\mathbf{0} \in \mathbb{R}^{{d_{\mathrm{model}}}-\lceil \log_2 k \rceil - 3}$ denotes a zero vector. These dimensions are used as a memory and a scratchpad.
\end{itemize}

These vector representations are implemented with the following embedding matrix:

\begin{equation}
\begin{aligned}
&W_{\mathrm{emb}} \left(\in \mathbb{R}^{{d_{\mathrm{model}}}\times K}\right)\\
&= 
\begin{matrix}
\begin{bmatrix}
\makebox[1em]{$\mathbf{t}_1$} & \makebox[1.4em]{$\cdots$} & \makebox[1em]{$\mathbf{t}_k$} & \makebox[1.3em]{$\mathbf{t}_1$} & \makebox[1.4em]{$\cdots$} & \makebox[1.3em]{$\mathbf{t}_k$} & \makebox[2.5em]{$\mathbf{0}$} & \makebox[2.5em]{$\mathbf{0}$}\\
1 & \cdots & 1 & -1 & \cdots & -1 & 0 & 0\\
0 & \cdots & 0 & 0 & \cdots & 0 & 1 & 0\\
1 & \cdots & 1 & 1 & \cdots & 1 & 1 & 1\\
\mathbf{0} & \cdots & \mathbf{0} & \mathbf{0} & \cdots & \mathbf{0} & \mathbf{0} & \mathbf{0}
\end{bmatrix} \\
\begin{matrix}
\uparrow & & \uparrow & \uparrow & & \uparrow & \uparrow & \uparrow  \\
\makebox[1em]{$\langle_1$} & \makebox[1.4em]{$\cdots$} & \makebox[1em]{$\langle_k$} & \makebox[1.3em]{$\rangle_1$} & \makebox[1.4em]{$\cdots$} & \makebox[1.3em]{$\rangle_k$}  & \makebox[2.5em]{$\texttt{<bos>}$} & \makebox[2.5em]{$\texttt{<eos>}$}
\end{matrix}
\end{matrix} .
\end{aligned}
\end{equation}

\setcounter{proposition}{0}
\section{Proof of Proposition \ref{prop:existence of language process}}\label{app: proof of prop 1}
\begin{proposition}[Restated]
For any language $\mathcal{L}$ over a finite alphabet and any probability distribution $p$ over $\mathcal{L}$, there exists a language generation process that produces the given probability distribution $p$. In other words, there exists a language generation process $p_\mathcal{L}(w_{i+1} \mid \texttt{<bos>} w_{1:i})$ such that for any string $w_{1:n} \in \mathcal{L}$, 
\begin{equation}
p(w_{1:n}) = p_\mathcal{L}(\texttt{<bos>} w_{1:n}\texttt{<eos>}),
\end{equation}
where
\begin{equation}
\begin{aligned}
p_\mathcal{L}&(\texttt{<bos>} w_{1:n}\texttt{<eos>}) \\
&= p_\mathcal{L}(\texttt{<bos>}) \\
& \quad \cdot \left(\prod_{i=1}^n p_\mathcal{L}(w_{i}\mid \texttt{<bos>} w_{1:i-1})\right) \\
& \quad \cdot p_\mathcal{L}(\texttt{<eos>}\mid \texttt{<bos>} w_{1:n}).
\end{aligned}
\end{equation}

\begin{proof}
We introduce a probability space to handle probabilities over the countably infinite set $\Sigma^*$. Given an alphabet $\Sigma$ and a probability space $(\Sigma^*, \mathcal{F}, P)$ over $\Sigma^*$, we can assume that for any $w_{1:n} \in \Sigma^*$ such that $p(w_{1:n}) > 0$, the singleton set $\{w_{1:n}\}$ belongs to $\mathcal{F}$. Here, there exists a unique minimal complete extension of the probability space $(\Sigma^*, \mathcal{F}', P')$, where for any string $w_{1:n} \in \Sigma^{*}$, the singleton set $\{w_{1:n}\} \in \mathcal{F}'$, indicating that $\mathcal{F}' = 2^{\Sigma^*}$. Therefore, any subset in $\Sigma^*$ is $\mathcal{F}'$-measurable.

Next, we define $\operatorname{Cyl}(w_{1:n})$ for a string $w_{1:n} \in \Sigma^*$ as follows:
\begin{align}
&\operatorname{Cyl}(w_{1:n}) = \{w^\prime_{1:n^\prime} \mid n' \geq n \wedge w^\prime_{1:n} = w_{1:n}\}.
\end{align}
Intuitively, $\operatorname{Cyl}(w_{1:n})$ is a string set whose elements have $w_{1:n}$ as a prefix.
Since $\{w_{1:n}\}$ and $ \operatorname{Cyl}(w_{1:n})$ are $\mathcal{F}'$-measurable, we can calculate the probability measure $P'(\{w_{1:n}\})$ and $ P'(\operatorname{Cyl}(w_{1:n}))$.

Then, the language generation process defined below corresponds to the probability distribution $p$.
\begin{equation}
\begin{aligned}
&p_\mathcal{L}(\texttt{<bos>}\mid\varepsilon) = 1, \\
&p_\mathcal{L}(w_{i+1}\mid \texttt{<bos>} w_{1:i}) \\
&= 
\begin{cases}
p_\mathcal{L}^{\mathrm{pos}} & \text{if }P'(\operatorname{Cyl}(w_{1:i})) > 0 \\
p_\mathcal{L}^{\mathrm{null}} & \text{otherwise}
\end{cases},
\end{aligned}
\end{equation}
where
\begin{align}
&\begin{aligned}
&p_\mathcal{L}^{\mathrm{pos}}(w_{i+1}\mid \texttt{<bos>} w_{1:i}) \\
&=\begin{cases}
\begin{aligned}\frac{P'(\{w_{1:i}\})}{P'(\operatorname{Cyl}(w_{1:i}))} \end{aligned}& \text{if }w_{i+1} = \texttt{<eos>} \\
\begin{aligned}\frac{P'(\operatorname{Cyl}(w_{1:i+1}))}{P'(\operatorname{Cyl}(w_{1:i}))}\end{aligned} & \text{otherwise }
\end{cases}, 
\end{aligned} \\
&\begin{aligned}
&p_\mathcal{L}^{\mathrm{null}}(w_{i+1}\mid \texttt{<bos>} w_{1:i}) \\
&=\begin{cases}
1 & \text{if }w_{i+1} = \texttt{<eos>} \\
0 & \text{otherwise }
\end{cases}. 
\end{aligned} \\
\end{align}

This is because, for any $w_{1:n} \in \Sigma^*$ such that $p(w_{1:n}) > 0$, 
\begin{equation}
\begin{aligned}
&p_\mathcal{L}(\texttt{<bos>}w_{1:n}\texttt{<eos>}) \\
&= p_\mathcal{L}(\texttt{<bos>}) \\
& \quad \cdot \prod_{i=1}^n p_\mathcal{L}(w_{i}\mid \texttt{<bos>} w_{1:i-1}) \\
& \quad \cdot p_\mathcal{L}(\texttt{<eos>}\mid \texttt{<bos>} w_{1:n}) \\
&= \prod_{i=1}^n \frac{P'(\operatorname{Cyl}(w_{1:i}))}{P'(\operatorname{Cyl}(w_{1:i-1}))} \\
& \quad \cdot \frac{P'(\{w_{1:n}\})}{P'(\operatorname{Cyl}(w_{1:n}))} \\
&= \frac{P'(\{w_{1:n}\})}{P'(\operatorname{Cyl}(\varepsilon))} \\
&= P'(\{w_{1:n}\}) = p(w_{1:n}).
\end{aligned}
\end{equation}
\end{proof}
\end{proposition}

\section{Proof of Proposition \ref{prop: correspondence of dyck_k language and language generation process}}\label{app: proof of prop 2}

\begin{proposition}[Restated]
For any length $n$ and $\texttt{Dyck}_k$ language generation process $p_{\texttt{Dyck}_k}(\cdot; q, r, \boldsymbol{\pi})$, there exists $\epsilon_n$ such that if $\boldsymbol{\pi} > 0$ then 
\begin{equation}
\begin{aligned}
&p_{\texttt{Dyck}_k}(\texttt{<bos>} w_{1:n} \texttt{<eos>};q, r, \boldsymbol{\pi}) \\
&\begin{cases}
\geq \epsilon_n & \text{if } w_{1:n} \in \texttt{Dyck}_k \\
= 0 & \text{if } w_{1:n} \notin \texttt{Dyck}_k \\
\end{cases}
\end{aligned}
\end{equation}
holds. 

\begin{proof}
When $w_{1:n} \in \texttt{Dyck}_k$, 
\begin{equation}
\begin{aligned}
&p_{\texttt{Dyck}_k}(\texttt{<bos>} w_{1:n} \texttt{<eos>}) \\
&= p_{\texttt{Dyck}_k}(\texttt{<eos>} \mid \texttt{<bos>} w_{1:n}) \\
&\qquad \cdot p_{\texttt{Dyck}_k}(\texttt{<bos>} w_{1:n})\\
& \qquad \vdots \\
&= p_{\texttt{Dyck}_k}(\texttt{<eos>} \mid \texttt{<bos>} w_{1:n}) \\
& \qquad \cdot \left(\prod_{j=1}^n p_{\texttt{Dyck}_k}(w_j \mid \texttt{<bos>} w_{1:j-1}) \right) \\
& \qquad \cdot p_{\texttt{Dyck}_k}(\texttt{<bos>})\\
& \geq (1-r) \cdot  \left(\min \{r, 1-q, q\pi_\mathrm{min}\}\right)^n \, (=: \epsilon_n) 
\end{aligned}
\end{equation}
where $\pi_\mathrm{min} = \min \left(\{\pi_t | 1 \leq t \leq k\}\right)$.

On the other hand, when $w_{1:n} \notin \texttt{Dyck}_k$, either $\operatorname{d}(w_{1:n}) > 0$ or $w_{1:n}$ has some invalid prefixes. If $\operatorname{d}(w_{1:n}) > 0$, $p_{\texttt{Dyck}_k}(\texttt{<eos>} \mid \texttt{<bos>} w_{1:n}) = 0$, indicating $p(\texttt{<bos>} w_{0:n} \texttt{<eos>} ) = 0$. If $w_{1:n}$ has some incorrect prefixes, regarding the shortest prefix $w_{1:j}$, either $w_{j}$ is an invalid closed 
bracket or a token other than brackets. In both cases, $p_{\texttt{Dyck}_k}(w_j \mid \texttt{<bos>} w_{1:j-1}) = 0$ holds, indicating $p_{\texttt{Dyck}_k}(\texttt{<bos>} w_{1:n} \texttt{<eos>} ) = 0$.
\end{proof}
\end{proposition}

\section{Proof of Theorem \ref{theorem: transformers with bos recognize dyck_k}}\label{app: dyck recognition with bos}
In this section, we present a constructive proof that Transformers without positional encoding can recognize the $\texttt{Dyck}_k$ language using the BOS token. We restate Theorem \ref{theorem: transformers with bos recognize dyck_k} for convenience. 

\begin{theorem}[Restated, Transformers with a starting token, $\texttt{Dyck}_k$ recognition]
For all $k$, there exists a 5-layer $O(\log k)$-width causal Transformer without positional encoding that recognizes the $\texttt{Dyck}_k$ language. Each layer incorporates both the residual connection and the layer normalization. 
This network is followed by a fully-connected layer and a sign function to output an acceptance signal.
\end{theorem}

\begin{proof}
As shown in the proof sketch of Theorem \ref{theorem: transformers with bos recognize dyck_k}, each layer performs the following operations. Note that $w_{0:i}$ corresponds to $\texttt{<bos>} w_{1:i}$.

\begin{description}
    \item[First layer] creates pseudo positional encoding $(\cos\phi(i), \sin\phi(i))$ at position $i$, where $\phi(i) = \tan^{-1}(i/\exp(a))$ and $a$ is an attention score on $\texttt{<bos>}$.
    
    \item[Second and third layers] count depth $\operatorname{d}(w_{0:i})$ and $\operatorname{d}(w_{0:i})+1$, respectively. These computations are achieved by constructing a value matrix that outputs $1$ for open brackets and $-1$ for closed brackets in a specific dimension.

    \item[Fourth layer] calculates a propositional variable $Q(w_{0:i})$ as follows: 
    \begin{equation}
    \begin{aligned}
    &Q(w_{0:i}) =\begin{cases}
        \texttt{True} & \text{if } w_{1:i} \in\operatorname{Pre}(\texttt{Dyck}_k) \\
        \texttt{False} & \text{otherwise}
    \end{cases}
    \end{aligned}.
    \end{equation}
    Note that $Q(w_{0:i})$ is guaranteed to return the correct value only when $i=0$ or $w_{1:{i-1}}$ is a prefix for $\texttt{Dyck}_k$. 
    
    \item[Fifth layer] calculates (i) whether $w_{1:n}$ is a prefix for $\texttt{Dyck}_k$ with $\bigwedge_{i=1}^n Q(w_{0:i})$ and (ii) whether $\operatorname{d}(w_{0:i}) = 0$ or not.
\end{description}

We show the specific implementations for each layer in the subsequent subsections.

Note that we explicitly represent the layer number to which each variable or parameter belongs as a superscript. For instance, $W_V^{(2)}$ represents the value matrix that belongs to the second attention layer. In addition, we use concise notation $\operatorname{d}_i$ instead of $\operatorname{d}(w_{0:i})$. Moreover, we frequently use omitted representations for vectors or matrices, where the omitted dimensions of the transformation matrices are zero-padded. For instance, let 
\begin{equation}
    \mathbf{y}_i = \begin{bmatrix}
        \mathbf{a}_i \\
        \mathbf{b}_i \\
        c_i \\
        \mathbf{0}
    \end{bmatrix}
    \begin{matrix*}[l]
    \} d_a \text{ dim.}\\
    \} d_b \text{ dim.}\\
    \} 1 \text{ dim.}\\
    \} d_0 \text{ dim.}
    \end{matrix*}
\end{equation}
be an example of an input vector. In this case, if we use omitted representations 
\begin{align}
     & \mathbf{y}_i = \begin{bmatrix}
        \mathbf{a}_i \\
        \vdots \\
        c_i \\
        \vdots
    \end{bmatrix},\\
    & W = \begin{bmatrix}
    \mathbf{w}_{11}^\top & \cdots & w_{12} & \cdots \\
    W_{21} & \cdots & \mathbf{w}_{22} & \cdots \\
    \vdots & & \vdots & 
    \end{bmatrix}
    \begin{matrix*}[l]
    \} 1 \text{ dim.}\\
    \} d_w \text{ dim.}\\
    \, 
    \end{matrix*}, 
\end{align}
then, the matrix-vector product $W \mathbf{y}_i$ corresponds to the following computation:
\begin{equation}
\begin{aligned}
&W \mathbf{y}_i \\
&= \begin{bmatrix}
\mathbf{w}_{11}^\top & \cdots & w_{12} & \cdots \\
W_{21} & \cdots & \mathbf{w}_{22}  & \cdots \\
\vdots &  & \vdots & 
\end{bmatrix}
\begin{bmatrix}
    \mathbf{a}_i \\
    \vdots \\
    c_i \\
    \vdots
\end{bmatrix} \\
&= \begin{matrix}
\begin{bmatrix}
\makebox[1.8em]{$\mathbf{w}_{11}^\top$} & \makebox[1.8em]{$\mathbf{0}^\top$} & \makebox[1.8em]{$w_{12}$} & \makebox[1.8em]{$\mathbf{0}^\top$}  \\
W_{21}& O & \mathbf{w}_{22} &  O  \\
\end{bmatrix} \\
\begin{matrix}
\overset{\underbrace{\hphantom{W_{21}}}_{d_a \text{ dim.}}}{{\hphantom{W_{21}}}} & 
\overset{\underbrace{\hphantom{\mathbf{0}^\top}}_{d_a \text{ dim.}}}{{\hphantom{\mathbf{0}^\top}}} & 
\overset{\underbrace{\hphantom{\mathbf{w}_{22}}}_{1 \text{ dim.}}}{{\hphantom{\mathbf{w}_{22}}}} &
\overset{\underbrace{\hphantom{\mathbf{0}^\top}}_{d_0 \text{ dim.}}}{{\hphantom{\mathbf{0}^\top}}}
\end{matrix} 
\end{matrix} \begin{bmatrix}
    \mathbf{a}_i \\
    \mathbf{b}_i \\
    c_i \\
    \mathbf{0}
\end{bmatrix}
\begin{matrix*}[l]
    \} d_a \text{ dim.}\\
    \} d_b \text{ dim.}\\
    \} 1 \text{ dim.}\\
    \} d_0 \text{ dim.}
\end{matrix*} \\
& = \begin{bmatrix}
    \mathbf{w}_{11}^\top \mathbf{a}_i + w_{12}c_i \\
    W_{21} \mathbf{a}_i + c_i \mathbf{w}_{22} \\
    \mathbf{0}
\end{bmatrix}\begin{matrix*}[l]
    \} 1 \text{ dim.}\\
    \} d_w \text{ dim.}\\
    \, 
    \end{matrix*}.
\end{aligned}
\end{equation}

\subsection{First layer}\label{app: dyck recognition with bos, subsec: first layer}
In the first layer, the following positional encoding is created.
\begin{equation}
\mathbf{p}_i = 
\begin{bmatrix}
\cos \phi(i) \\
\sin \phi(i)
\end{bmatrix} \in \mathbb{R}^2,
\end{equation}
where $\phi(i) = \tan^{-1} \left(\frac{i}{\exp(a)}\right)$ and $a \in \mathbb{R}$ is a constant.

\subsubsection*{First layer ---Attention layer}
We omit the unnecessary dimensions of input vector $\mathbf{x}_i^{(1)}$ in this layer as follows:
\begin{equation}
\mathbf{x}_i^{(1)} = \begin{bmatrix}
\vdots \\
s_i \\
1 \\
\vdots 
\end{bmatrix}.
\end{equation}
Set the parameters $W_Q^{(1)}, W_K^{(1)}, W_V^{(1)}\in \mathbb{R}^{d_\mathrm{model} \times d_{\mathrm{model}}}$ as follows:
\begin{align}
&W_Q^{(1)}  = 
\begin{bmatrix}
\cdots & 0 & 1 & \cdots \\
 & \vdots & \vdots &
\end{bmatrix} ,\\
& W_K^{(1)} = 
\begin{bmatrix}
\cdots & a & 0 & \cdots \\
 & \vdots & \vdots &
\end{bmatrix},  \\
&W_V^{(1)} =\begin{bmatrix}
 & \vdots & \vdots & \\
\cdots & 1 & 0 & \cdots \\
\cdots & -1 & 1 &\cdots \\
 & \vdots & \vdots &
\end{bmatrix}.
\end{align}

Then, we obtain 
\begin{align}
& W_Q^{(1)} \mathbf{x}_{i_q}^{(1)}= \begin{bmatrix}
1 \\
\vdots
\end{bmatrix},\\
& W_K^{(1)} \mathbf{x}_{i_k}^{(1)}= \begin{bmatrix}
s_{i_k}\cdot a  \\
\vdots
\end{bmatrix},\\
& W_V^{(1)} \mathbf{x}_{i_k}^{(1)}=
\begin{bmatrix}
\vdots \\
s_{i_k}\\
1 - s_{i_k} \\
\vdots
\end{bmatrix},\\
& \left\langle W_K^{(1)} \mathbf{x}_{i_k}^{(1)}, W_Q^{(1)} \mathbf{x}_{i_q}^{(1)}\right\rangle = s_{i_k}\cdot a.
\end{align}

Therefore, $\mathbf{a}^{(1)}_i$ becomes

\begin{equation}
\begin{aligned}
\mathbf{a}^{(1)}_i &=
\frac{\exp(a)}{\exp(a) + i} W_V^{(1)}\mathbf{x}_{0}^{(1)}\\
& \qquad  +  \sum_{j=1}^i \frac{1}{\exp(a) + i}  W_V^{(1)}\mathbf{x}_{j}^{(1)}
\\
&= 
\begin{bmatrix}
\vdots \\
\frac{\exp(a)}{\exp(a) + i}  \\
0  \\
\vdots
\end{bmatrix}
+ \sum_{j=1}^i
\begin{bmatrix}
\vdots \\
0 \\
\frac{1}{\exp(a) + i}  \\
\vdots 
\end{bmatrix} \\
&= 
\begin{bmatrix}
\vdots \\
\frac{\exp(a)}{\exp(a) + i}  \\
\frac{i}{\exp(a) + i}  \\
\vdots 
\end{bmatrix}
\end{aligned}
\end{equation}

Finally, considering the residual connection, we obtain
\begin{equation}
\begin{aligned}
\mathbf{h}_i^{(1)} &= \mathbf{x}_i^{(1)} + \begin{bmatrix}
\vdots \\
\frac{\exp(a)}{\exp(a) + i}  \\
\frac{i}{\exp(a) + i}  \\
\vdots 
\end{bmatrix} \\
& = \begin{bmatrix}
\mathbf{t}_i \\
o_i \\
s_i \\
1 \\ 
\frac{\exp(a)}{\exp(a) + i} \\
\frac{i}{\exp(a) + i}  \\
\mathbf{0} \\
\end{bmatrix}.
\end{aligned}
\end{equation}

\subsubsection*{First layer --- Feed-forward network layer}
We omit the unnecessary dimensions of input vector $\mathbf{h}_i^{(1)}$ in this layer as follows:
\begin{equation}
\mathbf{h}_i^{(1)} = \begin{bmatrix}
\vdots \\
\frac{\exp(a)}{\exp(a) + i} \\
\frac{i}{\exp(a) + i}  \\
\vdots 
\end{bmatrix}.
\end{equation}
Set the parameters $W_1^{(1)}, W_2^{(1)}\in \mathbb{R}^{d_\mathrm{model} \times d_{\mathrm{model}}}$ and $\boldsymbol{\beta}^{(1)}, \boldsymbol{\gamma}^{(1)} \in \mathbb{R}^{d_\mathrm{model}}$ as follows:
\begin{align}
&W_1^{(1)} = \begin{bmatrix}
\cdots & 1 & 0 & \cdots \\
\cdots & 0 & 1 & \cdots \\
\cdots & \mathbf{0} & \mathbf{0} & \cdots 
\end{bmatrix}, \\
&W_2^{(1)} = \begin{bmatrix}
\vdots & \vdots & \vdots \\
1 & 0 & \mathbf{0}^\top \\
0 & 1 & \mathbf{0}^\top \\
\vdots & \vdots & \vdots
\end{bmatrix}, \\
&\boldsymbol{\beta}^{(1)} = \mathbf{0}, \\
&\boldsymbol{\gamma}^{(1)} = \sqrt{\frac{1}{d_\mathrm{model}}}\mathbf{1}.
\end{align}

Then, the output of the FFN becomes
\begin{equation}
\begin{aligned}
&W_2^{(1)} \left[ \operatorname{LN}_{\mathrm{RMS}} \left(W_1^{(1)} \mathbf{h}_i^{(1)}\right)\right]_+ \\
&= W_2^{(1)} \left[ \operatorname{LN}_{\mathrm{RMS}} \left(\begin{bmatrix}
\frac{\exp(a)}{\exp(a) + i}  \\
\frac{i}{\exp(a) + i}  \\
\mathbf{0}
\end{bmatrix}\right)\right]_+ \\
&= W_2^{(1)}
\begin{bmatrix}
\frac{\exp(a)}{\sqrt{\exp(a)^2 + i^2}}  \\
\frac{i}{\sqrt{\exp(a)^2 + i^2}}  \\
\mathbf{0}
\end{bmatrix}_+ \\
&=\begin{bmatrix}
\vdots & \vdots & \vdots \\
1 & 0 & \mathbf{0}^\top \\
0 & 1 & \mathbf{0}^\top \\
\vdots & \vdots & \vdots
\end{bmatrix}   
\begin{bmatrix}
\cos \phi(i)\\
\sin \phi(i)\\
\mathbf{0}
\end{bmatrix}  \\
&\,\,\,\,\, \left(\text{ from } \frac{\sin \phi(i)}{\cos \phi(i)} = \tan \phi(i) = \frac{i}{\exp(a)}\right)\\
&= \begin{bmatrix}
\vdots \\
\cos \phi(i)\\
\sin \phi(i)\\
\vdots 
\end{bmatrix}.
\end{aligned}
\end{equation}

Finally, considering the residual connection, we obtain
\begin{equation}
\begin{aligned}
\mathbf{x}_i^{(2)}& = \mathbf{h}_i^{(1)} +  \begin{bmatrix}
\vdots \\
\cos \phi(i)\\
\sin \phi(i)\\
\vdots 
\end{bmatrix} \\
& = \begin{bmatrix}
\mathbf{t}_i \\
o_i \\
s_i \\
1 \\ 
\frac{\exp(a)}{\exp(a) + i} \\
\frac{i}{\exp(a) + i}  \\
\cos \phi(i)\\
\sin \phi(i)\\
\vdots 
\end{bmatrix}.
\end{aligned}
\end{equation}

\subsection{Second layer}\label{app: dyck recognition with bos, subsec: second layer}
In the second layer, the following vector that indicates the depth $\operatorname{d}(w_{0:i})$ is calculated:
\begin{equation}
\begin{bmatrix}
\cos \theta(\operatorname{d}(w_{0:i})) \\
\sin \theta(\operatorname{d}(w_{0:i}))
\end{bmatrix},
\end{equation}
where $\theta(\operatorname{d}) = \tan^{-1}(\operatorname{d} / \exp(a))$.

\subsubsection*{Second layer --- Attention layer}
We omit the unnecessary dimensions of input vector $\mathbf{x}_i^{(2)}$ in this layer as follows:
\begin{equation}
\mathbf{x}_i^{(2)} = \begin{bmatrix}
\vdots \\
o_i \\
s_i \\
1 \\ 
\vdots
\end{bmatrix}.
\end{equation}

Set the parameters $W_Q^{(2)}, W_K^{(2)}, W_V^{(2)}\in \mathbb{R}^{d_\mathrm{model} \times d_{\mathrm{model}}}$ as follows:
\begin{align}
W_Q^{(2)} &= 
\begin{bmatrix}
\cdots & 0 & 0 & 1 & \cdots \\
 & \vdots & \vdots & \vdots &
\end{bmatrix}, \\
W_K^{(2)} &= 
\begin{bmatrix}
\cdots & 0 & a & 0 & \cdots \\
 & \vdots & \vdots & \vdots &
\end{bmatrix}, \\
W_V^{(2)} &= 
\begin{bmatrix}
 & \vdots & \vdots & \vdots &  \\
\cdots & 0 & 1 & 0 & \cdots \\
\cdots & 1 & 0 & 0 & \cdots \\
 & \vdots & \vdots & \vdots &
\end{bmatrix}.
\end{align}
Then, we obtain 
\begin{align}
& W_Q^{(2)} \mathbf{x}_{i_q}^{(2)}= \begin{bmatrix}
1 \\
\vdots 
\end{bmatrix},\\
& W_K^{(2)} \mathbf{x}_{i_k}^{(2)}= \begin{bmatrix}
s_{i_k}\cdot a  \\
\vdots 
\end{bmatrix},\\
& W_V^{(2)} \mathbf{x}_{i_k}^{(2)}=
\begin{bmatrix}
\vdots  \\
s_{i_k} \\
o_{i_k} \\
\vdots 
\end{bmatrix}, \\
& \left\langle W_K^{(2)} \mathbf{x}_{i_k}^{(2)}, W_Q^{(2)} \mathbf{x}_{i_q}^{(2)}\right\rangle = s_{i_k}\cdot a.
\end{align}

Therefore, the output of the attention layer $\mathbf{a}^{(2)}_i$ becomes

\begin{equation}
\begin{aligned}
\mathbf{a}^{(2)}_i &=
\frac{\exp(a)}{\exp(a) + i} W_V^{(2)}\mathbf{x}_{0}^{(2)}\\
& \qquad  +  \sum_{j=1}^i \frac{1}{\exp(a) + i}  W_V^{(2)}\mathbf{x}_{j}^{(2)}
\\
&= 
\begin{bmatrix}
\vdots \\
\frac{\exp(a)}{\exp(a) + i}  \\
0  \\
\vdots 
\end{bmatrix}
+ 
\sum_{j=1}^i\begin{bmatrix}
\vdots  \\
0 \\
\frac{o_i}{\exp(a) + i}  \\
\vdots 
\end{bmatrix} \\
&= 
\begin{bmatrix}
\vdots \\
\frac{\exp(a)}{\exp(a) + i}  \\
\frac{\operatorname{d}_i}{\exp(a) + i}  \\
\vdots 
\end{bmatrix}
\end{aligned}
\end{equation}

Finally, considering the residual connection, we obtain
\begin{equation}
\begin{aligned}
\mathbf{h}_i^{(2)} &= \mathbf{x}_i^{(2)} + \begin{bmatrix}
\vdots \\
\frac{\exp(a)}{\exp(a) + i}  \\
\frac{\operatorname{d}_i}{\exp(a) + i}  \\
\vdots 
\end{bmatrix} \\
& =  \begin{bmatrix}
\mathbf{t}_i \\
o_i \\
s_i \\
1 \\ 
\vdots \\
\cos \phi(i)\\
\sin \phi(i)\\
\frac{\exp(a)}{\exp(a) + i}  \\
\frac{\operatorname{d}_i}{\exp(a) + i}  \\
\mathbf{0} \\
\end{bmatrix}
\end{aligned}
\end{equation}

\subsubsection*{Second layer --- Feed-forward network layer}
We omit the unnecessary dimensions of input vector $\mathbf{h}_i^{(2)}$ in this layer as follows:
\begin{equation}
\mathbf{h}_i^{(2)} = \begin{bmatrix}
\vdots \\
\frac{\exp(a)}{\exp(a) + i} \\
\frac{\operatorname{d}_i}{\exp(a) + i}  \\
\vdots 
\end{bmatrix}.
\end{equation}

Set the parameters $W_1^{(2)}, W_2^{(2)}\in \mathbb{R}^{d_\mathrm{model} \times d_{\mathrm{model}}}$ and $\boldsymbol{\beta}^{(2)}, \boldsymbol{\gamma}^{(2)} \in \mathbb{R}^{d_\mathrm{model}}$ as follows:
\begin{align}
W_1^{(2)} &= 
\begin{bmatrix}
\cdots & 1 & 0 & \cdots \\
\cdots & -1 & 0 & \cdots \\
\cdots & 0 & 1 & \cdots \\
\cdots & 0 & -1 & \cdots \\
& \vdots & \vdots & 
\end{bmatrix},\\
W_2^{(2)} &= 
\begin{bmatrix}
\vdots & \vdots & \vdots & \vdots &  \\
1 & 0 & 0 & 0 & \cdots \\
0 & 0 & 1 & -1 & \cdots \\
\vdots & \vdots & \vdots & \vdots &
\end{bmatrix},\\
\boldsymbol{\beta}^{(2)} &= \mathbf{0}, \\
\boldsymbol{\gamma}^{(2)} &= \sqrt{\frac{2}{d_\mathrm{model}}}\mathbf{1}.
\end{align}
Then, the output of the feed-forward network becomes
\begin{equation}
\begin{aligned}
& W_2^{(2)} \left[ \operatorname{LN}_{\mathrm{RMS}} \left(W_1^{(2)} \mathbf{h}_i^{(2)}\right)\right]_+ \\
&= W_2^{(2)} \left[ \operatorname{LN}_{\mathrm{RMS}} \left(\begin{bmatrix}
\frac{\exp(a)}{\exp(a) + i}  \\
-\frac{\exp(a)}{\exp(a) + i}  \\
\frac{\operatorname{d}_i}{\exp(a) + i}  \\
-\frac{\operatorname{d}_i}{\exp(a) + i}  \\
\vdots \\
\end{bmatrix}\right)\right]_+ \\
&= \begin{bmatrix}
\vdots & \vdots & \vdots & \vdots &  \\
1 & 0 & 0 & 0 & \cdots \\
0 & 0 & 1 & -1 & \cdots \\
\vdots & \vdots & \vdots & \vdots &
\end{bmatrix}
\begin{bmatrix}
\frac{\exp(a)}{\sqrt{\operatorname{d}_i^2 + \exp(a)^2}}  \\
-\frac{\exp(a)}{\sqrt{\operatorname{d}_i^2 + \exp(a)^2}}  \\
\frac{\operatorname{d}_i}{\sqrt{\operatorname{d}_i^2 + \exp(a)^2}}  \\
-\frac{\operatorname{d}_i}{\sqrt{\operatorname{d}_i^2 + \exp(a)^2}}  \\
\vdots \\
\end{bmatrix}_+ \\
&= \begin{bmatrix}
\vdots \\
\cos \theta(\operatorname{d}_i)    \\
\left[\sin \theta(\operatorname{d}_i) \right]_+ - \left[-\sin \theta(\operatorname{d}_i) \right]_+  \\
\vdots \\
\end{bmatrix}  \\
&\left(\text{ from } \frac{\sin \theta(\operatorname{d}_i)}{\cos \theta(\operatorname{d}_i)} = \tan \theta(\operatorname{d}_i) = \frac{\operatorname{d}_i}{\exp(a)}\right)\\
&=  \begin{bmatrix}
\vdots \\
\cos \theta(\operatorname{d}_i)   \\
\sin \theta(\operatorname{d}_i)  \\
\vdots
\end{bmatrix}
\end{aligned}
\end{equation}

Finally, considering the residual connection, we obtain
\begin{equation}
\begin{aligned}
\mathbf{x}_i^{(3)} &= \mathbf{h}_i^{(2)} + \begin{bmatrix}
\vdots \\
\cos \theta(\operatorname{d}_i)  \\
\sin \theta(\operatorname{d}_i)  \\
\vdots
\end{bmatrix} \\
& =\begin{bmatrix}
\mathbf{t}_i \\
o_i \\
s_i \\
1 \\ 
\vdots \\
\cos \phi(i)\\
\sin \phi(i)\\
\vdots \\
\cos \theta(\operatorname{d}_i)\\
\sin \theta(\operatorname{d}_i)\\
\mathbf{0} \\
\end{bmatrix} .
\end{aligned}
\end{equation}

\subsection{Third layer}\label{app: dyck recognition with bos, subsec: third layer}
The third layer counts depth $\operatorname{d}(w_{0:i}) + 1$ in addition to $\operatorname{d}(w_{0:i})$ that is counted in the second layer. This is because the depth of the closed bracket is smaller by $1$ than the corresponding open bracket. For instance, the depths calculated for $``\langle_1\rangle_1"$ are $1$ for $``\langle_1"$ and $0$ for $``\rangle_1"$.

The way to construct parameters is largely the same as that of the second layer. Specifically, we slightly modify the value matrix: we use
\begin{align}
W_V^{(3)} &= 
\begin{bmatrix}
 & \vdots & \vdots & \vdots &  \\
\cdots & 0 & 1 & 0 & \cdots \\
\cdots & 1 & \exp(-a) & 0 & \cdots \\
 & \vdots & \vdots & \vdots &   
\end{bmatrix}
\end{align}
instead of 
\begin{align}
W_V^{(2)} &= 
\begin{bmatrix}
 & \vdots & \vdots & \vdots &  \\
\cdots & 0 & 1 & 0 & \cdots \\
\cdots & 1 & 0 & 0 & \cdots \\
 & \vdots & \vdots & \vdots &   
\end{bmatrix}.
\end{align}

Then, we obtain

\begin{equation}
\begin{aligned}
\mathbf{a}^{(3)}_i &=
\frac{\exp(a)}{\exp(a) + i} W_V^{(3)}\mathbf{x}_{0}^{(3)}\\
& \qquad  +  \sum_{j=1}^i \frac{1}{\exp(a) + i}  W_V^{(3)}\mathbf{x}_{j}^{(3)}
\\
&= 
\begin{bmatrix}
\vdots \\
\frac{\exp(a)}{\exp(a) + i}  \\
\frac{1}{\exp(a) + i}  \\
\vdots 
\end{bmatrix}
+ \sum_{j=1}^i
\begin{bmatrix}
\vdots  \\
0 \\
\frac{o_i}{\exp(a) + i}  \\
\vdots 
\end{bmatrix} \\
&= 
\begin{bmatrix}
\vdots  \\
\frac{\exp(a)}{\exp(a) + i}  \\
\frac{\operatorname{d}_i+1}{\exp(a) + i}  \\
\vdots 
\end{bmatrix}.
\end{aligned}
\end{equation}

Therefore, using the subsequent feed-forward network layer, we obtain
\begin{equation}
\begin{aligned}
\mathbf{x}_i^{(4)} =
\begin{bmatrix}
\mathbf{t}_i \\
o_i \\
s_i \\
1 \\ 
\vdots \\
\cos \phi(i)\\
\sin \phi(i)\\
\vdots \\ 
\cos \theta(\operatorname{d}_i)\\
\sin \theta(\operatorname{d}_i)\\
\vdots \\ 
\cos \theta(\operatorname{d}_i+1)\\
\sin \theta(\operatorname{d}_i+1)\\
\mathbf{0} \\
\end{bmatrix}.
\end{aligned}
\end{equation}

\subsection{Fourth layer}\label{app: dyck recognition with bos, subsec: fourth layer}

The last two layers determine whether the input string belongs to the $\texttt{Dyck}_k$ language, leveraging the position vectors and depth vectors computed so far. Note that the necessary and sufficient condition for a string $w_{1:n}$ to belong to the $\texttt{Dyck}_k$ language is that the following two conditions are simultaneously satisfied.
\begin{description}
    \item[Condition (i)] $w_{1:n} \in \operatorname{Pre}(\texttt{Dyck}_k)$.
    \item[Condition (ii)] $\operatorname{d}(w_{1:n}) = 0$.
\end{description}

We can check \textbf{Condition (i)} by calculating $\bigwedge_{i=0}^n Q(w_{0:i})$, where $Q(w_{0:i})$ is a propositional variable that is guaranteed to return the correct values only if $i=0$ or $w_{0:i-1} \in \operatorname{Pre}(\texttt{Dyck}_k)$. Specifically, 
\begin{equation}
Q(w_{0:i}) =\begin{cases}
    \texttt{True} & \text{if } w_{1:i} \in\operatorname{Pre}(\texttt{Dyck}_k) \\
    \texttt{False} & \text{otherwise}
\end{cases}.
\end{equation}
If $\bigwedge_{i=0}^n Q(w_{0:i}) = \texttt{True}$ --- for all $i$, $Q(w_{0:i}) = \texttt{True}$ --- all propositional variables are guaranteed to return the correct values, indicating $w_{1:n}$ is a prefix for $\texttt{Dyck}_k$. Otherwise, among the propositional variables that return $\texttt{False}$, the propositional variable at the smallest index $j$ is guaranteed to return the correct value because all preceding variables return $\texttt{True}$, indicating that $w_{1:n}$ is not a prefix for $\texttt{Dyck}_k$. In contrast, \textbf{Condition (ii)} can be easily checked using $\sin \theta(\operatorname{d}_i)$.

Therefore, the fourth layer calculates the value that corresponds to the propositional variable $Q(w_{0:i})$.

\subsubsection*{Fourth layer --- Attention layer}
In the attention layer, each closed bracket at position $i$ fetches the bracket type $\mathbf{t}$ at the largest index among $\{0\} \cup \{j \leq i \mid o_j =1 \wedge \operatorname{d}_j = \operatorname{d}_i + 1\}$. 

Before presenting the specific parameters, we first outline the method for calculating the attention scores in two steps: (i) assign high attention scores to the indices  $\{0\} \cup \{j \leq i \mid o_j =1 \wedge \operatorname{d}_j = \operatorname{d}_i + 1\}$; that is, extract a starting token and depth-matched open brackets and (ii) within those tokens, assign higher attention scores to tokens closer to the query, thereby focusing on the token with the largest index. Figure \ref{fig:attention score} illustrates this calculation, where the first step corresponds to the term $\operatorname{T}^{\mathrm{depth}}$ and the second step corresponds to the term $\operatorname{T}^{\mathrm{pos}}$.

We then show the specific parameters that achieve the desired operation. 
We omit the unnecessary dimensions of input vector $\mathbf{x}_i^{(4)}$ in this layer as follows:
\begin{equation}
\begin{aligned}
\mathbf{x}_i^{(4)} =
\begin{bmatrix}
\mathbf{t}_i \\
o_i \\
s_i \\
1 \\ 
\vdots \\
\cos \phi(i)\\
\sin \phi(i)\\
\vdots \\ 
\cos \theta(\operatorname{d}_i)\\
\sin \theta(\operatorname{d}_i)\\
\vdots \\ 
\cos \theta(\operatorname{d}_i+1)\\
\sin \theta(\operatorname{d}_i+1)\\
\vdots
\end{bmatrix}.
\end{aligned}
\end{equation}

Set the parameters $W_Q^{(4)}, W_K^{(4)}, W_V^{(4)}\in \mathbb{R}^{d_\mathrm{model} \times d_{\mathrm{model}}}$ as follows (Note that in some cases, the transposed matrices are described to accommodate the limited space):
\begin{align}
&W_Q^{(4)}= \begin{bmatrix}
C_1^{(4)} W_Q^{ \mathrm{depth}} \\
W_Q^{\mathrm{pos}} \\
C_1^{(4)} \mathbf{w}_Q^{ \mathrm{open}\top} \\
\vdots 
\end{bmatrix}, \\
&W_K^{(4)}= \begin{bmatrix}
C_2^{(4)} W_K^{ \mathrm{depth}} \\
C_2^{(4)}W_K^{\mathrm{pos}} \\
C_2^{(4)}\mathbf{w}_K^{ \mathrm{open}\top} \\
\vdots 
\end{bmatrix}, \\
&W_V^{(4)\top}= \begin{bmatrix}
\cdots & I & \cdots \\
\cdots & \mathbf{0}^\top & \cdots \\
\cdots & \mathbf{0}^\top & \cdots \\
\cdots & \mathbf{0}^\top & \cdots \\
 & \vdots & \\
\cdots & \mathbf{0}^\top & \cdots \\
\cdots & \mathbf{0}^\top & \cdots \\
 & \vdots & \\
\cdots & \mathbf{0}^\top & \cdots \\
\cdots & \mathbf{0}^\top & \cdots \\
 & \vdots & \\
\cdots & \mathbf{0}^\top & \cdots \\
\cdots & \mathbf{0}^\top & \cdots \\
 & \vdots & \\
\end{bmatrix},
\end{align}
where $C_1^{(4)}$ and $C_2^{(4)}$ are positive constants, 
\begin{align}
&W_Q^{\mathrm{depth}\top} = \begin{bmatrix}
\mathbf{0} & \mathbf{0} & \mathbf{0} & \mathbf{0} \\
0 & 0 & 0 & 0 \\
0 & 0 & 0 & 0 \\
0 & 0 & 1 & 1 \\
\vdots & \vdots & \vdots & \vdots \\
0 & 0 & 0 & 0 \\
0 & 0 & 0 & 0 \\
\vdots & \vdots & \vdots  & \vdots \\
0 & 0 & 0 & 0 \\
0 & 0 & 0 & 0 \\
\vdots & \vdots & \vdots & \vdots \\
1 & 0 & -1 & 0 \\
0 & 1 & 0 & 0 \\
\vdots & \vdots & \vdots & \vdots 
\end{bmatrix}, \\
&W_Q^{\mathrm{pos}\top} = \begin{bmatrix}
\mathbf{0} & \mathbf{0} \\
0 & 0 \\
0 & 0 \\
0 & 0 \\
\vdots & \vdots \\
0 & 1 \\
1 & 0 \\
\vdots & \vdots  \\
0 & 0 \\
0 & 0\\
\vdots & \vdots \\
0 & 0 \\
0 & 0 \\
\vdots & \vdots 
\end{bmatrix}, 
\mathbf{w}_Q^{\mathrm{open}} = \begin{bmatrix}
\mathbf{0} \\
1 \\
-1 \\
1 \\
\vdots \\
0 \\
0 \\
\vdots  \\
0 \\
0\\
\vdots \\
0 \\
0 \\
\vdots 
\end{bmatrix}, \\
&W_K^{ \mathrm{depth}\top} = \begin{bmatrix}
\mathbf{0} & \mathbf{0} & \mathbf{0} & \mathbf{0} \\
0 & 0 & 0 & 1 \\
0 & 0 & 1 & 1 \\
0 & 0 & 0 & -1 \\
\vdots & \vdots & \vdots & \vdots \\
0 & 0 & 0 & 0 \\
0 & 0 & 0 & 0 \\
\vdots & \vdots & \vdots  & \vdots \\
1 & 0 & 0 & 0 \\
0 & 1 & 0 & 0 \\
\vdots & \vdots & \vdots & \vdots \\
0 & 0 & 0 & 0 \\
0 & 0 & 0 & 0 \\
\vdots & \vdots & \vdots & \vdots 
\end{bmatrix}, \\
&W_K^{\mathrm{pos}\top} = \begin{bmatrix}
\mathbf{0} & \mathbf{0} \\
0 & 0 \\
0 & 0 \\
0 & 0 \\
\vdots & \vdots \\
1 & 0 \\
0 & 1 \\
\vdots & \vdots  \\
0 & 0 \\
0 & 0\\
\vdots & \vdots \\
0 & 0 \\
0 & 0 \\
\vdots & \vdots 
\end{bmatrix}, 
\mathbf{w}_K^{\mathrm{open}} = \begin{bmatrix}
\mathbf{0} \\
0 \\
1 \\
0 \\
\vdots \\
0 \\
0 \\
\vdots  \\
0 \\
0\\
\vdots \\
0 \\
0 \\
\vdots 
\end{bmatrix}.
\end{align}

Then, we obtain
\begin{align}
& W_Q^{(4)} \mathbf{x}_{i_q}^{(4)}= 
\begin{bmatrix}
C_1^{(4)}W_Q^{ \mathrm{depth}} \mathbf{x}_{i_q}^{(4)} \\
W_Q^{ \mathrm{pos}} \mathbf{x}_{i_q}^{(4)} \\
C_1^{(4)}\mathbf{w}_Q^{\mathrm{open}\top} \mathbf{x}_{i_q}^{(4)} \\
\vdots 
\end{bmatrix}, \\
& W_K^{(4)} \mathbf{x}_{i_k}^{(4)} =
\begin{bmatrix}
C_2^{(4)}W_K^{ \mathrm{depth}} \mathbf{x}_{i_k}^{(4)} \\
C_2^{(4)}W_K^{ \mathrm{pos}} \mathbf{x}_{i_k}^{(4)} \\
C_2^{(4)}\mathbf{w}_K^{\mathrm{open}\top} \mathbf{x}_{i_k}^{(4)} \\
\vdots 
\end{bmatrix}, \\
& W_V^{(4)} \mathbf{x}_{i_k}^{(4)}=
\begin{bmatrix}
\vdots \\
\mathbf{t}_{i_k}\\
\vdots
\end{bmatrix}, \\
&\begin{aligned}
&\left\langle W_K^{(4)} \mathbf{x}_{i_k}^{(4)}, W_Q^{(4)} \mathbf{x}_{i_q}^{(4)}\right\rangle \\
&=C_2^{(4)}C_1^{(4)}\left\langle W_K^{\mathrm{depth}} \mathbf{x}_{i_k}^{(4)}, W_Q^{\mathrm{depth}} \mathbf{x}_{i_q}^{(4)}\right\rangle\\
&\quad + C_2^{(4)}\left\langle W_K^{\mathrm{pos}} \mathbf{x}_{i_k}^{(4)}, W_Q^{\mathrm{pos}} \mathbf{x}_{i_q}^{(4)}\right\rangle\\
&\quad + C_2^{(4)}C_1^{(4)}\mathbf{w}_K^{\mathrm{open}\top} \mathbf{x}_{i_k}^{(4)}\cdot \mathbf{w}_Q^{\mathrm{open}\top} \mathbf{x}_{i_q}^{(4)} \\
&=C_2^{(4)}\left(C_1^{(4)} \operatorname{T}^\mathrm{depth}_{i_q, i_k} + \operatorname{T}^\mathrm{pos}_{i_q, i_k} + C_1^{(4)}\operatorname{T}^\mathrm{open}_{i_q, i_k}\right),
\end{aligned}
\end{align}
where
\begin{align}
& \begin{aligned}
\operatorname{T}^\mathrm{depth}_{i_q, i_k} &= \left\langle W_K^{\mathrm{depth}} \mathbf{x}_{i_k}^{(4)}, W_Q^{\mathrm{depth}} \mathbf{x}_{i_q}^{(4)}\right\rangle\\
&=\cos(\theta(\operatorname{d}_{i_q} + 1)-\theta(\operatorname{d}_{i_k})) \\
& \quad + \left(1 - \cos\theta (\operatorname{d}_{i_q}+ 1))\right)\cdot s_{i_k} \\
& \quad + \left(o_{i_k} + s_{i_k} - 1\right) \\
& \begin{cases}
= 1 & \text{if } w_{i_k} = ``\texttt{<bos>}" \\
= 1 & \begin{aligned}\text{if } & w_{i_k} = ``\langle_\cdot" \, \\
&\wedge \,  d_{i_q} + 1 = d_{i_k} \end{aligned}\\
< 1 & \text{otherwise} 
\end{cases},
\end{aligned} \\
& \begin{aligned}
\operatorname{T}^\mathrm{pos}_{i_q, i_k} &= \left\langle W_K^{\mathrm{pos}} \mathbf{x}_{i_k}^{(4)}, W_Q^{\mathrm{pos}} \mathbf{x}_{i_q}^{(4)}\right\rangle\\
&= - \sin(\phi(i_q) - \phi(i_k)),
\end{aligned} \\
& \begin{aligned}
\operatorname{T}^\mathrm{open}_{i_q, i_k} &= \mathbf{w}_K^{\mathrm{open}\top} \mathbf{x}_{i_k}^{(4)}\cdot \mathbf{w}_Q^{\mathrm{open}\top} \mathbf{x}_{i_q}^{(4)} \\
&=(o_{i_q} - s_{i_q} + 1)\cdot s_{i_k}\\
&=\begin{cases}
2 & \begin{aligned}\text{if }&w_{i_q} = ``\rangle_\cdot" \, \\
&\, \wedge \, w_{i_k} = ``\texttt{<bos>}"\end{aligned}\\
0 & \text{otherwise}
\end{cases}.
\end{aligned}
\end{align}

Intuitively, $\operatorname{T}^\mathrm{depth}_{i_q, i_k}$ is a term that extracts the depth-matched open brackets and the BOS token, and $\operatorname{T}^\mathrm{pos}_{i_q, i_k}$ is a term that extracts the nearest token among them. Moreover, $\operatorname{T}^\mathrm{open}_{i_q, i_k}$ is a term that makes the query focus on the starting token only when the query is an open bracket. For example, the query $``\rangle_3"$ in the input string $``\texttt{<bos>} \langle_2 \langle_1 \rangle_1 \rangle_2 \langle_3 \rangle_3"$ fetches the nearest depth-matched open bracket $``\langle_3"$ as shown in Figure \ref{fig:attention score}.

\begin{figure}[h]
  \includegraphics[width=\columnwidth]{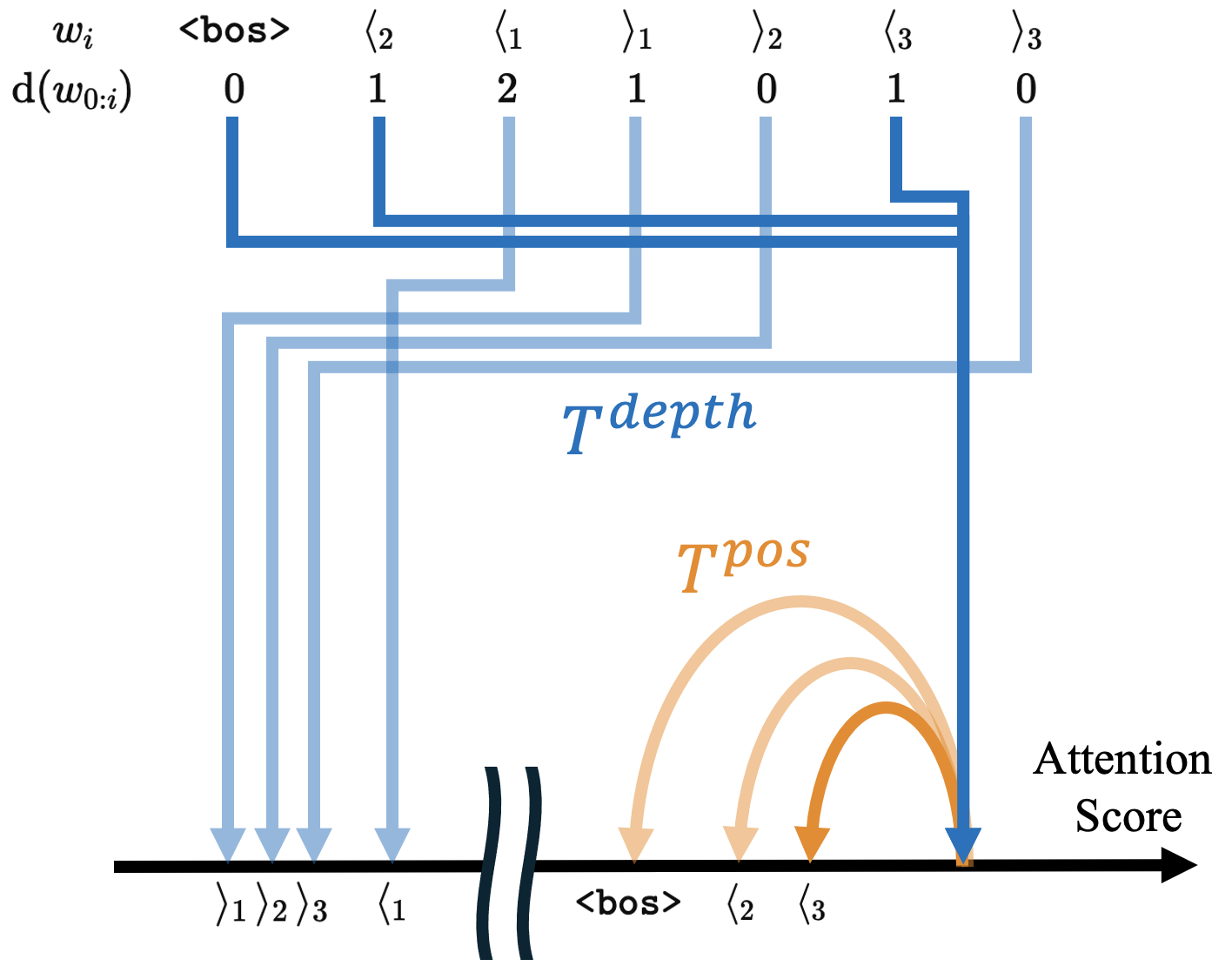}
  \caption{Illustration of the process where the query $``\rangle_3"$ in the input string $``\texttt{<bos>} \langle_2 \langle_1 \rangle_1 \rangle_2 \langle_3 \rangle_3"$ fetches the nearest depth-matched open bracket $``\langle_3"$. At first, using $\operatorname{T}^{\mathrm{depth}}$, only the depth-matched open brackets and $\texttt{<bos>}$ are extracted, and then, using $\operatorname{T}^{\mathrm{pos}}$, the nearest one among them is extracted.}
  \label{fig:attention score}
\end{figure}

Therefore, when the query is a closed bracket, given a sufficiently large constant $C_1^{(4)}$ that satisfies $C_1^{(4)}\cos(\theta(\operatorname{d}_{i_q}) - \theta(\operatorname{d}_{i_k})) > 1$ if $\operatorname{d}_{i_q} \neq \operatorname{d}_{i_k}$, 
\begin{align}
& \begin{aligned}
&\frac{1}{C_2^{(4)}}\left\langle W_K^{(4)} \mathbf{x}_{i_k}^{(4)}, W_Q^{(4)} \mathbf{x}_{i_q}^{(4)}\right\rangle \\
&= C_1^{(4)} \operatorname{T}^\mathrm{depth}_{i_q, i_k} + \operatorname{T}^\mathrm{pos}_{i_q, i_k} \\
&\begin{cases}
= C_1^{(4)} - \sin(\phi({i_q}) - \phi({i_k})) \\
\qquad \text{if } w_{i_k} = ``\langle_\cdot" \, \wedge \, \operatorname{d}_{i_q} + 1 = \operatorname{d}_{i_k}\\
< C_1^{(4)} - 1 \\
\qquad \text{otherwise}
\end{cases}
\end{aligned}
\end{align}
holds, indicating that given a sufficiently large constant $C_2^{(4)}$, the query can focus on the nearest token among the BOS token and depth-matched open brackets.

On the other hand, when the query is an open bracket, 
\begin{align}
& \begin{aligned}
&\frac{1}{C_2^{(4)}}\left\langle W_K^{(4)} \mathbf{x}_{i_k}^{(4)}, W_Q^{(4)} \mathbf{x}_{i_q}^{(4)}\right\rangle \\
&= \begin{cases}
C_1^{(4)} \operatorname{T}^\mathrm{depth}_{i_q, i_k} + \operatorname{T}^\mathrm{pos}_{i_q, i_k} + 2C_1^{(4)}  \\
\qquad \qquad \text{if } w_{i_k} = ``\texttt{<bos>}"\\
C_1^{(4)} \operatorname{T}^\mathrm{depth}_{i_q, i_k} + \operatorname{T}^\mathrm{pos}_{i_q, i_k}  \\
\qquad \qquad \text{otherwise } 
\end{cases} \\
&\begin{cases}
\geq 3C_1^{(4)} - 1 & \text{if } w_{i_k} = ``\texttt{<bos>}" \\
\leq C_1^{(4)} & \text{otherwise } 
\end{cases}
\end{aligned}
\end{align}
holds, indicating that given a sufficiently large constant $C_1^{(4)}$, the query can focus on the BOS token.

From the above, it is confirmed that the desired operations are performed correctly.

Thus, the output of the attention layer $\mathbf{a}^{(4)}_i$ becomes
\begin{equation}
\begin{aligned}
\mathbf{a}^{(4)}_i &= \sum_{j=0}^i \frac{1}{i+1}  W_V^{(4)}\mathbf{x}_{j}^{(4)}\\
&= 
\begin{bmatrix}
\vdots\\
\tilde{\mathbf{t}}_{i} \\
\vdots
\end{bmatrix},
\end{aligned}
\end{equation}
where $\tilde{\mathbf{t}}_{i}$ is the bracket-type vector $\overline{\mathbf{t}}_{i}$ of the nearest depth-matched open bracket when $o_i = -1$ and $w_i$ contains one or more such brackets; otherwise, it is set to the zero vector $\mathbf{0}$.
Here, we treat softmax attention as hardmax attention for simplicity. However, as in Appendix \ref{app: validity of treating softmax as hardmax}, it is sufficient if the attention allocated to the target token exceeds $\frac{2}{3}$ in practice.

Finally, considering the residual connection, we obtain
\begin{equation}
\begin{aligned}
\mathbf{h}_i^{(4)} &= \mathbf{x}_i^{(4)} + \begin{bmatrix}
\vdots\\
\tilde{\mathbf{t}}_{i} \\
\vdots
\end{bmatrix} \\
& = \begin{bmatrix}
\mathbf{t}_i \\
o_i \\
s_i \\
1 \\ 
\vdots \\
\cos \phi(i)\\
\sin \phi(i)\\
\vdots \\ 
\cos \theta(\operatorname{d}_i)\\
\sin \theta(\operatorname{d}_i)\\
\vdots \\ 
\cos \theta(\operatorname{d}_i+1)\\
\sin \theta(\operatorname{d}_i+1)\\
\tilde{\mathbf{t}}_{i} \\
\mathbf{0} 
\end{bmatrix}.
\end{aligned}
\end{equation}

\subsubsection*{Fourth layer --- Feed-forward network layer}
In this layer, the objective is to compute $\operatorname{q}(w_{0:i})$, where $\operatorname{q}(w_{0:i})$ is positive when $Q(w_{0:i})$ is $\texttt{True}$ and negative when $Q(w_{0:i})$ is $\texttt{False}$. Specifically, 
\begin{equation}\label{eq: conditions for q_i}
\operatorname{q}(w_{0:i}) \begin{cases}
> 1 & \text{if } o_i \neq -1 \\
> 1 & \text{if } o_i = -1 \wedge \mathbf{t}_i = \tilde{\mathbf{t}}_i \\
< -1 & \text{if } o_i = -1 \wedge \mathbf{t}_i \neq \tilde{\mathbf{t}}_i 
\end{cases}.
\end{equation}

In the following proof, we use concise notation $\operatorname{q}_i$ instead of $\operatorname{q}(w_{0:i})$ and we omit the unnecessary dimensions of input vector $\mathbf{h}_i^{(4)}$ in this layer as follows:
\begin{equation}
\mathbf{h}_i^{(4)} = \begin{bmatrix}
\mathbf{t}_i \\
o_i \\
\vdots \\
1 \\
\vdots \\
\tilde{\mathbf{t}}_i \\
\vdots 
\end{bmatrix}.
\end{equation}

Set the parameters $W_1^{(4)}, W_2^{(4)}\in \mathbb{R}^{d_\mathrm{model} \times d_{\mathrm{model}}}$ and $\boldsymbol{\beta}^{(4)}, \boldsymbol{\gamma}^{(4)} \in \mathbb{R}^{d_\mathrm{model}}$ as follows:
\begin{align}
W_1^{(4)} &= 
\begin{bmatrix}
I & \mathbf{0} & \cdots & \mathbf{0} & \cdots & -I & \cdots \\
-I & \mathbf{0} & \cdots & \mathbf{0} & \cdots & I & \cdots \\
\mathbf{0}^\top & 1 & \cdots & 1 & \cdots & \mathbf{0}^\top & \cdots \\
\mathbf{0}^\top & 0 & \cdots & 1 & \cdots & \mathbf{0}^\top & \cdots \\
\vdots & \vdots & & \vdots & & \vdots 
\end{bmatrix}, \\
W_2^{(4)} &= 
\begin{bmatrix}
\vdots & \vdots & \vdots & \vdots &  \\
- \mathbf{1}^\top & -\mathbf{1}^\top& C_3^{(4)} & 1 & \cdots \\
\vdots & \vdots & \vdots & \vdots &  \\
\end{bmatrix}, \\
\boldsymbol{\beta}^{(4)} &= \mathbf{0}, \\
\boldsymbol{\gamma}^{(4)} &= 8\sqrt{\frac{\lceil\log_2 k \rceil}{d_\mathrm{model}}}\mathbf{1},
\end{align}
where $C_3^{(4)} = 2(\lceil \log_2 k \rceil + 1)$. Then, we obtain 
\begin{equation}
\begin{aligned}
&W_2^{(4)} \left[ \operatorname{LN}_{\mathrm{RMS}} \left(W_1^{(4)} \mathbf{h}_i^{(4)}\right)\right]_+ \\
&= W_2^{(4)} \left[ \operatorname{LN}_{\mathrm{RMS}} \left(\begin{bmatrix}
\mathbf{t}_i - \tilde{\mathbf{t}}_i \\
-(\mathbf{t}_i - \tilde{\mathbf{t}}_i) \\
o_i+1\\
1 \\
\vdots
\end{bmatrix}\right)\right]_+ \\
&= \begin{bmatrix}
\vdots \\
\operatorname{q}_i \\
\vdots
\end{bmatrix},
\end{aligned}
\end{equation}
where 
\begin{align}
&\begin{aligned}
\operatorname{q}_i &= 8\sqrt{\frac{\lceil\log_2 k \rceil}{d_\mathrm{model}}}\cdot \frac{1}{\operatorname{RMS}\left(W_1^{(4)} \mathbf{h}_i^{(4)}\right)} \\
& \quad \cdot \left(- \left\|\mathbf{t}_i - \tilde{\mathbf{t}}_i\right\|_1 + C_3^{(4)}(o_i+1) + 1\right) \\
&= 4\sqrt{\frac{\lceil\log_2 k \rceil}{2\|\mathbf{t}_i - \tilde{\mathbf{t}}_i\|_2^2 + (o_i+1)^2 + 1}} \operatorname{q}_i^\prime,
\end{aligned}
\end{align}
\begin{equation}\label{eq: definition of q'}
\operatorname{q}_i^\prime = -2 \left\|\mathbf{t}_i - \tilde{\mathbf{t}}_i\right\|_1 + 2C_3^{(4)}(o_i+1) + 2.
\end{equation}

Here
\begin{equation}
\begin{aligned}
&4\sqrt{\frac{\lceil\log_2 k \rceil}{2\|\mathbf{t}_i - \tilde{\mathbf{t}}_i\|_2^2 + (o_i+1)^2 + 1}} \\
& \geq4\sqrt{\frac{\lceil\log_2 k \rceil}{2\cdot 2^2\lceil \log_2 k \rceil + 2^2 + 1}} \\
& \geq4\sqrt{\frac{\lceil\log_2 k \rceil}{8\lceil \log_2 k \rceil + 8\lceil \log_2 k \rceil }} \\
& \geq 1
\end{aligned}
\end{equation}
holds, indicating that it is sufficient to check $\operatorname{q}_i^\prime$ satisfies the conditions instead of $\operatorname{q}_i$. We confirm the conditions by checking three patterns (i) $o_i \neq -1$; that is, $w_i$ is an open bracket or \texttt{<bos>}, (ii) $o_i = -1 \wedge \mathbf{t}_i = \tilde{\mathbf{t}}_i$; that is, $w_i$ is a closed bracket and has a depth- and type-matched open bracket, and (iii) $o_i = -1 \wedge \mathbf{t}_i \neq \tilde{\mathbf{t}}_i$; that is, $w_i$ is a closed bracket and faces type conflict. 

\paragraph*{(i) $w_i$ is an open bracket or \texttt{<bos>}.}
In this case, since $o_i + 1 \geq 1$, 
\begin{equation}
\begin{aligned}
\operatorname{q}_i^\prime&=-2 \left\|\mathbf{t}_i - \tilde{\mathbf{t}}_i\right\|_1 + 2C_3^{(4)}(o_i+1) + 2\\
&\geq -4\lceil \log_2 k \rceil + 2C_3^{(4)}  + 2 \\
&= 6 > 1.
\end{aligned}
\end{equation}

\paragraph*{(ii) $w_i$ is a closed bracket and has a depth- and type-matched open bracket.}
In this case, since $\tilde{\mathbf{t}}_i = \mathbf{t}_i$ and $o_i + 1 = 0$,
\begin{equation}
\begin{aligned}
\operatorname{q}_i^\prime&=-2 \left\|\mathbf{t}_i - \tilde{\mathbf{t}}_i\right\|_1 + 2C_3^{(4)}(o_i+1) + 2\\
&= 2 > 1.
\end{aligned}
\end{equation}

\paragraph*{(iii) $w_i$ is a closed bracket and faces a type conflict.}
In this case, there are two exclusive subcases: (i) $w_i$ has no depth-matched open bracket; that is, $\tilde{\mathbf{t}}_i = \mathbf{0}$ holds and (ii) $w_i$ has depth-matched open brackets but faces type conflict; that is, $\tilde{\mathbf{t}}_i = \overline{\mathbf{t}}_i \neq \mathbf{t}_i$. In both subcases, $\|\tilde{\mathbf{t}}_i - \mathbf{t}_i\|_1 \geq 1$ and $o_i = -1$ hold; thus, we obtain
\begin{equation}
\begin{aligned}
\operatorname{q}_i^\prime&=-2 \left\|\mathbf{t}_i - \tilde{\mathbf{t}}_i\right\|_1 + 2C_3^{(4)}(o_i+1) + 2\\
&\leq -4 + 2 < -1.
\end{aligned}
\end{equation}

From the above, it is confirmed that the inequality \eqref{eq: conditions for q_i} holds.

Finally, considering the residual connection, we obtain the following vectors:
\begin{equation}
\begin{aligned}
\mathbf{x}_i^{(5)} &= \mathbf{h}_i^{(4)} + \begin{bmatrix}
\vdots\\
\operatorname{q}_i \\
\vdots
\end{bmatrix} \\
& = \begin{bmatrix}
\mathbf{t}_i \\
o_i \\
s_i \\
1 \\ 
\vdots \\
\cos \phi(i)\\
\sin \phi(i)\\
\vdots \\ 
\cos \theta(\operatorname{d}_i)\\
\sin \theta(\operatorname{d}_i)\\
\vdots \\ 
\cos \theta(\operatorname{d}_i+1)\\
\sin \theta(\operatorname{d}_i+1)\\
\tilde{\mathbf{t}}_{i} \\
\operatorname{q}_i \\
\mathbf{0} 
\end{bmatrix}.
\end{aligned}
\end{equation}

\subsection{Fifth layer}\label{app: dyck recognition with bos, subsec: fifth layer}
The fifth layer check the two conditions: $\bigwedge_{i=0}^n (\operatorname{q}_i > 0)$ and $\operatorname{d}(w_{1:n}) = 0$.

\subsubsection*{Fifth layer --- Attention layer}
We omit the unnecessary dimensions of input vector $\mathbf{x}_i^{(5)}$ in this layer as follows:
\begin{equation}
\mathbf{x}_i^{(5)} = \begin{bmatrix}
\vdots\\
s_i \\
1 \\
\vdots \\
\operatorname{q}_i \\
\vdots
\end{bmatrix}.
\end{equation}

Set the parameters $W_Q^{(5)}, W_K^{(5)}, W_V^{(5)}\in \mathbb{R}^{d_\mathrm{model} \times d_{\mathrm{model}}}$ as follows:
\begin{align}
W_Q^{(5)} &= 
\begin{bmatrix}
\cdots & 0 & C_1^{(5)} & \cdots & 0 & \cdots \\
\cdots & 0 & C_1^{(5)} & \cdots & 0 & \cdots \\
 & \vdots & \vdots &  & \vdots & 
\end{bmatrix}, \\
W_K^{(5)} &= 
\begin{bmatrix}
\cdots & 0 & 0 & \cdots & -1 & \cdots \\
\cdots & \operatorname{q}_0 & 0 & \cdots & 0 & \cdots \\
 & \vdots & \vdots &  & \vdots & 
\end{bmatrix}, \\
W_V^{(5)} &= 
\begin{bmatrix}
 & \vdots & \vdots &  & \vdots & \\
\cdots & -1 & 1 & \cdots & 0 & \cdots \\
 & \vdots & \vdots &  & \vdots & 
\end{bmatrix},
\end{align}
where $C_1^{(5)}$ is a positive constant. Note that $\operatorname{q}_0$ can be treated as a constant because $\operatorname{q}_0$ does not depend on the input string.

Then, we obtain
\begin{align}
& W_Q^{(5)} \mathbf{x}_{i_q}^{(5)}= 
\begin{bmatrix}
C_1^{(5)} \\
C_1^{(5)} \\
\mathbf{0}
\end{bmatrix}, \\
& W_K^{(5)} \mathbf{x}_{i_k}^{(5)}
=\begin{bmatrix}
- \operatorname{q}_{i_k} \\
 \operatorname{q}_{0}\cdot s_{i_k} \\
\mathbf{0}
\end{bmatrix}, \\
& W_V^{(5)} \mathbf{x}_{i_k}^{(5)}=
\begin{bmatrix}
\mathbf{0} \\
1-s_{i_k} \\
\mathbf{0}
\end{bmatrix}, \\
& \begin{aligned}
&\left\langle W_K^{(5)} \mathbf{x}_{i_k}^{(5)}, W_Q^{(5)} \mathbf{x}_{i_q}^{(5)}\right\rangle \\
& \quad = C_1^{(5)} \left( -\operatorname{q}_{i_k} + \operatorname{q}_{0}\cdot s_{i_k} \right) \\
& \quad 
\begin{cases}
 = 0 & \text{if } i_k = 0 \\
 < -C_1^{(5)} & \text{if } i_k \neq 0 \wedge \operatorname{q}_{i_k} > 0 \\
 > C_1^{(5)} & \text{if } i_k \neq 0 \wedge \operatorname{q}_{i_k} < 0 \\
\end{cases}.
\end{aligned}
\end{align}

Intuitively, if $\{\operatorname{q}_{i_k}\}_{{i_k}=1}^{i_q}$ are all positive, attention scores on all tokens except on $\texttt{<BOS>}$ are much smaller than $0$, making the query $\mathbf{x}_{i_q}$ focus on $\texttt{<BOS>}$. In other words, the query $\mathbf{x}_{i_q}$ can focus on $\texttt{<BOS>}$ if and only if $w_{1:{i_q}}$ is a prefix for the $\texttt{Dyck}_k$ language.
Therefore, given a sufficiently large constant $C_1^{(5)}$, the output of attention layer $\mathbf{a}_i^{(5)}$ becomes
\begin{equation}
\mathbf{a}_i^{(5)} = 
\begin{bmatrix}
\mathbf{0} \\
\operatorname{q}_{\leq i} \\
\mathbf{0}
\end{bmatrix}, 
\end{equation}
where
\begin{equation}
\operatorname{q}_{\leq i} = 
\begin{cases}
0 & \text{if $\forall j \in [i]. \operatorname{q}_j > 0$} \\
1 & \text{if $\exists j \in [i]. \operatorname{q}_j < 0$}
\end{cases}
\end{equation}

Finally, considering the residual connection, we obtain the following vectors:
\begin{equation}
\begin{aligned}
\mathbf{h}_i^{(5)} &= \mathbf{x}_i^{(5)} + \begin{bmatrix}
\vdots\\
\operatorname{q}_{\leq i} \\
\vdots
\end{bmatrix} \\
& = \begin{bmatrix}
\mathbf{t}_i \\
o_i \\
s_i \\
1 \\ 
\vdots \\
\cos \phi(i)\\
\sin \phi(i)\\
\vdots \\ 
\cos \theta(\operatorname{d}_i)\\
\sin \theta(\operatorname{d}_i)\\
\vdots \\ 
\cos \theta(\operatorname{d}_i+1)\\
\sin \theta(\operatorname{d}_i+1)\\
\tilde{\mathbf{t}}_{i} \\
\operatorname{q}_i \\
\operatorname{q}_{\leq i} \\
\mathbf{0} 
\end{bmatrix}.
\end{aligned}
\end{equation}
Although we treat softmax attention as hardmax attention, it is sufficient that there exists a constant such that $\max \{ \operatorname{q}_{\leq i} \mid \forall j \in [i]. \operatorname{q}_j > 0\} < \min \{ \operatorname{q}_{\leq i} \mid \exists j \in [i]. \operatorname{q}_j < 0\}$, similar to the fourth layer.

\subsubsection*{Fifth layer --- Feed-forward network layer}
We omit the unnecessary dimensions of input vector $\mathbf{h}_i^{(5)}$ in this layer as follows:
\begin{equation}
\mathbf{h}_i^{(5)} = \begin{bmatrix}
\vdots\\
\cos \theta(\operatorname{d}_i) \\
\sin \theta(\operatorname{d}_i) \\
\vdots \\
\operatorname{q}_{\leq i} \\
\vdots
\end{bmatrix}.
\end{equation}

Set the parameters $W_1^{(5)}, W_2^{(5)}\in \mathbb{R}^{d_\mathrm{model} \times d_{\mathrm{model}}}$ and $\boldsymbol{\beta}^{(5)}, \boldsymbol{\gamma}^{(5)} \in \mathbb{R}^{d_\mathrm{model}}$ as follows:
\begin{align}
W_1^{(5)} &= 
\begin{bmatrix}
\cdots & 0 & 0 & \cdots & 1 & \cdots \\
\cdots & 1 & 0 & \cdots & 0 & \cdots \\
\cdots & 0 & 1 & \cdots & 0 & \cdots \\
 & \vdots & \vdots & & \vdots &  \\
\end{bmatrix}, \\
W_2^{(5)} &= 
\begin{bmatrix}
\vdots & \vdots & \vdots & \\
1 & 0 & 1 & \cdots \\
\vdots & \vdots & \vdots & 
\end{bmatrix}, \\
\boldsymbol{\beta}^{(5)} &= \mathbf{0}, \\
\boldsymbol{\gamma}^{(5)} &= \sqrt{\frac{1}{d_\mathrm{model}}}\mathbf{1}.
\end{align}

Then, the output of the feed-forward network layer becomes
\begin{equation}
\begin{aligned}
& W_2^{(5)} \left[ \operatorname{LN}_{\mathrm{RMS}} \left(W_1^{(5)} \mathbf{h}_i^{(5)}\right)\right]_+ \\
&=W_2^{(5)} \left[ \operatorname{LN}_{\mathrm{RMS}} \left(\begin{bmatrix}
\operatorname{q}_{\leq i}\\
\cos{\theta(\operatorname{d}_i)} \\
\sin{\theta(\operatorname{d}_i)} \\
\vdots 
\end{bmatrix}\right)\right]_+ \\
&= W_2^{(5)} \begin{bmatrix}
\frac{\left[\operatorname{q}_{\leq i}\right]_+}{\sqrt{1 + \operatorname{q}_{\leq i}^2}}\\
\frac{\left[\cos\theta(\operatorname{d}_i)\right]_+}{\sqrt{1 + \operatorname{q}_{\leq i}^2}}\\
\frac{\left[\sin\theta(\operatorname{d}_i)\right]_+}{\sqrt{1 + \operatorname{q}_{\leq i}^2}}\\
\vdots 
\end{bmatrix} \\
&=\begin{bmatrix}
\vdots \\
\frac{\left[\operatorname{q}_{\leq i}\right]_+ + \left[\sin\theta(\operatorname{d}_i)\right]_+}{\sqrt{1 + \operatorname{q}_{\leq i}^2}}\\
\vdots 
\end{bmatrix}
\end{aligned}
\end{equation}

Finally, considering the residual connection, we obtain the following vectors:
\begin{equation}
\begin{aligned}
\mathbf{x}_i^{(6)} &= \mathbf{h}_i^{(5)} + \begin{bmatrix}
\vdots \\
\frac{\left[\operatorname{q}_{\leq i}\right]_+ + \left[\sin\theta(\operatorname{d}_i)\right]_+}{\sqrt{1 + \operatorname{q}_{\leq i}^2}}\\
\vdots 
\end{bmatrix} \\
& = \begin{bmatrix}
\mathbf{t}_i \\
o_i \\
s_i \\
1 \\ 
\vdots \\
\cos \phi(i)\\
\sin \phi(i)\\
\vdots \\ 
\cos \theta(\operatorname{d}_i)\\
\sin \theta(\operatorname{d}_i)\\
\vdots \\ 
\cos \theta(\operatorname{d}_i+1)\\
\sin \theta(\operatorname{d}_i+1)\\
\tilde{\mathbf{t}}_{i} \\
\operatorname{q}_i \\
\operatorname{q}_{\leq i} \\
\frac{\left[\operatorname{q}_{\leq i}\right]_+ + \left[\sin\theta(\operatorname{d}_i)\right]_+}{\sqrt{1 + \operatorname{q}_{\leq i}^2}}\\
\mathbf{0} 
\end{bmatrix}.
\end{aligned}
\end{equation}

\subsection{Classifier}
Finally, the classifier can classify the input sequence based on the value calculated in the fifth layer. The lower bound of the value when the input does not belong to the $\texttt{Dyck}_k$ language is calculated as follows:
\begin{equation}
\begin{aligned}
& \frac{\left[\operatorname{q}_{\leq i}\right]_+ + \left[\sin\theta(\operatorname{d}_i)\right]_+}{\sqrt{1 + \operatorname{q}_{\leq i}^2}}\\
& \begin{cases}
\geq \frac{1}{\sqrt{1 + 1^2 }} & \text{if } \operatorname{q}_{\leq i} > 0\\
\geq \frac{\sin \theta(1)}{\sqrt{1 + 1^2 }} & \text{if } \operatorname{d}_{i} > 0 \\
0 & \text{otherwise}
\end{cases} \\
& \begin{cases}
\geq \frac{\sin \theta(1)}{\sqrt{2}} & \text{if } \operatorname{q}_{\leq i} > 0 \vee \operatorname{d}_{i} > 0\\
0 & \text{otherwise}
\end{cases}
\end{aligned}
\end{equation}
Therefore, by subtracting a positive value less than this value as a bias, $\operatorname{sgn}(\cdot)$ can correctly classify whether the sequence belongs to $\texttt{Dyck}_k$. 

For instance, We omit the unnecessary dimensions of input vector $\mathbf{x}_i^{(6)}$ in this layer as follows:
\begin{equation}
\mathbf{x}_i^{(6)} = \begin{bmatrix}
\vdots\\
\frac{\left[\operatorname{q}_{\leq i}\right]_+ + \left[\sin{\theta(\operatorname{d}_i})\right]_+}{\sqrt{1 + \operatorname{q}_{\leq i}^2 }} \\
\vdots
\end{bmatrix}.
\end{equation}

Then, by setting 
\begin{align}
&\mathbf{w}^{\mathrm{cls}\top} = \begin{bmatrix}
\cdots & -1 & \cdots
\end{bmatrix}, \\
&b^{\mathrm{cls}} =  \frac{\sin \theta(1)}{2\sqrt{2}},
\end{align}
we obtain 
\begin{equation}
\begin{aligned}
&\mathbf{w}^{\mathrm{cls}\top}  \mathbf{x}_i^{(6)} + b^{\mathrm{cls}}  \\
& = -\frac{\left[\operatorname{q}_{\leq i}\right]_+ + \left[\sin{\theta(\operatorname{d}_i})\right]_+}{\sqrt{1 + \operatorname{q}_{\leq i}^2 }} +  \frac{\sin \theta(1)}{2\sqrt{2}}\\
&\begin{cases}
= \frac{\sin \theta(1)}{2\sqrt{2}} & \text{if } w_{0:i} \in \texttt{Dyck}_k \\
\leq - \frac{\sin \theta(1)}{2\sqrt{2}} & \text{if } w_{0:i} \notin \texttt{Dyck}_k
\end{cases}.
\end{aligned}
\end{equation}
\end{proof}

\section{Proof of Theorem \ref{theorem: transformers with bos generate dyck}}\label{app: dyck language generation with bos}

We restate Theorem \ref{theorem: transformers with bos generate dyck} for convenience. 
\begin{theorem}[Restated, Transformers with a starting token, $\texttt{Dyck}_k$ generation]
For all $k$, there exists a $3$-layer $O(\log k)$-width causal Transformer network without positional encoding that generates the $\texttt{Dyck}_k$ language. Each layer incorporates both the residual connection and the layer normalization. 
This network is followed by a fully-connected layer and softmax layer to output the probability distribution. 
\end{theorem}

\begin{proof}
Here, we present a method to construct a Transformer that realizes the $\texttt{Dyck}_k$ language generation process $p_{\texttt{Dyck}_k}(\cdot; q, r, \boldsymbol{\pi})$. We assume that the output probabilities take the following form: 
\begin{equation}
\begin{bmatrix}
p_{\langle_1} \\
\vdots \\
p_{\langle_k} \\
p_{\rangle_1} \\
\vdots \\
p_{\rangle_k} \\
p_{\texttt{<bos>}} \\
p_{\texttt{<eos>}} \\
\end{bmatrix}.
\end{equation}

As shown in the proof sketch of Theorem \ref{theorem: transformers with bos generate dyck}, each layer performs the following operations. Note that $w_{0:i}$ corresponds to $\texttt{<bos>} w_{1:i}$.
\begin{description}
    \item[First layer] creates pseudo positional encoding $(\cos\phi(i), \sin\phi(i))$.
    
    \item[Second layer] counts depth $\operatorname{d}(w_{0:i})$.

    \item[Third layer] fetches the valid closed bracket if one exists; otherwise, a zero vector is fetched. This operation is achieved by placing attention on the largest $j$ among $\{0\} \cup \{j \mid \operatorname{d}(w_{0:j}) = \operatorname{d}(w_{0:i})\}$.
\end{description}

\subsection{First and second layer}
We use the first two layers to compute positional encoding $(\cos \phi(i), \sin \phi(i))$ and depth $(\cos \theta (\operatorname{d}_i), \sin \theta (\operatorname{d}_i))$, following the same procedure as described in Appendix \ref{app: dyck recognition with bos, subsec: first layer} and \ref{app: dyck recognition with bos, subsec: second layer}. 

Therefore, the output from the second layer is as follows:
\begin{equation}
\mathbf{x}_i^{(3)} = 
\begin{bmatrix}
\vdots \\
\cos \phi(i) \\
\sin \phi(i) \\
\vdots \\
\cos \theta (\operatorname{d}_i) \\
\sin \theta (\operatorname{d}_i) \\
\vdots
\end{bmatrix}
\end{equation}

\subsection{Third layer}
Moreover, in the third layer, we leverage the attention layer to fetch the nearest open bracket with the same depth as the query, in almost the same manner as described in Appendix \ref{app: dyck recognition with bos, subsec: fourth layer}. The difference from the construction in the previous section is that we replace the query depth $\operatorname{d}_i + 1$ with $\operatorname{d}_i$. Therefore, the output from the attention layer is as follows:
\begin{equation}
\begin{bmatrix}
\vdots \\
\cos \phi(i) \\
\sin \phi(i) \\
\vdots \\
\cos \theta (\operatorname{d}_i) \\
\sin \theta (\operatorname{d}_i) \\
\tilde{\mathbf{t}}_i \\
\vdots
\end{bmatrix}
\end{equation}

Then, in the feed-forward network layer, set the parameters as follows:
\begin{align}
W_1^{(3)} &= \begin{bmatrix}
\cdots & 0 & 0 & \cdots & 0 & 1 & \mathbf{0}^\top &\cdots \\
\cdots & 0 & 0 & \cdots & 0 & 1 & \mathbf{0}^\top &\cdots \\
\cdots & 0 & 0 & \cdots & 0 & 1 & \mathbf{0}^\top &\cdots \\
\cdots & 0 & 0 & \cdots & 0 & 1 & \mathbf{0}^\top &\cdots \\
\cdots & 0 & 0 & \cdots & 2 & 0 & \mathbf{0}^\top &\cdots \\
 & \vdots & \vdots & & \vdots & \vdots & \vdots & 
\end{bmatrix}, \\
W_2^{(3)} &= \begin{bmatrix}
1 & 0 & 0 & 0 & 0 & \cdots \\
0 & 1 & 0 & 0 & 0 & \cdots \\
0 & 0 & 1 & 0 & 0 & \cdots \\
0 & 0 & 0 & 1 & 0 & \cdots \\
\vdots & \vdots & \vdots & \vdots & \vdots &  \\
\end{bmatrix}, \\
\boldsymbol{\beta}^{(3)} &= \begin{bmatrix}
0 \\
- \epsilon^{(3)} \\
0 \\
\epsilon^{(3)} \\
0 \\
\mathbf{0}
\end{bmatrix}, \\
\boldsymbol{\gamma}^{(3)} &= \sqrt{\frac{4}{d_\mathrm{model}}}\mathbf{1},
\end{align}
where $\epsilon^{(3)}$ is a positive constant. Then
\begin{equation}
\begin{aligned}
& \operatorname{LN}_{\mathrm{RMS}} \left(W_1^{(3)} \mathbf{h}_i^{(3)}\right) \\
&= \begin{bmatrix}
\sin \theta (\operatorname{d}_i) \\
\sin \theta (\operatorname{d}_i) - \epsilon^{(3)} \\
-\sin \theta (\operatorname{d}_i) \\
-\sin \theta (\operatorname{d}_i) + \epsilon^{(3)}\\
2 \cos \theta (\operatorname{d}_i) \\
\vdots
\end{bmatrix}
\end{aligned},
\end{equation}
because 
\begin{equation}
\begin{aligned}
&\operatorname{RMS}\left(W_1^{(3)} \mathbf{h}_i^{(3)}\right) \\
& = \sqrt{\frac{4 \sin^2 \theta (\operatorname{d}_i) + 4 \cos^2 \theta (\operatorname{d}_i)}{d_\mathrm{model}}} \\
& = \sqrt{\frac{4}{d_\mathrm{model}}}.
\end{aligned}
\end{equation}
Therefore, we obtain
\begin{equation}
\begin{aligned}
& W_2^{(3)} \left[\operatorname{LN}_\mathrm{RMS}\left(W_1^{(3)} \mathbf{h}_i^{(3)}\right)\right]_+ \\
& = \begin{bmatrix}
\vdots \\
[\sin \theta (\operatorname{d}_i)]_+ \\
[\sin \theta (\operatorname{d}_i) - \epsilon^{(3)}]_+ \\
[-\sin \theta (\operatorname{d}_i)]_+ \\
[-(\sin \theta (\operatorname{d}_i) - \epsilon^{(3)})]_+ \\
\vdots
\end{bmatrix}
\end{aligned}
\end{equation}

Finally, we obtain the input vector to the subsequent generator head as follows:
\begin{equation}
\mathbf{x}_i^{(4)} = \begin{bmatrix}
\vdots \\
1 \\
\vdots \\
\tilde{\mathbf{t}}_i \\
\left[\sin \theta(\operatorname{d}_i)\right]_+ \\
\left[\sin \theta(\operatorname{d}_i) - \epsilon^{(3)} \right]_+ \\
\left[- \sin \theta(\operatorname{d}_i)\right]_+\\
\left[- (\sin \theta(\operatorname{d}_i) - \epsilon^{(3)} )\right]_+\\
\vdots
\end{bmatrix}
\end{equation}

\subsection{Generator head}
For clarity, we implement $W^{\mathrm{gen}}$ as a composition of two linear transformations $W_1^{\mathrm{gen}} \in \mathbb{R}^{(k+4)\times d_\mathrm{model}}, W_2^{\mathrm{gen}} \in \mathbb{R}^{(2k+2)\times(k+4)}$ as follows (the transposed matrices are described to accommodate the limited space):
\begin{align}
&W_1^{\mathrm{gen}\top} = \begin{bmatrix}
\vdots & & \vdots & \vdots & \vdots & \vdots & \vdots \\
\lceil \log_2 k \rceil & \cdots & \lceil \log_2 k \rceil & 0 & 0 & 0 & 0 \\
\vdots & & \vdots & \vdots & \vdots & \vdots & \vdots \\
-\mathbf{t}_1 & \cdots & -\mathbf{t}_k & \mathbf{0} & \mathbf{0} & \mathbf{0} & \mathbf{0} \\
0 & \cdots & 0 & 1 & 0 & 0 & 0 \\
0 & \cdots & 0 & 0 & 1 & 0 & 0 \\
0 & \cdots & 0 & 0 & 0 & 1 & 0 \\
0 & \cdots & 0 & 0 & 0 & 0 & 1 \\
\vdots & & \vdots & \vdots & \vdots & \vdots & \vdots 
\end{bmatrix}, \\
&W_2^{\mathrm{gen}\top} = 
\begin{matrix}\begin{bmatrix}
O & -C_0^{\mathrm{gen}}I & \mathbf{0} & \mathbf{0} \\
\mathbf{0}^\top & \frac{C_1^{\mathrm{gen}}}{\epsilon^{(3)}}\mathbf{1}^\top & 0 & 0 \\
\mathbf{0}^\top & \frac{C_1^{\mathrm{gen}}}{\epsilon^{(3)}}\mathbf{1}^\top & 0 & 0 \\
\mathbf{0}^\top & \mathbf{0}^\top & 0 & \frac{-C_2^{\mathrm{gen}}}{\epsilon^{(3)}} \\
\mathbf{0}^\top & \mathbf{0}^\top & 0 & \frac{-C_2^{\mathrm{gen}}}{\epsilon^{(3)}} 
\end{bmatrix} \\
\begin{matrix}
\overset{\underbrace{\hphantom{W_{21}}}_{k \text{ dim.}}}{{\hphantom{W_{21}}}} & 
\overset{\underbrace{\hphantom{\mathbf{0}^\top 000 }}_{k \text{ dim.}}}{{\hphantom{\mathbf{0}^\top}}} & 
\overset{\underbrace{\hphantom{\mathbf{0}^\top\mathbf{0}^\top\mathbf{0}^\top}}_{2 \text{ dim.}}}{{\hphantom{\mathbf{w}_{22}}}}
\end{matrix} 
\end{matrix}, \\
&\mathbf{b}^{\mathrm{gen}} = \begin{bmatrix}
C_0^{\mathrm{gen}} + \log \pi_1 \\
\vdots \\
C_0^{\mathrm{gen}} + \log \pi_k \\
0 \\
\vdots \\
0 \\
0 \\
0
\end{bmatrix}\begin{matrix*}
\scalebox{1.15}{$\left.\rule{0pt}{2.0em}\right\}$} k \text{ dim.} \\
\scalebox{1.05}{$\left.\rule{0pt}{2.0em}\right\}$} k \text{ dim.} \rule{0pt}{2.5em}\\
\scalebox{1.05}{$\left.\rule{0pt}{1.0em}\right\}$} 2 \text{ dim.} \rule{0pt}{1.6em}
\end{matrix*}, \\
\end{align}
where $C_0^{\mathrm{gen}}$ is a positive constant and 
\begin{align}
C_1^{\mathrm{gen}} = \log\left(\frac{1-q}{q}\right) + C_0^{\mathrm{gen}}, \\
C_2^{\mathrm{gen}} = \log\left(\frac{1-r}{r}\right) + C_0^{\mathrm{gen}}.
\end{align}
Then, given a sufficiently small constant $\epsilon^{(3)}$,
\begin{equation}
\begin{aligned}
& W^{\mathrm{gen}} \mathbf{x}_i^{(4)} + \mathbf{b}^{\mathrm{gen}}\\
& = W_2^{\mathrm{gen}} W_1^{\mathrm{gen}} \mathbf{x}_i^{(4)} +   \mathbf{b}^{\mathrm{gen}} \\
& = W_2^{\mathrm{gen}} \begin{bmatrix}
-\mathbf{t}_1^\top\overline{\mathbf{t}} + \lceil \log_2 k \rceil \\
\vdots \\
-\mathbf{t}_k^\top\overline{\mathbf{t}} + \lceil \log_2 k \rceil\\
\left[\sin \theta(\operatorname{d}_i)\right]_+ \\
\left[\sin \theta(\operatorname{d}_i) - \epsilon^{(3)} \right]_+\\
\left[- \sin \theta(\operatorname{d}_i)\right]_+\\
\left[- (\sin \theta(\operatorname{d}_i) - \epsilon^{(3)} )\right]_+\\
\end{bmatrix} + \mathbf{b}^{\mathrm{gen}} \\
& = \begin{bmatrix}
C_0^{\mathrm{gen}} + \log \pi_1 \\
\vdots \\
C_0^{\mathrm{gen}} + \log \pi_k \\
\left(\begin{aligned}
&- C_0^{\mathrm{gen}} \left( \lceil \log_2 k \rceil-\mathbf{t}_1^\top\overline{\mathbf{t}}\right) \\
&\quad + C_1^{\mathrm{gen}} \mathbb{I}\left[ \operatorname{d}_i \geq 1 \right]
\end{aligned}\right) \\
\vdots \\
\left(\begin{aligned}
&- C_0^{\mathrm{gen}} \left( \lceil \log_2 k \rceil-\mathbf{t}_k^\top\overline{\mathbf{t}}\right) \\
&\quad + C_1^{\mathrm{gen}} \mathbb{I}\left[ \operatorname{d}_i \geq 1 \right]
\end{aligned}\right) \\
0 \\
C_2^{\mathrm{gen}} \mathbb{I}\left[ \operatorname{d}_i \leq 0 \right]
\end{bmatrix},
\end{aligned}
\end{equation}
where
\begin{align}
&\begin{aligned}
& \mathbb{I}\left[ \operatorname{d}_i \geq 1 \right] \\ 
& \quad = \frac{\left[\sin \theta(\operatorname{d}_i)\right]_+ - \left[\sin \theta(\operatorname{d}_i) - \epsilon^{(3)} \right]_+}{\epsilon^{(3)}} \\
& \quad = \begin{cases}
1 & \text{if } \operatorname{d}_i \geq 1 \\
0 & \text{otherwise}
\end{cases},
\end{aligned} \\
&\begin{aligned}
& \mathbb{I}\left[ \operatorname{d}_i \leq 0 \right] \\ 
& \quad = \frac{\left[-(\sin \theta(\operatorname{d}_i) - \epsilon^{(3)})\right]_+ - \left[-\sin \theta(\operatorname{d}_i)\right]_+}{\epsilon^{(3)}} \\
& \quad = \begin{cases}
1 & \text{if } \operatorname{d}_i \leq 0 \\
0 & \text{otherwise}
\end{cases}.
\end{aligned}
\end{align}

\subsection{Softmax}
We compute the logit separately for the cases where (i) $\operatorname{d}_i = 0$ and (ii) $\operatorname{d}_i \geq 1$. Let $\operatorname{logit}$ be the output logit vector. Note that we identify logit vectors that become identical through translation because they are projected to the same probability vector by the softmax operation. Let $\equiv$ be the equivalence relation on logits.

\paragraph{(i) In the case of $\operatorname{d}_i = 0$.}
\begin{equation}
\begin{aligned}
&\operatorname{logit} \\
&= \begin{bmatrix}
C_0^{\mathrm{gen}} + \log \pi_1 \\
\vdots \\
C_0^{\mathrm{gen}} + \log \pi_k \\
\left(\begin{aligned}
&- C_0^{\mathrm{gen}} \left( \lceil \log_2 k \rceil-\mathbf{t}_1^\top\overline{\mathbf{t}}\right) \\
&\quad + C_1^{\mathrm{gen}} \mathbb{I}\left[ \operatorname{d}_i \geq 1 \right]
\end{aligned}\right) \\
\vdots \\
\left(\begin{aligned}
&- C_0^{\mathrm{gen}} \left( \lceil \log_2 k \rceil-\mathbf{t}_k^\top\overline{\mathbf{t}}\right) \\
&\quad + C_1^{\mathrm{gen}} \mathbb{I}\left[ \operatorname{d}_i \geq 1 \right]
\end{aligned}\right) \\
0 \\
C_2^{\mathrm{gen}} \mathbb{I}\left[ \operatorname{d}_i \leq 0 \right]
\end{bmatrix} \\
&\equiv \begin{bmatrix}
\log r \pi_1 \\
\vdots \\
\log r \pi_k \\
-C_0^{\mathrm{gen}} \left( \lceil \log_2 k \rceil-\mathbf{t}_1^\top\overline{\mathbf{t}} + 1\right) + \log r \\
\vdots \\
-C_0^{\mathrm{gen}} \left( \lceil \log_2 k \rceil-\mathbf{t}_k^\top\overline{\mathbf{t}} + 1\right) + \log r  \\
- C_0^{\mathrm{gen}}  \\
\log\left({1-r}\right)
\end{bmatrix} \\
& =: \operatorname{logit}^\prime
\end{aligned}
\end{equation}
To establish the upper bound of the total variation distance, we first derive an upper bound and lower bound for the softmax denominator:
\begin{align}
&\begin{aligned}
&\sum_{l=1}^K \exp(\operatorname{logit}^\prime_l) \\
&= \sum_{t=1}^k r \pi_t \\
& \quad + \sum_{t=1}^k r \exp \left(-C_0^{\mathrm{gen}}\left(\lceil \log_2 k \rceil -\mathbf{t}_t^\top\overline{\mathbf{t}} + 1 \right)\right)\\
& \quad + \exp(-C_0^{\mathrm{gen}}) + 1-r\\
&\leq r + kr\exp(-C_0^{\mathrm{gen}}) + \exp(-C_0^{\mathrm{gen}}) + 1-r \\
&= 1 + (k + 1) \exp(-C_0^{\mathrm{gen}}),  \\
\end{aligned} \\
&\begin{aligned}
&\sum_{l=1}^K \exp(\operatorname{logit}^\prime_l) \\
&\geq \sum_{t=1}^k r \pi_t  + k\cdot 0 + 0 + 1 - r = 1
\end{aligned} 
\end{align}
Therefore, the lower bound of the total variation distance from the true probability distribution is given by:
\begin{equation}
\begin{aligned}
&\operatorname{TV}(\mathbb{S}(\text{logit}), p_{\texttt{Dyck}_k}(q, r, \boldsymbol{\pi})) \\
&=\operatorname{TV}(\mathbb{S}(\text{logit}^{\prime}), p_{\texttt{Dyck}_k}(q, r, \boldsymbol{\pi})) \\
& \begin{aligned}
\leq & \sum_{t=1}^k \frac{(k+1)\exp(-C_0^{\mathrm{gen}})}{1 + (k+1)\exp(-C_0^{\mathrm{gen}})}r\pi_t \\
         & + \sum_{t=1}^k \exp(- C_0^{\mathrm{gen}})\\
         & + \exp(- C_0^{\mathrm{gen}}) \\
         & + \frac{(k+1)\exp(-C_0^{\mathrm{gen}})}{1 + (k+1)\exp(-C_0^{\mathrm{gen}})}(1 - r) 
\end{aligned} \\
& \begin{aligned}
    = & \frac{(k+1)\exp(-C_0^{\mathrm{gen}})}{1 + (k+1)\exp(-C_0^{\mathrm{gen}})} \\
         & + (k+1) \exp(- C_0^{\mathrm{gen}})
\end{aligned} \\
& \leq 2(k+1)\exp(-C_0^{\mathrm{gen}}).
\end{aligned}
\end{equation}

\paragraph{(ii) In the case of $\operatorname{d}_i\geq 1$.}

Similar to the case (i), the upper bound of TV distance can be calculated as follows:
\begin{equation}
\begin{aligned}
&\operatorname{TV}(\mathbb{S}(\text{logit}), p_{\texttt{Dyck}_k}(q, r, \boldsymbol{\pi})) \\
& \quad \leq 2(k+1)\exp(-C_0^{\mathrm{gen}}).
\end{aligned}
\end{equation}

Therefore, for any $\epsilon>0$, by choosing a constant $C_0^{\mathrm{gen}}$ to satisfy 
\begin{align}
& 2(k+1)\exp(-C_0^{\mathrm{gen}}) < \epsilon \\
&\Leftrightarrow C_0^{\mathrm{gen}} > \log \frac{2(k+1)}{\epsilon},
\end{align}
then
\begin{equation}
\operatorname{TV}(\mathbb{S}([\text{logit}]), p_{\texttt{Dyck}_k}(q, r, \boldsymbol{\pi})) < \epsilon
\end{equation}
is satisfied.

Based on the above, the Transformer realizes the $\texttt{Dyck}_k$ language generation process $p_{\texttt{Dyck}_k}(\cdot; q, r, \boldsymbol{\pi})$.

\end{proof}

\section{Proof of Proposition \ref{proposition: transformers with bos recognize shuffle dyck_k}}\label{app:proof of shuffle dyck recognition}

\begin{proposition}[Restated, Transformers with a starting token, $\texttt{Shuffle-Dyck}_k$ recognition]
For all $k$, there exists a $3$-layer $O(\log k)$-width causal Transformer without positional encoding that recognizes the $\texttt{Shuffle-Dyck}_k$ language. Each layer incorporates both the residual connection and the layer normalization. 
This network is followed by a fully-connected layer and a sign function to output an acceptance signal.
\begin{proof}
In this section, we show how to implement a Transformer that recognizes $\texttt{Shuffle-Dyck}_k$.

We assume the same vector representation as defined in Appendix \ref{app: vector representation}:
\begin{equation}
\mathbf{x}_i=
\begin{bmatrix}
\mathbf{t}_i \\
o_i \\
s_i \\
1 \\
\mathbf{0}
\end{bmatrix}\in \mathbb{R}^{d_\mathrm{model}}.
\end{equation}

The first layer computes the positional encoding, and the second layer calculates the depth in almost the same manner as in Theorem \ref{theorem: transformers with bos recognize dyck_k}. However, in the second layer, the feed-forward network layer calculates $\left[\sin \theta(\operatorname{d}(w_{0:i}))\right]_+$, instead of calculating
$\cos \theta(\operatorname{d}(w_{0:i}))$, $\sin (-\theta(\operatorname{d}(w_{0:i})))$; that is, the output from the second layer becomes:
\begin{equation}
\mathbf{x}_i^{(3)} = \begin{bmatrix}
\mathbf{t}_i \\
o_i \\
s_i \\
1 \\
\cos \phi(i) \\
\sin \phi(i) \\
\left[\sin \theta(\operatorname{d}(w_{0:i}))\right]_+ \\
\mathbf{0}
\end{bmatrix}
\end{equation}

\subsection{Third layer}
The third layer calculates the depth of the substring that matches the same type as the query; that is, this layer computes
\begin{equation}
\left[\sin (-\theta(\operatorname{d}(w_{0:i}\mid \mathbf{t}_i)))\right]_+,
\end{equation}
where $\operatorname{d}(w_{0:i}\mid \mathbf{t}_i)$ represents the depth when focusing on the substring corresponding to type $\mathbf{t}_i$. 
For instance, for the input sequence $``\texttt{([(\{\})])}"$, this layer outputs the depth vectors corresponding to $[1,1,2,1,0,1,0,0]$. This is realized in a similar way to Appendix \ref{app: dyck recognition with bos, subsec: second layer} with slight modification. Specifically, by replacing the query and key matrix with the matrices as follows:
\begin{align}
W_Q^{(3)} &= 
\begin{bmatrix}
C^{(3)}I & \mathbf{0} & \mathbf{0} & \mathbf{0} & \cdots \\
\mathbf{0}^\top & 0 & 0 & C^{\prime(3)} &  \cdots \\
\vdots & \vdots & \vdots & \vdots &   
\end{bmatrix} ,\\
W_K^{(3)} &= 
\begin{bmatrix}
I & \mathbf{0} & \mathbf{0} & \mathbf{0} & \cdots \\
\mathbf{0}^\top & 0 & 1 & 0 & \cdots \\
\vdots & \vdots & \vdots & \vdots &   
\end{bmatrix},
\end{align}
where $C^{(3)}$ is a positive constant and $C^{\prime(3)} = C^{(3)}\lceil \log_2 k \rceil + a$.

Then, we obtain
\begin{align}
\begin{aligned}
&\left\langle W_K^{(3)} \mathbf{x}_{i_k}^{(3)}, W_Q^{(3)} \mathbf{x}_{i_q}^{(3)}\right\rangle \\
& \quad = C^{(3)}\langle \mathbf{t}_{i_q}, \mathbf{t}_{i_k}\rangle + C^{\prime(3)}s_{i_k} \\
&\quad\begin{cases}
=C^{(3)} \lceil \log_2 k \rceil & \text{if } \mathbf{t}_{i_k} = \mathbf{t}_{i_q} \\
=C^{(3)}\lceil \log_2 k \rceil + a  & \text{if } w_{i_k} = \texttt{<bos>} \\
\leq C^{(3)} (\lceil \log_2 k \rceil - 2) & \text{otherwise }
\end{cases}.
\end{aligned}
\end{align}

Therefore, for a sufficiently large constant $C^{(3)}$, we obtain

\begin{equation}
\begin{aligned}
\mathbf{x}_i^{(4)} &= \mathbf{h}_i^{(3)} + \begin{bmatrix}
\vdots \\
\left[\sin (-\theta(\operatorname{d}(w_{0:i} \mid \mathbf{t}_i)))\right]_+\\
\vdots
\end{bmatrix} \\
& =\begin{bmatrix}
\vdots \\
\left[\sin (\theta(\operatorname{d}(w_{0:i})))\right]_+ \\
\vdots \\
\left[\sin (-\theta(\operatorname{d}(w_{0:i} \mid \mathbf{t}_i)))\right]_+\\
\vdots 
\end{bmatrix}.
\end{aligned}
\end{equation}

\subsection{Fourth layer}
The fourth layer computes a necessary and sufficient condition for the string $w_{1:i}$ to belong to $\texttt{Shuffle-Dyck}_k$. 
The necessary and sufficient condition is that the following two conditions are simultaneously satisfied.
\begin{description}
    \item[Condition (i)]  $\operatorname{d}(w_{0:i}) \leq 0$.
    \item[Condition (ii)] $\operatorname{d}(w_{0:i} \mid \mathbf{t}_i) \geq 0 (\forall i)$ .
\end{description}

Here, it is sufficient to calculate
\begin{equation}
\begin{aligned}
&\operatorname{q}_\texttt{Shuffle-dyck}(w_{0:i})\\
&=\left[\sin \theta(\operatorname{d}(w_{0:i})) \right]_+ \\
&\quad + \frac{1}{i+1}\sum_{j=0}^i \left[\sin \theta(- \operatorname{d}(w_{0:j} \mid \mathbf{t}_{j}))\right]_+,
\end{aligned}
\end{equation}\label{eq: cond for shuffle dyck}
because $\operatorname{q}_\texttt{Shuffle-dyck}(w_{0:i})$ is always non-negative and becomes $0$ if and only if the two conditions above are simultaneously satisfied. We show how to implement a Transformer block that computes $\operatorname{q}_\texttt{Shuffle-dyck}(w_{0:i})$ in the fourth layer.

\subsubsection*{Fourth layer --- Attention layer}
We omit the unnecessary dimensions of input vector $\mathbf{x}_i^{(4)}$ in this layer as follows:
\begin{equation}
\mathbf{x}_i^{(4)} = \begin{bmatrix}
\vdots \\
\left[\sin (-\theta(\operatorname{d}(w_{0:i} \mid \mathbf{t}_i)))\right]_+\\
\vdots 
\end{bmatrix}.
\end{equation}
By setting the parameters $W_Q^{(4)} = O, W_K^{(4)} = O, W_V^{(4)} = \begin{bmatrix}
 & \vdots & \\
\cdots & 1 & \cdots \\
 & \vdots &
\end{bmatrix}$, we obtain the mean vector $\frac{1}{i+1}\sum_{j=0}^i \left[\sin \theta(- \operatorname{d}(w_{0:j} \mid \mathbf{t}_{j}))\right]_+$. Therefore, by adding the mean vector to the dimension corresponding to $\left[\sin (\theta(\operatorname{d}(w_{0:i})))\right]_+$, we obtain
\begin{equation}
\mathbf{h}_i^{(4)} =
\begin{bmatrix}
\vdots \\
\operatorname{q}_\texttt{Shuffle-dyck}(w_{0:i}) \\
\vdots 
\end{bmatrix}.
\end{equation}

\subsubsection*{Fourth layer --- Feed-forward network layer}
The feed-forward network has nothing to do. By setting $W_1^{(4)} = O,  W_2^{(4)} = O$, we obtain $\mathbf{x}_i^{(5)} = \mathbf{h}_i^{(4)}$.

\subsection{Classifier}
The classifier head can simply determine the string as positive if $\operatorname{q}_\texttt{Shuffle-dyck}(w_{0:i})$ is $0$ and as negative if it is strictly greater than $0$. 

Specifically, we omit the unnecessary dimensions of input vector $\mathbf{x}_i^{(5)}$ in this layer as follows:
\begin{equation}
\mathbf{x}_i^{(5)} = \begin{bmatrix}
\vdots \\
\operatorname{q}_\texttt{Shuffle-dyck}(w_{0:i}) \\
\vdots 
\end{bmatrix}.
\end{equation}
Set the parameter $\mathbf{w}^{\mathrm{cls}} \in \mathbb{R}^{d_\mathrm{model}}$ and $b^{\mathrm{cls}} \in \mathbb{R}$ as follows: 
\begin{align}
&\mathbf{w}^{\mathrm{cls}} = \begin{bmatrix}
\vdots \\
-1 \\
\vdots 
\end{bmatrix}, \\
&b^{\mathrm{cls}} = \epsilon.
\end{align}
Then, the desired computation can be achieved because 
\begin{equation}
\begin{aligned}
&\mathbf{w}^{\mathrm{cls}\top} \mathbf{x}_i^{(5)} + b^{\mathrm{cls}} \\
&\begin{cases}
 =\epsilon & \text{if } w_{1:i}\in \texttt{Shuffle-Dyck}_k \\
< 0 & \text{if } w_{1:i}\notin \texttt{Shuffle-Dyck}_k
\end{cases}.
\end{aligned}
\end{equation}
\end{proof}
\end{proposition}

\section{Proof of Proposition \ref{proposition: transformers with bos generate shuffle dyck_k}}\label{app: proof of shuffle dyck generation}

\begin{proposition}[Restated, Transformers with a starting token, $\texttt{Shuffle-Dyck}_k$ generation]
For all $k$, there exists a $3$-layer $O(k)$-width causal Transformer without positional encoding that generates the $\texttt{Shuffle-Dyck}_k$ language. Each layer incorporates both the residual connection and the layer normalization. 
This network is followed by a fully-connected layer and softmax layer to output the probability distribution. 
\end{proposition}

\begin{proof}
Here, unlike the other sections, we assume one-hot vectors as the bracket-type vectors, where we multiply by $+1$ for open brackets and by $-1$ for closed brackets. 
For instance, $``\langle_2"$ is mapped into $\begin{bmatrix}0 & 1 & \cdots & 0\end{bmatrix}^\top$ and $``\rangle_1"$ is mapped into $\begin{bmatrix}-1 & 0 & \cdots & 0\end{bmatrix}^\top$. In addition, we prepare an $O(k)$-dimensional zero vector that acts as a memory.
Therefore, the input vector without positional encoding $\mathbf{x}_i^{(1)}$ becomes as follows:
\begin{equation}
\mathbf{x}_i^{(1)}=
\begin{bmatrix}
\mathbf{t}_i \\
\mathbf{0}
\end{bmatrix}
\begin{matrix*}[l]
\} k \text{ dim.}\\
\} (d_{\mathrm{model}} - k) \text{ dim.}
\end{matrix*}\in \mathbb{R}^{d_{\mathrm{model}}}.
\end{equation}

\subsection{First layer}

\subsubsection*{First layer --- Attention layer}
In the first layer, using uniform attention, the query at position $i$ computes the mean vector of $\{\mathbf{t}_j\}_{j=0}^i$; that is, the output of the attention layer becomes:
\begin{equation}
\mathbf{h}_i^{(1)} = \begin{bmatrix}
\mathbf{t}_{i} \\
\frac{1}{i+1} \sum_{j=0}^i \mathbf{t}_{j} \\
\mathbf{0}
\end{bmatrix}.
\end{equation}

\subsubsection*{First layer --- Feed-forward network layer}
We omit the unnecessary dimensions of input vector $\mathbf{h}_i^{(1)}$ in this layer as follows:
\begin{equation}
\mathbf{h}_i^{(1)} = \begin{bmatrix}
\mathbf{t}_{i} \\
\frac{1}{i+1} \sum_{j=0}^i \mathbf{t}_{j} \\
\vdots 
\end{bmatrix}.
\end{equation}
Set the parameters $W_1^{(1)}, W_2^{(1)}\in \mathbb{R}^{d_\mathrm{model} \times d_{\mathrm{model}}}$ and $\boldsymbol{\beta}^{(1)}, \boldsymbol{\gamma}^{(1)} \in \mathbb{R}^{d_\mathrm{model}}$ as follows:
\begin{align}
&W_1^{(1)} = \begin{bmatrix}
O & -I & \cdots \\
O & -I & \cdots\\
I & O & \cdots\\
\vdots & \vdots & 
\end{bmatrix}, \\
&W_2^{(1)} = \begin{bmatrix}
\vdots & \vdots & \vdots &  \\
-I & I & O & \cdots \\
\vdots & \vdots & \vdots & 
\end{bmatrix}, \\
&\boldsymbol{\beta}^{(1)} = \begin{bmatrix}
\mathbf{0} \\
\mathbf{1} \\
\mathbf{0} \\
\vdots
\end{bmatrix}, \\
&\boldsymbol{\gamma}^{(1)} = \frac{1}{\epsilon\sqrt{d_\mathrm{model}}}\mathbf{1}.
\end{align}
Then, we obtain
\begin{equation}
\begin{aligned}
& W_2^{(1)} \left[ \operatorname{LN}_{\mathrm{RMS}} \left(W_1^{(1)} \mathbf{h}_i^{(1)}\right)\right]_+ \\
&= W_2^{(1)} \left[ \operatorname{LN}_{\mathrm{RMS}} \left(\begin{bmatrix}
-\frac{1}{i+1} \sum_{j=0}^i \mathbf{t}_{j} \\
-\frac{1}{i+1} \sum_{j=0}^i \mathbf{t}_{j} \\
\mathbf{t}_{i} \\
\vdots 
\end{bmatrix}\right)\right]_+ \\
&= W_2^{(1)}  
\begin{bmatrix}
-\frac{\sum_{j=0}^i \mathbf{t}_{j}}{\epsilon\sqrt{(i + 1)^2 + 2\|\sum_{j=0}^i \mathbf{t}_{j}\|_2^2}} \\
\mathbf{1} -\frac{\sum_{j=0}^i \mathbf{t}_{j}}{\epsilon\sqrt{(i + 1)^2 + 2\|\sum_{j=0}^i \mathbf{t}_{j}\|_2^2}} \\
\frac{\mathbf{t}_{i}}{\epsilon\sqrt{(i + 1)^2 + 2\|\sum_{j=0}^i \mathbf{t}_{j}\|_2^2}} \\
\vdots
\end{bmatrix}_+ \\
&= \begin{bmatrix}
\vdots \\
\mathbb{I}[\operatorname{d}(w_{0:i} \mid  t=1) \leq 0] \\
\vdots \\
\mathbb{I}[\operatorname{d}(w_{0:i} \mid  t=k) \leq 0] \\
\vdots
\end{bmatrix},
\end{aligned}
\end{equation}
where 
\begin{equation}
\begin{aligned}
& \mathbb{I}[\operatorname{d}(w_{0:i} \mid  t=t') \leq 0] \\
& = -\left[- \frac{\sum_{j=0}^i t_{j, t'}}{\epsilon\sqrt{(i + 1)^2 + 2\|\sum_{j=0}^i \mathbf{t}_{j}\|_2^2}}\right]_+ \\
& \qquad + \left[1 - \frac{\sum_{j=0}^i t_{j, t'}}{\epsilon\sqrt{(i + 1)^2 + 2\|\sum_{j=0}^i \mathbf{t}_{j}\|_2^2}}\right]_+ \\
& = \begin{cases}
1 & \text{if } \operatorname{d}(w_{0:i} \mid  t=t') \leq 0 \\
0 & \text{otherwise}
\end{cases}.
\end{aligned} 
\end{equation}

Finally, considering the residual connection, we obtain
\begin{equation}
\begin{aligned}
\mathbf{x}_i^{(2)}& = \mathbf{h}_i^{(1)} +  \begin{bmatrix}
\vdots \\
\mathbb{I}[\operatorname{d}(w_{0:i} \mid  t=1) \leq 0] \\
\vdots \\
\mathbb{I}[\operatorname{d}(w_{0:i} \mid  t=k) \leq 0] \\
\vdots
\end{bmatrix} \\
& = \begin{bmatrix}
\mathbf{t}_i \\
\frac{1}{i+1} \sum_{j=0}^i \mathbf{t}_{j} \\
\mathbb{I}[\operatorname{d}(w_{0:i} \mid  t=1) \leq 0] \\
\vdots \\
\mathbb{I}[\operatorname{d}(w_{0:i} \mid  t=k) \leq 0] \\
\vdots 
\end{bmatrix}.
\end{aligned}
\end{equation}

\subsection{Generator head}
We omit the unnecessary dimensions of input vector $\mathbf{x}_i^{(2)}$ in this layer as follows:
\begin{equation}
\mathbf{x}_i^{(2)} = \begin{bmatrix}
\vdots \\
\mathbf{m}(w_{0:i}) \\
\vdots
\end{bmatrix},
\end{equation}
where 
\begin{equation}
\mathbf{m}(w_{0:i}) = \begin{bmatrix}
\mathbb{I}[\operatorname{d}(w_{0:i} \mid  t=1) \leq 0] \\
\vdots \\
\mathbb{I}[\operatorname{d}(w_{0:i} \mid  t=k) \leq 0] \\
\end{bmatrix}.
\end{equation}

Set the parameters $W^{\mathrm{gen}} \in \mathbb{R}^{(2k+2) \times d_{\mathrm{model}}}, \mathbf{b}^{\mathrm{gen}} \in \mathbb{R}^{(2k+2)}$ as follows:
\begin{align}
&W^{\mathrm{gen}} = \begin{bmatrix}
\cdots & O & \cdots \\
\cdots & -C^{\mathrm{gen}} I & \cdots \\
\cdots & \mathbf{0}^\top & \cdots \\
\cdots & C^{\mathrm{gen}} \mathbf{1}^\top & \cdots
\end{bmatrix}\begin{matrix*}
\} k \text{ dim.}\\
\} k \text{ dim.}\\
\} 1 \text{ dim.}\\
\} 1\text{ dim.}
\end{matrix*}, \\
& \mathbf{b}^{\mathrm{gen}} = \begin{bmatrix}
\log \boldsymbol{\pi} \\
\log \left(\frac{1-q}{q}\right)\boldsymbol{\overline{\pi}} \\
-C^{\mathrm{gen}} \\
\log \left(\frac{1-r}{r}\right) - kC^{\mathrm{gen}} 
\end{bmatrix}\begin{matrix*}
\} k \text{ dim.}\\
\scalebox{0.9}{$\Big\}$} k \text{ dim.}\\
\} 1 \text{ dim.}\\
\scalebox{0.9}{$\Big\}$} 1 \text{ dim.}
\end{matrix*}.
\end{align}
Then, we obtain
\begin{equation}
\begin{aligned}
&\mathrm{logit} \\
&= W^{\mathrm{gen}}\mathbf{x}_i^{(2)} + \mathbf{b}^{\mathrm{gen}} \\
&= \begin{bmatrix}
\log \boldsymbol{\pi} \\
\log \left(\frac{1-q}{q}\right)\boldsymbol{\overline{\pi}} - C^{\mathrm{gen}} \mathbf{m}(w_{0:i})\\
- C^{\mathrm{gen}} \\
\log \left(\frac{1-r}{r}\right) - C^{\mathrm{gen}}(k - \mathbf{1}^\top \mathbf{m}(w_{0:i}))\\
\end{bmatrix}.
\end{aligned}
\end{equation}

\subsubsection*{Softmax}
Similar to the Appendix \ref{app: dyck language generation with bos}, it is possible to show that the $\texttt{Shuffle-Dyck}_k$ language generation process can be approximated with arbitrary precision. Here, for clarity, we treat $-C^{\mathrm{gen}}$ as a masking operation in the softmax function and show how it realizes the language generation process.

\paragraph{(i) in case that $\forall t'. \operatorname{d}(w_{0:i} \mid  t=t') = 0 $.}
Since $\mathbf{m}(w_{0:i}) = \mathbf{1}$, 
\begin{equation}
\begin{aligned}
&\mathrm{logit} \\
& = \begin{bmatrix}
\log \boldsymbol{\pi} \\
\log \left(\frac{1-q}{q}\right) \boldsymbol{\overline{\pi}} - C^{\mathrm{gen}} \mathbf{1} \\
- C^{\mathrm{gen}} \\
\log \left(\frac{1-r}{r}\right)
\end{bmatrix}.
\end{aligned}
\end{equation}
Therefore, 
\begin{equation}
\begin{aligned}
&\mathbb{S}(\mathrm{logit}) \\
& \simeq \mathbb{S}\left(\begin{bmatrix}
\log \boldsymbol{\pi} \\
\texttt{masked} \\
\texttt{masked} \\
\log \left(\frac{1-r}{r}\right)
\end{bmatrix}\right) \\
& =\mathbb{S}\left(\begin{bmatrix}
\log r\boldsymbol{\pi} \\
\texttt{masked} \\
\texttt{masked} \\
\log \left(1-r\right)
\end{bmatrix}\right),
\end{aligned}
\end{equation}
indicating that the $\texttt{Shuffle-Dyck}_k$ language generation process is realized.

\paragraph{(i) in case that $\exists t'. \operatorname{d}(w_{0:i} \mid  t=t') > 0 $.}
Since $\mathbf{1}^\top \mathbf{m}(w_{0:i}) \leq k - 1$,
\begin{equation}
\begin{aligned}
&\mathrm{logit} \\
& = \begin{bmatrix}
\log \boldsymbol{\pi} \\
\log \left(\frac{1-q}{q}\right) \boldsymbol{\overline{\pi}} - C^{\mathrm{gen}} \mathbf{m}(w_{0:i})\\
- C^{\mathrm{gen}} \\
\log \left(\frac{1-r}{r}\right) - C^{\mathrm{gen}} (k - \mathbf{1}^\top \mathbf{m}(w_{0:i}))
\end{bmatrix}.
\end{aligned}
\end{equation}
Therefore, we obtain
\begin{equation}
\begin{aligned}
&\mathbb{S}(\mathrm{logit}) \\
& \simeq \mathbb{S}\left(\begin{bmatrix}
\log \boldsymbol{\pi} \\
\log \left(\frac{1-q}{q}\right) \boldsymbol{\overline{\pi}} - C^{\mathrm{gen}} \mathbf{m}(w_{0:i}) \\
\texttt{masked} \\
\texttt{masked}
\end{bmatrix}\right) \\
& = \mathbb{S}\left(\begin{bmatrix}
\log q\boldsymbol{\pi} \\
\log (1-q) \boldsymbol{\overline{\pi}} - C^{\mathrm{gen}} \mathbf{m}(w_{0:i}) \\
\texttt{masked} \\
\texttt{masked}
\end{bmatrix}\right).
\end{aligned}
\end{equation}
In addition, the $t$-th element of ${\log \boldsymbol{\overline{\pi}} - C^{\mathrm{gen}} \mathbf{m}(w_{0:i})}$ is masked if and only if the depth of type $t$ is $0$ or less than $0$, indicating that the $\texttt{Shuffle-Dyck}_k$ language generation process is also realized in this case.

\end{proof}

\section{Proof of Proposition \ref{proposition: Networks with sub-polynomial width, cannot generate Shuffle Dyck}}\label{app: proof of prop 3}

\begin{proposition}[Restated]
There is no network whose width grows strictly slower than $k/\log k$ that generates $\texttt{Shuffle-Dyck}_k$; that is, if 
\begin{equation}
\lim_{k\rightarrow\infty} \frac{d_{\mathrm{model}}}{k/ \log k} = 0
\end{equation}
holds, then there exists $k_0$ such that for any $k \geq k_0$, $d_\mathrm{model}$-width networks cannot generate $\texttt{Shuffle-Dyck}_k$. For example, $\sqrt{k}$ grows strictly slower than $k/\log k$.
\end{proposition}

We provide the proof sketch first. Then, we show some lemmas in Section \ref{app: proof of prop 3, subsec: preliminaries} and give a proof of Proposition \ref{proposition: Networks with sub-polynomial width, cannot generate Shuffle Dyck} in Section \ref{app: proof of prop 3, subsec: proof}.

\begin{proof}[Proof sketch]
We give a proof by contradiction. Consider the $2^k$ different input strings: concerning the $l \in [2^k]$-th input, when the $t$-th bit of the binary representation of $l$ is $1$, we add an open bracket of type $t$. For example, when $k=2$, we consider the following $2^2$ inputs:
\begin{equation}
\begin{aligned}
00 & \mapsto \texttt{<bos>}, \\
01 & \mapsto \texttt{<bos>} \langle_1, \\
10 & \mapsto \texttt{<bos>} \langle_2, \\
11 & \mapsto \texttt{<bos>} \langle_1 \langle_2 .\\
\end{aligned}
\end{equation}
Then, it is necessary to satisfy the following $2^k$ constraints to generate $\texttt{Shuffle-Dyck}_k$ correctly: concerning the $l$-th constraint, if the $t_0$ and $t_1$-th bit of the binary representation of $l$ are $0$ and $1$, respectively, the $t_1$-th logit is strictly greater than the $t_0$-th logit. However, there is no linear transformation $\mathbb{R}^{d_\mathrm{model}} \rightarrow \mathbb{R}^k$ that satisfies the constraints above.
\end{proof}
In this section, we explicitly express the dependence of $d_\mathrm{model}$ on $k$, denoting it by $d(k)$ for clarity. Additionally, we occasionally use the concise notation $e$ instead of $\exp$ to conserve space.

\subsection{Preliminary Lemmas}\label{app: proof of prop 3, subsec: preliminaries}
\begin{definition}[Subspace]
Given a vector set $\{\mathbf{w}_{t}\}_{t=1}^{k} \subset \mathbb{R}^{d(k)}$. Let $\operatorname{bin}(l) = l_1 \cdots l_k$ be the binary representation of an integer $l \in [2^k]$ and $\operatorname{bin}_t(l) = l_t$ be the $t$-th bit of $\operatorname{bin}(l)$. Then, we define subspace $R(l) (\subset \mathbb{R}^{d(k)})$ as follows:
\begin{equation}
R(l) = \left\{\mathbf{x} \left|
\begin{aligned}
&(\mathbf{w}_{t_1} - \mathbf{w}_{t_0})^\top \mathbf{x} > 0 \\
&\forall t_0 \in \operatorname{id}_0(l), t_1 \in \operatorname{id}_1(l) 
\end{aligned}
\right.\right\},
\end{equation}
where
\begin{align}
&\operatorname{id}_0(l) = \{t \mid  \operatorname{bin}_t(l) = 0\},\\
&\operatorname{id}_1(l) = \{t \mid  \operatorname{bin}_t(l) = 1\}.
\end{align}
In addition, we say $R(l)$ and $R(l')$ are distinct if $R(l) \cap R(l') = \emptyset$. Moreover, we say the subspace set $\{R(\cdot)\}$ is distinct if for any two subspaces are distinct.
\end{definition}
Intuitively, for $\mathbf{x} \in R(l)$, $\mathbf{w}_{t_1}^\top \mathbf{x} > \mathbf{w}_{t_0}^\top \mathbf{x}$ holds, indicating the logit for type-$t_1$ is greater than that for type-$t_0$.

\begin{lemma}\label{lemma: distinct condition 1}
$R(l)$ and $R(l')$ are distinct if there exists $t\neq t'$ such that $t \in \operatorname{id}_0(l) \wedge t' \in \operatorname{id}_1(l) \wedge t \in \operatorname{id}_1(l') \wedge t' \in \operatorname{id}_0(l')$.
\begin{proof}
Since $t \in \operatorname{id}_0(l) \wedge t' \in \operatorname{id}_1(l)$, 
\begin{equation}
R(l) \subset \left\{\mathbf{x} \left| (\mathbf{w}_{t'} - \mathbf{w}_{t})^\top \mathbf{x} > 0 \right.\right\}
\end{equation}
holds. In contrast, since $t \in \operatorname{id}_1(l') \wedge t' \in \operatorname{id}_0(l')$, 
\begin{equation}
R(l') \subset \left\{\mathbf{x} \left| (\mathbf{w}_{t} - \mathbf{w}_{t'})^\top \mathbf{x} > 0 \right.\right\}
\end{equation}
holds, indicating $R(l) \cap R(l') = \emptyset$.
\end{proof}
\end{lemma}

\begin{lemma}\label{lemma: distinct condition 2}
$R(l)$ and $R(l')$ are distinct if $l \neq l' \wedge \#_1(l) = \#_1(l')$ holds, where $\#_1(l)$ is the number of ones in $\operatorname{bin}(l)$.
\begin{proof}
Since $l \neq l'$, there exists a digit $t$ such that $l_t \neq l'_t $. Without loss of generality, we can assume that $l_t = 1 \wedge l'_t = 0$. In addition, since $l$ and $l'$ have same number of ones, there exists $t'$ such that $l_{t'} = 0 \wedge l'_{t'} = 1$, indicating that $R(l)$ and $R(l')$ are distinct from Lemma \ref{lemma: distinct condition 1}.
\end{proof}
\end{lemma}

\begin{lemma}\label{lemma: stirling approximation}
For any integer $m \geq 1$, 
\begin{equation}
e^{\frac{11}{12}} m^{m+\frac{1}{2}} e^{-m} \leq m!\leq e^{\frac{13}{12}} m^{m+\frac{1}{2}} e^{-m}
\end{equation}
holds.
\begin{proof}
From the results in \citet{785bbcf5-e12d-3b4f-9b4c-913b654991d7}, 
\begin{equation}
\frac{m!e^m}{m^{m+\frac{1}{2}}} = C e^{r_m},
\end{equation}
holds for $m \geq 1$, where $C \in \left(e^{\frac{11}{12}}, e^{\frac{12}{13}}\right)$ and $r_m \in \left(\frac{1}{12m+1}, \frac{1}{12m}\right)$.
Therefore
\begin{equation}
\begin{aligned}
\frac{m!e^m}{m^{m+\frac{1}{2}}} &\in \left(e^{\left(\frac{11}{12} + \frac{1}{12m+1}\right)}, e^{\left(\frac{12}{13} + \frac{1}{12m}\right)}\right) \\
& \subseteq \left(e^{\frac{11}{12}}, e^{\frac{13}{12}}\right),
\end{aligned}
\end{equation}
indicating that the inequality holds for $m \geq 1$.
\end{proof}
\end{lemma}

\begin{lemma}\label{lemma: necessary number of distinct subspace}
For $k > 2$ and an integer set $\left[2^k\right] = \{0, \cdots, 2^k-1\}$. Then, at least $\left\lfloor \sqrt{2}^k\right\rfloor$-size distinct subspace set is necessary to make $R(l) \subset \mathbb{R}^{d_\mathrm{model}}$ non-empty for any $l\in \left[2^k\right]$.
\begin{proof}
In case that $k$ is an even number, 
\begin{equation}
\left| \left\{ l \left| \#_1(l) = \frac{k}{2}\right. \right\}  \right|= \binom{k}{k/2}, 
\end{equation}
On the other hand, in case that $k$ is an odd number, 
\begin{equation}
\left| \left\{ l \left| \#_1(l) = \frac{k-1}{2}  \right.\right\} \right|= \binom{k}{(k-1)/2}.
\end{equation}
When $k$ is even, at least $\binom{k}{k/2}$-size distinct subspace set is necessary because from Lemma \ref{lemma: distinct condition 2}, $\{R(l)\}_{l \in \left\{ l' \left| \#_1(l') = \frac{k}{2}\right. \right\}}$ is distinct.
Here,
\begin{align}
& \begin{aligned}
& \binom{k}{k/2} \\
& = \frac{k !}{(k/2)! (k/2)!} \\
& \geq \frac{e^{\frac{11}{12}} k^{k+\frac{1}{2}} e^{-k}}{\left(e^{\frac{13}{12}} (\frac{k}{2})^{\frac{k}{2}+\frac{1}{2}} e^{-\frac{k}{2}}\right)^2} \text{ (Lemma \ref{lemma: stirling approximation})} \\
& = 2e^{-\frac{37}{144}}\frac{2^k}{\sqrt{k}} \\
& \geq \sqrt{2}^k\frac{\sqrt{2}^k}{\sqrt{k}} \geq \sqrt{2}^k
\end{aligned} 
\end{align}
holds, indicating that $\left\lfloor \sqrt{2}^k\right\rfloor$-size distinct subspace set is necessary. Similarly, when $k$ is odd,
\begin{align}
& \begin{aligned}
& \binom{k}{(k-1)/2} \\
& \geq \sqrt{2}^k\frac{\sqrt{2}^k}{(k+1)/\sqrt{k}} \geq \sqrt{2}^k
\end{aligned}
\end{align}
holds, leading to the same result. 
\end{proof}
\end{lemma}

\begin{lemma}\label{lemma: inequality}
For any $x > 1$,
\begin{equation}
x^{\frac{1}{x}} < 1 + \frac{\log x + 1}{x}
\end{equation}
holds.

\begin{proof}
\begin{equation}
\begin{aligned}
x^{\frac{1}{x}} &= \exp\left(\frac{1}{x}\log x \right) \\
&= \sum_{p=0}^\infty \frac{\left(\frac{1}{x}\log x\right)^p}{p!} \\
&= 1 + \frac{1}{x}\log x + \sum_{p=2}^\infty \frac{\left(\frac{1}{x}\log x\right)^p}{p!} \\
&\leq 1 + \frac{1}{x}\log x + \left(\frac{1}{x}\log x\right)^2\sum_{p=2}^\infty \frac{1}{p!} \\
&< 1 + \frac{1}{x}\log x + \frac{1}{x}\cdot 1 \cdot (e-2) \\
&< 1 + \frac{\log x + 1}{x}.
\end{aligned}
\end{equation}
\end{proof}
\end{lemma}

We cite Lemma \ref{lemma: partition by hyperplanes} stated in \citet{On_the_Partition_of_Space_by_Hyperplanes_2023}. Note that we modify the statement to align with this paper.
\begin{lemma}[\citet{On_the_Partition_of_Space_by_Hyperplanes_2023}]\label{lemma: partition by hyperplanes}
Let $G(d_\mathrm{model}, m)$ be the maximum number of regions that are separated by $m$ hyperplanes in $\mathbb{R}^{d(k)}$. Then, 
\begin{equation}
G(d(k), m) = \sum_{d=0}^{d(k)} \binom{m}{d},
\end{equation}
where 
\begin{equation}
\binom{m}{d} = \begin{cases}
    \frac{m !}{d ! (m - d)!} & \text{if } d \leq m \\
    0 & \text{if } d > m
\end{cases}.
\end{equation}
\end{lemma}

\begin{lemma}\label{lemma: upper bound of distinct subspace set size}
For any function $d(k):\mathbb{N}\rightarrow\mathbb{N}$ that grows strictly slower than $k / \log k$; that is, $\lim_{k \rightarrow \infty} \frac{d(k)}{k / \log k} = 0$, $\log G\left(d(k), \binom{k}{2}\right)$ is a sub-linear function.

\begin{proof}
Since $d(k)$ grows strictly slower than $k / \log k$, there exists $k_0$ such that for any $k > k_0$, $d(k) < \frac{k^2}{2}$ holds. We assume $k > k_0$ for the remainder.
\begin{equation}
\begin{aligned}
& G\left(d(k), \binom{k}{2}\right) \\
& = \sum_{d=0}^{d(k)} \binom{\binom{k}{2}}{d} \text{ (Lemma \ref{lemma: partition by hyperplanes})}\\
& \leq \sum_{d=0}^{d(k)} \binom{k^2}{d} \\
& \leq \sum_{d=0}^{d(k)} \binom{k^2}{d(k)} \left(\text{because } d(k) < \frac{k^2}{2}\right)\\
& \leq 2d(k) \frac{k^2 (k^2 - 1) \cdots (k^2 - d(k) + 1)}{d(k)!}\\
& \leq  \frac{2d(k)k^{2d(k)}}{e^{\frac{11}{12}}d(k)^{d(k)+\frac{1}{2}} e^{-d(k)}}\text{ (Lemma \ref{lemma: stirling approximation})}\\
& \leq 2\sqrt{d(k)} e^{d(k)}\left(\frac{k^2}{d(k)}\right)^{\frac{d(k)}{k^2}k^2} \\
& \leq 2\sqrt{d(k)} e^{d(k)}\left(1 + \frac{\log\left(\frac{k^2}{d(k)}\right) + 1}{\frac{k^2}{d(k)}}\right)^{k^2} \\
& \qquad \qquad \qquad \qquad \qquad \text{ (Lemma \ref{lemma: inequality})}\\
& \leq 2\sqrt{d(k)} e^{d(k)}\left(1 + \frac{\tilde{d}(k)}{k^2}\right)^{\frac{k^2}{\tilde{d}(k)}\tilde{d}(k)} \\
& = 2\sqrt{d(k)} e^{d(k)+ \tilde{d}(k)},
\end{aligned}
\end{equation}
where
\begin{equation}
\begin{aligned}
\tilde{d}(k) &= d(k)\left(\log \left(\frac{k^2}{d(k)}\right)+1\right).
\end{aligned}
\end{equation}
Therefore, 
\begin{equation}
\begin{aligned}
& \frac{\log G\left(d(k), \binom{k}{2}\right)}{k} \\
& \leq \frac{\log \left(2\sqrt{d(k)} e^{d(k)+ \tilde{d}(k)}\right)}{k} \\
& \leq \frac{\log 2\sqrt{d(k)} + \tilde{d}{(k)} + \tilde{d}{(k)}}{k} \\
& = \frac{\log 2\sqrt{d(k)} + 2 d(k)\left(\log\left(\frac{k^2}{d(k)}\right) + 1\right)}{k}\\
& \leq \frac{\log 2\sqrt{d(k)} + d(k)\cdot6 \log k}{k}\\
& \leq \frac{\log 2\sqrt{d(k)}}{k} + 6\frac{d(k)}{k / \log k} \\
& \underset{k\rightarrow \infty}{\longrightarrow} 0,
\end{aligned}
\end{equation}
indicating $\log G\left(d(k), \binom{k}{2}\right)$ is a sub-linear function.
\end{proof}
\end{lemma}

\begin{lemma}\label{lemma: G is sub-exponential}
When $d(k)$ scales strictly slower than $k/\log k$, for any $\nu > 0$, 
\begin{equation}
\lim_{k\rightarrow \infty} \frac{G\left(d(k), \binom{k}{2}\right)}{\exp(\nu k)} = 0
\end{equation}
holds.

\begin{proof}
From Lemma \ref{lemma: upper bound of distinct subspace set size}, since $d(k)$ grows strictly slower than $k/\log k$, $\log G\left(d(k), \binom{k}{2}\right)$ grows sub-linearly. Therefore, for any $\nu > 0$, there exists $k_0$ such that for any $k > k_0$, 
\begin{equation}
\frac{\log G\left(d(k), \binom{k}{2}\right)}{k}< \frac{\nu}{2}
\end{equation}
holds; thus, for any $\nu$ and $k > k_0$, 
\begin{equation}
\begin{aligned}
&\lim_{k\rightarrow \infty} \frac{G\left(d(k), \binom{k}{2}\right)}{\exp(\nu k)} \\
& = \lim_{k\rightarrow \infty} \exp\left(\log G\left(d(k), \binom{k}{2}\right) - \nu k\right) \\
& = \lim_{k\rightarrow \infty} \exp\left(\left(\frac{\log G\left(d(k), \binom{k}{2}\right)}{k} - \nu\right) k\right) \\
& \leq \lim_{k\rightarrow \infty} \exp\left(- \frac{\nu}{2} k\right) \\
& = 0.
\end{aligned}
\end{equation}
\end{proof}
\end{lemma}

\subsection{Main proof}\label{app: proof of prop 3, subsec: proof}
\begin{proof}
We derive a contradiction by assuming the existence of a $d(k)$-width network and a generator head $f^{\mathrm{gen}}: \mathbb{R}^{d(k)} \rightarrow \mathbb{R}^{2k+2}$ that generates $\texttt{Shuffle-dyck}_k$. Denote the matrix of the generator head by 
\begin{equation}
W^{\mathrm{gen}} = \begin{bmatrix}
\vdots \\
\mathbf{w}_1^\top \\
\vdots \\
\mathbf{w}_k^\top \\
\vdots
\end{bmatrix}\begin{matrix*}
\scalebox{1.6}{\}}k \text{ dim.}\\
\scalebox{1.05}{$\left.\rule{0pt}{2.4em}\right\}$} k \text{ dim.} \\
\scalebox{1.6}{\}}2 \text{ dim.} \rule{0pt}{1.3em}
\end{matrix*} \in \mathbb{R}^{(2k+2)\times d(k)}.
\end{equation}

Take into account the $2^k$ vectors $\{\mathbf{x}_l \}_{l\in [2^k]}$ corresponding to the input strings described in the proof sketch; that is, $\operatorname{bin}_t(l) = 0$ means the type $t$ is closed, while $\operatorname{bin}_t(l) = 1$ means the type $t$ is unclosed. Here, the generator head satisfies 
\begin{equation}
\mathbf{w}_{t_0}^\top \mathbf{x}_{l} > \mathbf{w}_{t_1}^\top \mathbf{x}_{l},
\end{equation}
for any $l \in \left[2^k \right]$ and for any $t_1 \in \operatorname{id}_1(l), t_0 \in \operatorname{id}_0(l)$. This is because the logit for the unclosed type must be greater than that for the closed type.

Consider the subspace set defined by $\{\mathbf{w}_{t}\}_{t=1}^{k}$, since $\mathbf{x}_{l} \in R(l)$, from Lemma \ref{lemma: necessary number of distinct subspace}, there exists at least $\left\lfloor \sqrt{2}^k\right\rfloor$-size distinct subspace set.

However, the generator head can create at most $G\left(d(k), \binom{k}{2}\right)$-size separated regions in $\mathbb{R}^{d(k)}$, leading a contradiction: the number of separable regions increases strictly slower than the necessary size of distinct subspace set from Lemma \ref{lemma: G is sub-exponential}.
\end{proof}

\section{Proof of Proposition \ref{proposition: create pseudo starting signal without bos}}\label{app: proof of lemma create pseudo starting signal without bos}

\begin{proposition}[Restated]
Assume that there exists a linear subspace such that the embeddings are distinct from each other and have a constant $2$-norm. Then, there exists a Transformer block without a starting token that creates a pseudo starting signal $\hat{s}_i$ for any string $w_{1:n}$ whose first two tokens are different, where
\begin{equation}
\hat{s}_i = \begin{cases}
1 & \text{if } i = 1 \\
0 & \text{otherwise }
\end{cases}.
\end{equation}

Specifically, this block transforms the constants-padded vector $\hat{\mathbf{x}}_i$ as follows:
\begin{equation}
\hat{\mathbf{x}}_{i}  = 
\begin{bmatrix}
\mathbf{x}_{i} \\
\vdots \\
0 
\end{bmatrix} \mapsto \begin{bmatrix}
\mathbf{x}_{i} \\
\vdots \\
\hat{s}_i
\end{bmatrix}.
\end{equation}
\end{proposition}

\begin{proof}
Assume the extended input representation $\hat{\mathbf{x}}_{i} \in \mathbb{R}^{d_\mathrm{model}^\prime}$, where $d_\mathrm{model}^\prime = 2d_\mathrm{model} + 2$, instead of the original representation $\mathbf{x}_{i} \in \mathbb{R}^{d_\mathrm{model}}$ as follows:
\begin{equation}
\hat{\mathbf{x}}_{i}  = 
\begin{bmatrix}
\mathbf{x}_{i} \\
\mathbf{0} \\
1 \\
0
\end{bmatrix}
\begin{matrix*}[l]
\} d_\mathrm{model} \text{ dim.}\\
\} d_\mathrm{model} \text{ dim.}\\
\} 1 \text{ dim.}\\
\} 1 \text{ dim.}
\end{matrix*}.
\end{equation}

The attention layer, leveraging the uniform attention, transforms the vector into
\begin{equation}
\begin{bmatrix}
\mathbf{x}_{i} \\
\frac{1}{i}\sum_{j=1}^i \mathbf{x}_{j} \\
1 \\
0
\end{bmatrix}.
\end{equation}

Then, in the feed-forward network layer, the first linear transformation calculates $\mathbf{x}_{i} - \frac{1}{i}\sum_{j=1}^i \mathbf{x}_{j}$. The norm of this vector is $0$ if and only if $\mathbf{x}_{i} = \frac{1}{i}\sum_{j=1}^i \mathbf{x}_{j}$. Thanks to the subsequent layer normalization, the larger the $2$-norm of the vector $\mathbf{x}_{i} - \frac{1}{i}\sum_{j=1}^i \mathbf{x}_{j}$ becomes, the smaller the transformed value of the constant $1$ becomes. This allows the subsequent ReLU activations and the linear transformation to implement the conditional branch. We then show the specific implementation.

Set the parameters $W_1, W_2 \in \mathbb{R}^{d_\mathrm{model}^\prime \times d_\mathrm{model}^\prime}$ and $\boldsymbol{\beta}, \boldsymbol{\gamma} \in \mathbb{R}^{d_\mathrm{model}^\prime}$ as follows:
\begin{align}
&W_1 = \begin{bmatrix}
I & -I & \mathbf{0} & \mathbf{0} \\
O & O & \mathbf{0} & \mathbf{0} \\
\mathbf{0}^\top & \mathbf{0}^\top & 1 & 0 \\
\mathbf{0}^\top & \mathbf{0}^\top & 0 & 0 
\end{bmatrix}, \\
& W_2 = \begin{bmatrix}
O & O & \mathbf{0} & \mathbf{0} \\
O & O & \mathbf{0} & \mathbf{0} \\
\mathbf{0}^\top & \mathbf{0}^\top & 0 & 0 \\
\mathbf{0}^\top & \mathbf{0}^\top & \frac{1}{\epsilon} & 0
\end{bmatrix}, \\
& \boldsymbol{\beta} = 
\begin{bmatrix}
\mathbf{0} \\
\mathbf{0} \\
- 1 + \epsilon \\
0
\end{bmatrix}, \\
& \boldsymbol{\gamma} = \sqrt{\frac{1}{d_\mathrm{model}^\prime}}\mathbf{1},
\end{align}
where $\epsilon$ is a positive constant.

Then, the output of the RMS layer normalization is given by 
\begin{equation}
\begin{aligned}
&\operatorname{LN}_{\mathrm{RMS}}\left(W_1 \mathbf{h}_i \right) \\
&=\operatorname{LN}_{\mathrm{RMS}}\left(
\begin{bmatrix}
\mathbf{x}_{i} - \frac{1}{i}\sum_{j=1}^i \mathbf{x}_{j}\\
\mathbf{0} \\
1 \\
0
\end{bmatrix} \right) \\
&= \frac{1}{\|W_1 \mathbf{h}_i\|_2}
\begin{bmatrix}
\mathbf{x}_{i} - \frac{1}{i}\sum_{j=1}^i \mathbf{x}_{j}\\
\mathbf{0} \\
1 \\
0
\end{bmatrix} + \begin{bmatrix}
\mathbf{0} \\
\mathbf{0} \\
-1 + \epsilon \\
0
\end{bmatrix} \\
&= \frac{1}{\|W_1 \mathbf{h}_i\|_2}
\begin{bmatrix}
\mathbf{x}_{i} - \frac{1}{i}\sum_{j=1}^i \mathbf{x}_{j}\\
\mathbf{0}\\
1 - \|W_1 \mathbf{h}_i\|_2\left(1 - \epsilon\right) \\
0
\end{bmatrix}.
\end{aligned}
\end{equation}

Therefore, the output of the feed-forward network layer is given by
\begin{equation}
\begin{aligned}
&W_2 \left[\frac{1}{\|W_1 \mathbf{h}_i\|_2}
\begin{bmatrix}
\mathbf{x}_{i} - \frac{1}{i}\sum_{j=1}^i \mathbf{x}_{j}\\
\mathbf{0} \\
1 - \|W_1 \mathbf{h}_i\|_2\left(1 - \epsilon\right) \\
0
\end{bmatrix}\right]_+ \\
&= \begin{bmatrix}
\mathbf{0} \\
\mathbf{0} \\
0 \\
\frac{1}{\epsilon}\left[\frac{1}{\|W_1 \mathbf{h}_i\|_2} - 1 + \epsilon\right]_+
\end{bmatrix} \\
&= \begin{bmatrix}
\mathbf{0} \\
\mathbf{0} \\
0 \\
\mathbb{I}\left[\mathbf{x}_i = \frac{1}{i}\sum_{j=1}^i \mathbf{x}_j\right]
\end{bmatrix} 
\end{aligned}
\end{equation}

The reason why the last equality holds is explained below: since
\begin{equation}
\begin{aligned}
\|W_1 \mathbf{h}_i\|_2^2 &= \left\|\mathbf{x}_{i} - \frac{1}{i}\sum_{j=1}^i \mathbf{x}_{j}\right\|_2^2 +  1^2\\
&\begin{cases}
= 1 & \text{if } \mathbf{x}_{i} = \frac{1}{i}\sum_{j=1}^i \mathbf{x}_{j} \\
> 1 & \text{if } \mathbf{x}_{i} \neq \frac{1}{i}\sum_{j=1}^i \mathbf{x}_{j} \\
\end{cases},
\end{aligned}
\end{equation}
the entry $1$ is transformed to $1$ if $\mathbf{x}_{i} = \frac{1}{i}\sum_{j=1}^i \mathbf{x}_{j}$; otherwise, the entry becomes less than $1$. Therefore, given a sufficiently small constant $\epsilon$,
\begin{equation}
\begin{aligned}
&\frac{1}{\|W_1 \mathbf{h}_i\|_2} - 1 + \epsilon\\
&\begin{cases}
= \epsilon & \text{if }\mathbf{x}_{i} = \frac{1}{i}\sum_{j=1}^i \mathbf{x}_{j} \\
< 0 & \text{otherwise}
\end{cases}
\end{aligned}
\end{equation}
holds, indicating 
\begin{equation}
\frac{1}{\epsilon} \left[\frac{1}{\|W_1 \mathbf{h}_i\|_2} - 1 + \epsilon\right]_+ = \mathbb{I}\left[\mathbf{x}_i = \frac{1}{i}\sum_{j=1}^i \mathbf{x}_j\right].
\end{equation}

Finally, we give a proof of the following proposition: $\mathbf{x}_{1} \neq \mathbf{x}_{2} \Leftrightarrow \mathbf{x}_{i} \neq \frac{1}{i}\sum_{j=1}^i \mathbf{x}_{j}$ (for $i\geq 2$).
\begin{equation}\begin{aligned}
\mathbf{x}_{i} = \frac{1}{i}\sum_{j=1}^i \mathbf{x}_{j} &\Longrightarrow \left\langle \mathbf{x}_{i}, \frac{1}{i}\sum_{j=1}^i \mathbf{x}_{j}\right\rangle = 1 \\ 
&\Longrightarrow \frac{1}{i} \sum_{j=1}^i \left\langle\mathbf{x}_{i}, \mathbf{x}_{j}\right\rangle = 1 \\
&\Longrightarrow \forall j \in [i]. \mathbf{x}_i = \mathbf{x}_j.
\end{aligned}\end{equation}
This is because 
\begin{equation}
\begin{aligned}
&\frac{1}{i} \sum_{j=1}^i \left\langle\mathbf{x}_{i}, \mathbf{x}_{j}\right\rangle \leq \frac{1}{i} \sum_{j=1}^i \left\|\mathbf{x}_{i}\| \| \mathbf{x}_{j}\right\|  = 1
\end{aligned}
\end{equation}
holds for any $i$, and the equality holds if and only if $\forall j\in[i].\mathbf{x}_i = \mathbf{x}_j$ holds. On the other hand, the converse is straightforward. Therefore, 
\begin{align}
&\forall i \geq 2. \mathbf{x} \neq \frac{1}{i}\sum_{j=1}^i \mathbf{x}_{j} \\
&\Longleftrightarrow \forall i \geq 2. \exists j \in [i]. \mathbf{x}_i \neq \mathbf{x}_j \\
&\Longleftrightarrow \mathbf{x}_1 \neq \mathbf{x}_2, 
\end{align}
indicating that when $\mathbf{x}_1 \neq \mathbf{x}_2$, A Transformer block can create a pseudo starting signal $\hat{s}_i$ by itself.
\end{proof}

\section{Proof of Corollary \ref{corollary: transformers without bos recognize dyck_k}}\label{app: dyck recognition without bos}

Here, we present a method to construct a Transformer without positional encoding and the BOS token that recognizes the $\texttt{Dyck}_k$ language for an input string whose first two characters are different. 

\begin{corollary}[Restated, Transformers without a starting token, $\texttt{Dyck}_k$ probabilistic recognition]
There exists a $9$-layer causal Transformer without a starting token that recognizes the $\texttt{Dyck}_k$ language with probability at least $1-1/k$.
\end{corollary}

\begin{proof}
Here, for clarity, we omit the specific implementation except that of the fourth layer. Instead, we outline the construction.

Firstly, using the starting token created by Proposition \ref{proposition: create pseudo starting signal without bos}, create pseudo positional encoding, which allows the Transformer to compute the same representation $\mathbf{x}_{i-1}$ as used in Theorem \ref{theorem: transformers with bos recognize dyck_k}.
The proof of Theorem \ref{theorem: transformers with bos recognize dyck_k} does not require the query to assign an attention score on itself; thus, it is possible to calculate $\operatorname{q}(w_{0:i})$ by making the query matrix focus on $\mathbf{x}_i^{(\ell)}$ and the key/value matrices focus on $\mathbf{x}_{i-1}^{(\ell)}$. 
Moreover, by focusing solely on $\mathbf{x}_{i-1}^{(\ell)}$, it is possible to compute $\bigwedge_{j=0}^{i-1}Q(w_{0:j-1})$ in the same manner as described in Appendix \ref{app: dyck recognition with bos}. Finally, to check whether the input string is a prefix for $\texttt{Dyck}_k$, it is sufficient to compute $Q(w_{0:i}) \wedge \bigwedge_{j=0}^{i-1}Q(w_{0:j-1})$.

Specifically, the following nine layers can recognize the $\texttt{Dyck}_k$ language.

\begin{description}
    \item[First layer] creates a pseudo starting signal $\hat{s}_i$ using Proposition \ref{proposition: create pseudo starting signal without bos}.
    \item[Second and third layers] create vectors corresponding $\phi(i-1)$ and $\phi(i)$, respectively.
    \item[Fourth layer] computes the same representation $\mathbf{x}_{i-1}$ as in Appendix \ref{app: vector representation}.
    \item[Fifth and sixth layers] create vectors corresponding $\operatorname{d}(w_{0:i-1})$ and $\operatorname{d}(w_{0:i})+1$, respectively.
 
    \item[Seventh and eighth layers] compute $\operatorname{q}_{i-1}$ and $\operatorname{q}_{i-1}$, which correspond to the propositional variables $Q(w_{0:i-1})$ and $Q(w_{0:i})$, respectively.

    \item[Ninth layer] computes $Q(w_{0:i}) \wedge \bigwedge_{j=0}^{i-1}Q(w_{0:j}) \wedge \operatorname{d}(w_{0:i}) + 1= 1$.
\end{description}

\subsection{How to compute $\mathbf{x}_{i-1}$}
The attention layer in the fourth layer, leveraging the positional encoding, computes
\begin{align}
\mathbf{h}_i^{(4)} &=
\begin{bmatrix}
\vdots  \\
\hat{s}_i \\
\vdots  \\
\mathbf{t}_{i_\mathrm{prev}} \\
o_{i_\mathrm{prev}} \\
\vdots 
\end{bmatrix}, 
\end{align}
where 
\begin{align}
& \mathbf{t}_{i_\mathrm{prev}}  = \begin{cases}
    \mathbf{t}_{1} & \text{if } i = 1 \\
    \mathbf{t}_{i-1} & \text{if } i > 1 
\end{cases}, \\
& o_{i_\mathrm{prev}}  = \begin{cases}
    \mathbf{o}_{1} & \text{if } i = 1 \\
    \mathbf{o}_{i-1} & \text{if } i > 1 
\end{cases}. 
\end{align}
In the subsequent feed-forward network layer, set the parameters $W_1^{(4)}, W_2^{(4)}\in \mathbb{R}^{d_\mathrm{model} \times d_{\mathrm{model}}}$ and $\boldsymbol{\beta}^{(4)}, \boldsymbol{\gamma}^{(4)} \in \mathbb{R}^{d_\mathrm{model}}$ as follows:
\begin{align}
W_1^{(4)} &= 
\begin{bmatrix}
\cdots & \mathbf{0} & \cdots & I & \mathbf{0} & \cdots \\
\cdots & \mathbf{0} & \cdots & -I & \mathbf{0} & \cdots \\
\cdots & 0 & \cdots & \mathbf{0}^\top & 1 & \cdots \\
\cdots & 0 & \cdots & \mathbf{0}^\top & -1 & \cdots \\
\cdots & C^{(4)} & \cdots & \mathbf{0}^\top & 0 & \cdots \\
 & \vdots & & \vdots & \vdots & 
\end{bmatrix},\\
W_2^{(4)} &= 
\begin{bmatrix}
\vdots & \vdots & \vdots & \vdots & \vdots &  \\
I & -I & \mathbf{0} & \mathbf{0} & \mathbf{0} & \cdots \\
\mathbf{0}^\top & \mathbf{0}^\top & 1 & -1 & 0 & \cdots \\
\mathbf{0}^\top & \mathbf{0}^\top & 0 & 0 & 1 & \cdots \\
\vdots & \vdots & \vdots & \vdots & \vdots & 
\end{bmatrix},\\
\boldsymbol{\beta}^{(4)} &= \mathbf{0}, \\
\boldsymbol{\gamma}^{(4)} &= \frac{1}{{\sqrt{d_\mathrm{model}}}}\begin{bmatrix}
\sqrt{2\left(\lceil \log_2 k \rceil+1\right)} \mathbf{1} \\
\sqrt{2\left(\lceil \log_2 k \rceil+1\right)} \mathbf{1} \\
\sqrt{2\left(\lceil \log_2 k \rceil+1\right)} \\
\sqrt{2\left(\lceil \log_2 k \rceil+1\right)} \\
1 \\
\vdots 
\end{bmatrix}.
\end{align}

Given a sufficiently large constant $C^{(4)}$, since we obtain 
\begin{equation}
\begin{aligned}
&\operatorname{LN}_\mathrm{RMS} \left(
\begin{bmatrix}
\mathbf{t}_{i_\mathrm{prev}} \\
-\mathbf{t}_{i_\mathrm{prev}} \\
o_{i_\mathrm{prev}} \\
-o_{i_\mathrm{prev}} \\
C^{(4)} \hat{s}_i \\
\vdots 
\end{bmatrix}
\right) \\
&=\begin{cases}
\begin{bmatrix}
\mathbf{0} \\
\mathbf{0} \\
0 \\
0 \\
1 \\
\vdots 
\end{bmatrix} & \text{if } \hat{s}_i = 1 \\
\begin{bmatrix}
\mathbf{t}_{i_\mathrm{prev}} \\
-\mathbf{t}_{i_\mathrm{prev}} \\
o_{i_\mathrm{prev}} \\
-o_{i_\mathrm{prev}} \\
0 \\
\vdots 
\end{bmatrix} & \text{otherwise }
\end{cases},
\end{aligned}
\end{equation}
the output of the RMS layer normalization is given by:
\begin{equation}
\begin{aligned}
&\operatorname{LN}_\mathrm{RMS} \left(W_1^{(4)} \mathbf{h}_i^{(4)}\right)= \begin{bmatrix} 
\mathbf{t}_{i-1} \\
-\mathbf{t}_{i-1} \\
o_{i-1} \\
-o_{i-1} \\
s_{i-1} \\
\vdots
\end{bmatrix}.
\end{aligned}
\end{equation}

Therefore, the output of the feed-forward network layer is given by:
\begin{equation}
\begin{aligned}
& W_2^{(4)}\left[\operatorname{LN}_\mathrm{RMS} \left(W_1^{(4)} \mathbf{h}_i^{(4)}\right)\right]_+ \\
&= W_2^{(4)}\begin{bmatrix}
[\mathbf{t}_{i-1}]_+ \\
[-\mathbf{t}_{i-1}]_+  \\
[o_{i-1}]_+  \\
[-o_{i-1}]_+  \\
s_{i-1} \\
\vdots 
\end{bmatrix} = \begin{bmatrix}
\vdots \\
\mathbf{t}_{i-1}\\
o_{i-1}  \\
s_{i-1} \\
\vdots 
\end{bmatrix},
\end{aligned}
\end{equation}
indicating that $\mathbf{x}_{i-1}$ can be computed correctly.

Finally, since the probability of outputting the same type of open bracket as the first one is $\frac{r}{k}$, the first two characters are different with at least a probability of $1 - \frac{1}{k}$, which completes the proof.
\end{proof}

\section{Proof of Corollary \ref{corollary: transformers without bos generate dyck_k}}\label{app: dyck language generation without bos}
Here, we present a method to construct a Transformer without positional encoding and $\texttt{<bos>}$ that realizes the $\texttt{Dyck}_k$ language generation process $p_{\texttt{Dyck}_k}(\cdot; q, r, \boldsymbol{\pi})$. 

\begin{corollary}[Restated, Transformers without a starting token, $\texttt{Dyck}_k$ subset generation]
There exists a $7$-layer causal Transformer without a starting token that can generate a subset of $\texttt{Dyck}_k$ where the first two characters are different; that is, the Transformer can generate all possible subsequent sequences when there is an input string whose first two characters are different.
\end{corollary}
\begin{proof}

We assume that the output probabilities take the following form, which is the same as \ref{app: dyck language generation with bos}: 
\begin{equation}
\begin{bmatrix}
p_{\langle_1} \\
\vdots \\
p_{\langle_k} \\
p_{\rangle_1} \\
\vdots \\
p_{\rangle_k} \\
p_{\texttt{<bos>}} \\
p_{\texttt{<eos>}} \\
\end{bmatrix}.
\end{equation}

We omit the specific implementation. Instead, we outline the construction for clarity as in Appendix \ref{app: dyck recognition without bos}. Similar to the Transformer without $\texttt{<bos>}$ that recognizes the $\texttt{Dyck}_k$ language, the query vector does not need to assign attention to itself; thus, by making the query matrix focus on $\mathbf{x}_{i}^{(\ell)}$ and the key/value matrices focus on $\mathbf{x}_{i-1}^{(\ell)}$, the desired behavior can be realized.

Specifically, the five layers described below generate the $\texttt{Dyck}_k$ language when the first two characters of the input string are different.

\begin{description}
    \item[First layer] creates a pseudo starting signal $\hat{s}_i$ using Proposition \ref{proposition: create pseudo starting signal without bos}.
    \item[Second and third layers] create vectors corresponding $\phi(i-1)$ and $\phi(i)$, respectively.
    \item[Fourth layer] computes the same representation $\mathbf{x}_{i-1}$ as in Appendix \ref{app: vector representation}.
    \item[Fifth and sixth layers] create vectors corresponding $\operatorname{d}(w_{0:i-1})$ and $\operatorname{d}(w_{0:i})$, respectively.
    \item[Seventh layer] enables each closed bracket to fetch the nearest depth-matched open bracket.
\end{description}

\end{proof}

\section{Validity of Treating Softmax Attention as Hardmax Attention}\label{app: validity of treating softmax as hardmax}
In our constructive proofs, we occasionally treat softmax attention as hardmax attention. In this section, we validate these theoretical results; that is, we show that if the vector fetched by hardmax attention is included in the finite set of candidates, the subsequent feed-forward network layer can transform the vector obtained by softmax attention into that obtained by hardmax attention when the assigned attention weight exceeds a certain threshold.
Here, we discuss the fourth attention layer described in Appendix \ref{app: dyck recognition with bos, subsec: fourth layer} as an example. 

\subsection{Threshold of attention strength}

\begin{lemma}\label{lemma: condition identifiable attention threshold}
Assume a vector set $\{\mathbf{y}_i\}_{i=1}^n \subset \{1, 0, -1\}^d$. Let $\tilde{\mathbf{y}}_{i, \mathrm{H}}$ and $\tilde{\mathbf{y}}_{i, \mathrm{S}}$ be the vectors obtained by hardmax attention and softmax attention among $\{\mathbf{y}_i\}_{i=1}^n$, respectively.
Then, regarding softmax attention, if a query assigns the attention greater than $\frac{2}{3}$ on the target token, the vector obtained by hardmax attention $\tilde{\mathbf{y}}_{i, \mathrm{H}}$ can be identified by referencing $\tilde{\mathbf{y}}_{i, \mathrm{S}}$.

\begin{proof}
When a query assigns greater than $\frac{2}{3}$ on the target token, there exists $\rho > \frac{2}{3}$ and ${\mathbf{y}}_{\mathrm{CH}}$ in the convex hull of $\{{\mathbf{y}}_{i}\}$ such that 
\begin{equation}
\tilde{\mathbf{y}}_{i, \mathrm{S}} = \rho \tilde{\mathbf{y}}_{i, \mathrm{H}} + (1-\rho) {\mathbf{y}}_{\mathrm{CH}}.
\end{equation}
Since absolute value of each elements in ${\mathbf{y}}_{\mathrm{CH}}$ is at most $1$, regarding the $l$-th element $(\tilde{\mathbf{y}}_{i, \mathrm{S}})_l$ of $\tilde{\mathbf{y}}_{i, \mathrm{S}}$, 
\begin{align}
&\rho (\tilde{\mathbf{y}}_{i, \mathrm{H}})_l - (1-\rho) \leq (\tilde{\mathbf{y}}_{i, \mathrm{S}})_l, \\
&(\tilde{\mathbf{y}}_{i, \mathrm{S}})_l \leq \rho (\tilde{\mathbf{y}}_{i, \mathrm{H}})_l + (1-\rho)
\end{align}
hold.
Therefore, when 
\begin{align}
&\begin{cases}
\rho \cdot (-1) + (1-\rho) < \rho \cdot 0 - (1-\rho) \\
\rho \cdot 0 + (1-\rho) < \rho \cdot 1 - (1-\rho)
\end{cases}\\
&\Longleftrightarrow \rho > \frac{2}{3}
\end{align}
is satisfied, the original values are identifiable. 
\end{proof}
\end{lemma}

\subsection{Recovering the original value with feed-forward network layer}
Here, we show how to implement the feed-forward network layer that recovers the vectors obtained by hardmax attention and realizes the computation in the fourth layer. From Lemma \ref{lemma: condition identifiable attention threshold}, although it is feasible if attention weight is greater than $\frac{2}{3}$, we set this threshold to $\frac{4}{5}$ as an example.

Intuitively, we implement a function similar to a step function using the ReLU activations to recover vectors that include errors produced by the prior softmax attention. Specifically, since the  element of the $\left[\mathbf{t}_i - \mathbf{\tilde{t}}_i\right]_+$ and $o_i + 1$ take values of $0, 1, 2$, we implement the recovering function $\operatorname{Recov}(y)$ as follows:
\begin{equation}\label{eq: recovering function}
\begin{aligned}
&\operatorname{Recov}(y)\\
&=\left[\frac{y}{\epsilon}-\frac{9}{20\epsilon}\right]_+ +\left[\frac{y}{\epsilon}-\left(1+\frac{9}{20\epsilon}\right)\right]_+ \\
&\quad+\left[\frac{y}{\epsilon}-\frac{19}{15\epsilon}\right]_+-\left[\frac{y}{\epsilon}-\left(1+\frac{19}{15\epsilon}\right)\right]_+
\end{aligned}.
\end{equation}
The behavior of this function is described in Figure \ref{fig:recovering function}.

\begin{figure}[h]  
    \includegraphics[width=\columnwidth]{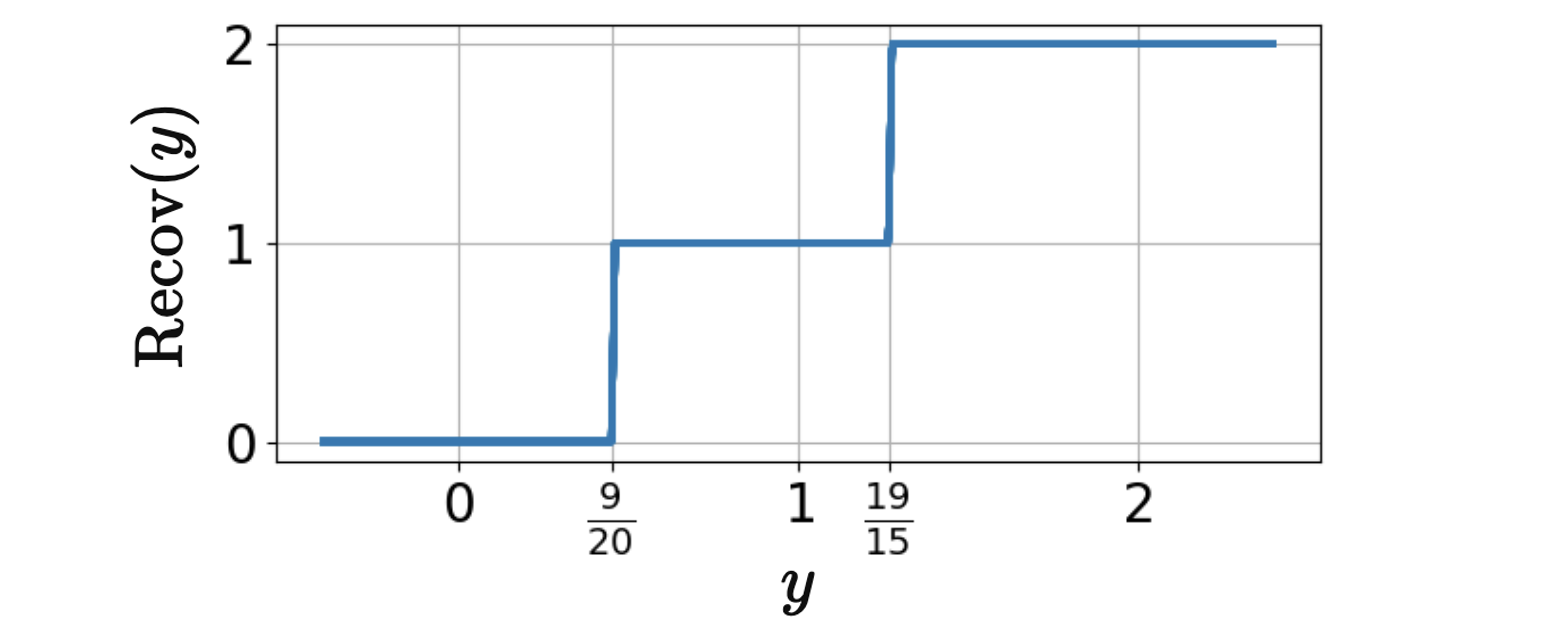}
    \caption{Illustration of the recovering function.}
    \label{fig:recovering function}
\end{figure}

Then, we show the specific implementation that realizes the recovering function.
We omit the unnecessary dimensions of input vector $\mathbf{h}_i^{(4)}$ in this layer as follows:
\begin{equation}
\mathbf{h}_i^{(4)} = \begin{bmatrix}
\mathbf{t}_i \\
o_i \\
\vdots \\
1 \\
\vdots \\
\tilde{\mathbf{t}}_i \\
\vdots 
\end{bmatrix}.
\end{equation}

Set the parameter $W_1^{(4)} \in \mathbb{R}^{d_\mathrm{model} \times d_\mathrm{model}}$ and $\boldsymbol{\beta}^{(4)},\boldsymbol{\gamma}^{(4)} \in \mathbb{R}^{d_\mathrm{model}}$ as follows:
\begin{align}
&W_1^{(4)} = 
\begin{bmatrix}
I & \mathbf{0} & \cdots & \mathbf{0} & \cdots & -I & \cdots \\
I & \mathbf{0} & \cdots & \mathbf{0} & \cdots & -I & \cdots \\
I & \mathbf{0} & \cdots & \mathbf{0} & \cdots & -I & \cdots \\
I & \mathbf{0} & \cdots & \mathbf{0} & \cdots & -I & \cdots \\
-I & \mathbf{0} & \cdots & \mathbf{0} & \cdots & I & \cdots \\
-I & \mathbf{0} & \cdots & \mathbf{0} & \cdots & I & \cdots \\
-I & \mathbf{0} & \cdots & \mathbf{0} & \cdots & I & \cdots \\
-I & \mathbf{0} & \cdots & \mathbf{0} & \cdots & I & \cdots \\
\mathbf{0}^\top & 1 & \cdots & 1 & \cdots & \mathbf{0}^\top & \cdots \\
\mathbf{0}^\top & 1 & \cdots & 1 & \cdots & \mathbf{0}^\top & \cdots \\
\mathbf{0}^\top & 1 & \cdots & 1 & \cdots & \mathbf{0}^\top & \cdots \\
\mathbf{0}^\top & 1 & \cdots & 1 & \cdots & \mathbf{0}^\top & \cdots \\
\mathbf{0}^\top & 0 & \cdots & C & \cdots & \mathbf{0}^\top & \cdots \\
\mathbf{0}^\top & 0 & \cdots & C & \cdots & \mathbf{0}^\top & \cdots \\
\vdots & \vdots & & \vdots & & \vdots & 
\end{bmatrix}, \\
&\boldsymbol{\beta}^{(4)} = \begin{bmatrix}
- \frac{9}{20\epsilon}\mathbf{1} \\
- (1 + \frac{9}{20\epsilon})\mathbf{1} \\
- \frac{19}{15\epsilon} \mathbf{1}\\
- (1 + \frac{19}{15\epsilon})\mathbf{1} \\
- \frac{9}{20\epsilon}\mathbf{1} \\
- (1 + \frac{9}{20\epsilon})\mathbf{1} \\
- \frac{19}{15\epsilon} \mathbf{1}\\
- (1 + \frac{19}{15\epsilon})\mathbf{1} \\
- \frac{9}{20\epsilon} \\
- (1 + \frac{9}{20\epsilon})\\
- \frac{19}{15\epsilon} \\
- (1 + \frac{19}{15\epsilon})\\
- \frac{9}{20\epsilon} \\
- (1 + \frac{9}{20\epsilon})\\
\mathbf{0}
\end{bmatrix}, \\
&\boldsymbol{\gamma}^{(4)} = \frac{1}{\epsilon}\sqrt{\frac{2C^2}{d_{\mathrm{model}}}}\begin{bmatrix}
\mathbf{1} \\
\mathbf{1} \\
\mathbf{1} \\
\mathbf{1} \\
\mathbf{1} \\
\mathbf{1} \\
\mathbf{1} \\
\mathbf{1} \\
1 \\
1 \\
1 \\
1 \\
\frac{1}{C} \\
\frac{1}{C} \\
\mathbf{0}
\end{bmatrix},
\end{align}
where $\epsilon$ is a positive constant that satisfies $\epsilon < \frac{1}{10}$.
Then, we obtain 
\begin{equation}
\begin{aligned}
&\operatorname{LN}_{\mathrm{RMS}} \left(W_1^{(4)} \mathbf{h}_i^{(4)}\right) \\
&= \operatorname{LN}_{\mathrm{RMS}} \left(\begin{bmatrix}
\mathbf{t}_i - \tilde{\mathbf{t}}_i \\
\mathbf{t}_i - \tilde{\mathbf{t}}_i \\
\mathbf{t}_i - \tilde{\mathbf{t}}_i \\
\mathbf{t}_i - \tilde{\mathbf{t}}_i \\
-(\mathbf{t}_i - \tilde{\mathbf{t}}_i )\\
-(\mathbf{t}_i - \tilde{\mathbf{t}}_i )\\
-(\mathbf{t}_i - \tilde{\mathbf{t}}_i )\\
-(\mathbf{t}_i - \tilde{\mathbf{t}}_i )\\
o_i+1\\
o_i+1\\
o_i+1\\
o_i+1\\
C \\
C \\
\vdots 
\end{bmatrix}\right)
\end{aligned}
\end{equation}
Here, 
\begin{align}
& \sup \frac{1}{\operatorname{RMS}\left(W_1^{(4)} \mathbf{h}_i^{(4)}\right)} = \sqrt{\frac{d_{\mathrm{model}}}{2C^2}},\\
& \begin{aligned}
& \inf \frac{1}{\operatorname{RMS}\left(W_1^{(4)} \mathbf{h}_i^{(4)}\right)} \\
& \,\,\,\,= \sqrt{\frac{d_{\mathrm{model}}}{8\cdot 2^2\lceil \log_2k \rceil + 4\cdot 2^2 + 2\cdot C^2}} \\
& \,\,\,\,= \sqrt{\frac{d_{\mathrm{model}}}{{2C^2}}}\left(1 + \frac{16\lceil \log_2k \rceil+ 8}{C^2}\right)^{-\frac{1}{2}} \\
& \,\,\,\,\geq \sqrt{\frac{d_{\mathrm{model}}}{{2C^2}}}\left(1 - \delta_C\right),
\end{aligned}
\end{align}
where $\delta_C = \frac{8\lceil \log_2k \rceil+ 4}{C^2}$.

Moreover, since 
\begin{equation}
\tilde{t}_{i, l} \in \begin{cases}
\left[-\frac{1}{5}, \frac{1}{5}\right] & \text{if } {t}_{i, l} = 0 \\
\left[\frac{3}{5}, 1\right] & \text{if } {t}_{i, l} = 1 \\
\left[-1, -\frac{3}{5}\right] & \text{if } {t}_{i, l} = -1\\
\end{cases},
\end{equation}

\begin{equation}
\begin{aligned}
&t_{i, l} - \tilde{t}_{i, l} \\
&\in \begin{cases}
\left[\frac{8}{5}, 2\right] & \text{if } {t}_{i, l} - \overline{t}_{i, l} = 2 \\
\left[\frac{3}{5}, \frac{6}{5}\right] & \text{if } {t}_{i, l} - \overline{t}_{i, l} = 1 \\
\left[-\frac{2}{5}, \frac{2}{5}\right] & \text{if } {t}_{i, l} - \overline{t}_{i, l} = 0 \\
\left[-\frac{6}{5}, -\frac{3}{5}\right] & \text{if } {t}_{i, l} - \overline{t}_{i, l} = -1 \\
\left[-2, -\frac{8}{5}\right] & \text{if } {t}_{i, l} - \overline{t}_{i, l} = -2 \\
\end{cases}.
\end{aligned}
\end{equation}
Therefore,
\begin{equation}
\begin{aligned}
&\sqrt{\frac{2C^2}{d_{\mathrm{model}}}}\cdot\frac{t_{i, l} - \tilde{t}_{i, l}}{\operatorname{RMS}\left(W_1^{(4)} \mathbf{h}_i^{(4)}\right)}\\
&\in \begin{cases}
\left[\frac{8(1-\delta_C)}{5}, 2\right] & \text{if } {t}_{i, l} - \overline{t}_{i, l} = 2 \\
\left[\frac{3(1-\delta_C)}{5}, \frac{6}{5}\right] & \text{if } {t}_{i, l} - \overline{t}_{i, l} = 1 \\
\left[-\frac{2}{5}, \frac{2}{5}\right] & \text{if } {t}_{i, l} - \overline{t}_{i, l} = 0 \\
\left[-\frac{6}{5}, -\frac{3(1-\delta_C)}{5}\right] & \text{if } {t}_{i, l} - \overline{t}_{i, l} = -1 \\
\left[-2, -\frac{8(1-\delta_C)}{5}\right] & \text{if } {t}_{i, l} - \overline{t}_{i, l} = -2 \\
\end{cases}.
\end{aligned}
\end{equation}
By setting $C$ to satisfy
\begin{equation}
\begin{aligned}
&\begin{cases}
\frac{6}{5} < \frac{8(1-\delta_C)}{5} \\
\frac{2}{5} < \frac{3(1-\delta_C)}{5} 
\end{cases} \\
&\Leftrightarrow \delta_C < \frac{1}{4} \\
&\Leftrightarrow C > 4\sqrt{2\lceil \log_2k \rceil+ 1},
\end{aligned}
\end{equation}
these five intervals become disjoint.
We proceed with our discussion under the assumption $C > 2\sqrt{6}\cdot \sqrt{2\lceil \log_2k \rceil+ 1} \Leftrightarrow \delta_C < \frac{1}{6}$ as an example. In this case, 
\begin{equation}
\begin{aligned}
&\sqrt{\frac{2C^2}{d_{\mathrm{model}}}}\cdot\frac{t_{i, l} - \tilde{t}_{i, l}}{\operatorname{RMS}\left(W_1^{(4)} \mathbf{h}_i^{(4)}\right)} \\
&\quad \in\begin{cases}
\left[\frac{4}{3}, 2\right] & \text{if } {t}_{i, l} - \overline{t}_{i, l} = 2 \\
\left[\frac{1}{2}, \frac{6}{5}\right] & \text{if } {t}_{i, l} - \overline{t}_{i, l} = 1 \\
\left[-\frac{2}{5}, \frac{2}{5}\right] & \text{if } {t}_{i, l} - \overline{t}_{i, l} = 0 \\
\left[-\frac{6}{5}, -\frac{1}{2}\right] & \text{if } {t}_{i, l} - \overline{t}_{i, l} = -1 \\
\left[-2, -\frac{4}{3}\right] & \text{if } {t}_{i, l} - \overline{t}_{i, l} = -2 \\
\end{cases}.
\end{aligned}
\end{equation}
Similarly, 
\begin{align}
& \begin{aligned}
&\sqrt{\frac{2C^2}{d_{\mathrm{model}}}}\cdot\frac{o_i + 1}{\operatorname{RMS}\left(W_1^{(4)} \mathbf{h}_i^{(4)}\right)} \\
&\qquad\begin{cases}
= 0 & \text{if } o_i + 1 = 0 \\
\in \left[\frac{5}{6}, 1\right] & \text{if } o_i + 1 = 1 \\
\in \left[\frac{5}{3}, 2\right] & \text{if } o_i + 1 = 2 \\
\end{cases}, \\
\end{aligned} \\
& \begin{aligned}
&\left(\sqrt{\frac{2C^2}{d_{\mathrm{model}}}}\cdot\frac{1}{C}\right)\cdot\frac{C}{\operatorname{RMS}\left(W_1^{(4)} \mathbf{h}_i^{(4)}\right)} \\
&\qquad \in \left[\frac{5}{6}, 1\right]
\end{aligned}
\end{align}
hold.
Therefore, by implementing the recovering function $\operatorname{Recov}(\cdot)$ defined in Equation \eqref{eq: recovering function}, the vectors obtained by hardmax attention are recovered. Specifically, by setting $W_2^{(4)} \in \mathbb{R}^{d_\mathrm{model} \times d_\mathrm{model}}$ 
\begin{equation}
\begin{aligned}
W_2^{(4)} &= \begin{bmatrix}
\vdots \\
\mathbf{w}_2^{(4)\top} \\
\vdots
\end{bmatrix},
\end{aligned}
\end{equation}
where 
\begin{equation}
\mathbf{w}_2^{(4)} = 2\begin{bmatrix}
-\mathbf{1} \\
\mathbf{1} \\
-\mathbf{1} \\
\mathbf{1} \\
-\mathbf{1} \\
\mathbf{1} \\
-\mathbf{1} \\
\mathbf{1} \\
2(\lceil \log_2 k \rceil + 1) \\
-2(\lceil \log_2 k \rceil + 1) \\
2(\lceil \log_2 k \rceil + 1) \\
-2(\lceil \log_2 k \rceil + 1) \\
1 \\
-1 \\
\mathbf{0}
\end{bmatrix},
\end{equation}
the desired vector is obtained; that is, the feed-forward network layer computes 
\begin{equation}
\begin{bmatrix}
\vdots \\
\operatorname{\tilde{q}} (w_{0:i}) \\
\vdots
\end{bmatrix},
\end{equation}
where 
\begin{equation}
\begin{aligned}
\operatorname{\tilde{q}} (w_{0:i}) =& -2\|\mathbf{t}_i - \overline{\mathbf{t}}_i\|_1 \\
& + 4(\lceil \log_2 k \rceil + 1)(o_i + 1) + 2,
\end{aligned}
\end{equation}
which the same expression as $\operatorname{q}'(w_{0:i})$ in Equation \eqref{eq: definition of q'}; that is,  $\operatorname{\tilde{q}} (w_{0:i})$ satisfies the conditions described in \eqref{eq: conditions for q_i}. This indicates that the hardmax attention is dispensable for our constructive proof.

\section{Rationale behind Architectural Modification}
Although the architecture adopted in \citet{yao-etal-2021-self} uses the conventional layer normalization, we adopt an architecture with the RMS layer normalization. This is not only because recent models such as  LLama \citep{touvron2023llamaopenefficientfoundation} and Llama 2 \citep{touvron2023llama2openfoundation} adopt the RMS layer normalization but also because we try to make our constructive proofs more concise.
In this section, we show that this change does not affect the critical aspects of our proofs; in other words, we give a proof that any transformation achievable with the RMS layer normalization can be achieved with the layer normalization. 

\begin{lemma}
For any feed-forward network with the RMS layer normalization and a hidden size of $d_{\mathrm{model}}$, there exists a feed-forward network with the layer normalization and a hidden size of $2 d_{\mathrm{model}}$ such that their outputs are identical.

\begin{proof}
Consider the feed-forward network layer with the RMS layer normalization parameterized by $W_1, W_2, \boldsymbol{\beta}$ and $\boldsymbol{\gamma}$, the output becomes
\begin{equation}
\begin{aligned}
&W_2[\operatorname{LN}_{\mathrm{RMS}}(W_1 \mathbf{x})]_+ \\
&= \frac{1}{\operatorname{RMS}\left(W_1 \mathbf{x}\right)} W_2 [\boldsymbol{\gamma} \odot (W_1 \mathbf{x}) + \boldsymbol{\beta}]_+.
\end{aligned}
\end{equation}

This output is realized by the feed-forward network layer with the layer normalization parameterized by $W_1^\prime = \begin{bmatrix}
W_1 \\
-W_1
\end{bmatrix}, W_2^\prime = \begin{bmatrix}
W_2 & O \\
\end{bmatrix},  \boldsymbol{\beta}^\prime = \begin{bmatrix} \boldsymbol{\beta} \\ \mathbf{0}\end{bmatrix}$ and $\boldsymbol{\gamma}^\prime = \begin{bmatrix} \boldsymbol{\gamma} \\ \mathbf{1}\end{bmatrix}$.
This is because 
\begin{equation}
\begin{aligned}
& W_2^\prime\left[\operatorname{LN}\left(W_1^\prime \mathbf{x}\right)\right]_+ \\
&= \frac{1}{\operatorname{RMS}\left(W_1^\prime \mathbf{x}\right)}\begin{bmatrix}
W_2 & O \\
\end{bmatrix} \begin{bmatrix}
[\boldsymbol{\gamma} \odot (W_1 \mathbf{x}) + \boldsymbol{\beta}]_+ \\
[ (-W_1 \mathbf{x})]_+
\end{bmatrix}. \\
&=\frac{1}{\operatorname{RMS}\left(W_1^\prime \mathbf{x}\right)} W_2 [\boldsymbol{\gamma} \odot (W_1 \mathbf{x}) + \boldsymbol{\beta}]_+.
\end{aligned}
\end{equation}
Here,
\begin{equation}
\begin{aligned}
&\operatorname{RMS}\left(W_1^\prime \mathbf{x}\right) \\
&= \sqrt{\frac{1}{2d_\mathrm{model}} (\|W_1 \mathbf{x}\|_2^2 + \|-W_1 \mathbf{x}\|_2^2)} \\
&= \sqrt{\frac{1}{d_\mathrm{model}}}\|W_1\mathbf{x}\|_2 \\
&= \operatorname{RMS}\left(W_1 \mathbf{x}\right),
\end{aligned}
\end{equation}
indicating that the two transformations produce the same outputs.
\end{proof}
\end{lemma}

\section{Extension to Architecture with The QK Normalization}
The QK normalization \citep{pmlr-v202-dehghani23a} applies the layer normalization \citep{ba2016layernormalization} individually to both the query and key vectors to stabilize training. Specifically, concerning calculating attention scores, the QK normalization uses 
\begin{equation}
\langle \operatorname{LN}(W_Q \mathbf{x}_{i_q}), \operatorname{LN}(W_K \mathbf{x}_{i_k})\rangle
\end{equation}
instead of 
\begin{equation}
\langle W_Q \mathbf{x}_{i_q}, W_K \mathbf{x}_{i_k}\rangle, 
\end{equation}
where $\operatorname{LN}(\cdot)$ is the layer normalization \citep{ba2016layernormalization} parameterized by $\boldsymbol{\beta}, \boldsymbol{\gamma} \in \mathbb{R}^{d_\mathrm{model}}$. Specifically,
\begin{equation}
\operatorname{LN}(\mathbf{y}) = \boldsymbol{\gamma} \odot \frac{\mathbf{y} - \mu(\mathbf{y}) \mathbf{1}}{\sigma(\mathbf{y})}+\boldsymbol{\beta},
\end{equation}
where 
\begin{align}
&\mu(\mathbf{y}) = \frac{1}{d_\mathrm{model}}\sum_{d=1}^{d_\mathrm{model}} y_d, \\
&\sigma(\mathbf{y}) = \sqrt{\frac{1}{d_\mathrm{model}}\sum_{d=1}^{d_\mathrm{model}} (y_d - \mu(\mathbf{y}))^2}. 
\end{align}

In this section, we show in two steps that the QK normalization can be incorporated into our constructive proof:
\begin{enumerate}
    \item We give a proof that the layer normalization and the RMS layer normalization are equivalent when they are incorporated into the QK normalization regarding their expressive power. 
    \item We show that our theoretical results also hold even when the QK normalization with the RMS layer normalization is incorporated into the architecture.
\end{enumerate}

For clarity, denote the QK normalization with the layer normalization by $\texttt{QK-LN}$ and the QK normalization with the RMS layer normalization by $\texttt{QK-RMSLN}$.

\subsection{Equivalence of the layer normalization and the RMS layer normalization under the QK normalization}
We give a proof that for any attention layer with $\texttt{QK-LN}$, there exists an attention layer with $\texttt{QK-RMSLN}$ that produces the same output (Lemma \ref{lemma: convert QK-LN to QK-RMSLN}). Similarly, we also show that the converse holds: for any attention layer with $\texttt{QK-RMSLN}$, there exists an attention layer with $\texttt{QK-LN}$ that produces the same output (Lemma \ref{lemma: convert QK-RMSLN to QK-LN}). 
Note that it is sufficient to show the existence of a network that outputs the same attention scores.

\begin{lemma}\label{lemma: convert QK-LN to QK-RMSLN}
For any attention layer with $\texttt{QK-LN}$, there exists an attention layer with $\texttt{QK-RMSLN}$ that produces the same output for any given input.
\begin{proof}
Assume the attention layer with $\texttt{QK-LN}$ parameterized by $\boldsymbol{\beta}_Q, \boldsymbol{\gamma}_Q$, $\boldsymbol{\beta}_K, \boldsymbol{\gamma}_K$, 
\begin{align}
&W_Q = \begin{bmatrix}
\mathbf{w}_{Q,1}^\top \\
\vdots \\
\mathbf{w}_{Q,d_\mathrm{model}}^\top 
\end{bmatrix}, W_K = \begin{bmatrix}
\mathbf{w}_{K,1}^\top \\
\vdots \\
\mathbf{w}_{K,d_\mathrm{model}}^\top 
\end{bmatrix}. 
\end{align}
Then, the attention layer with $\texttt{QK-RMSLN}$ parameterized by $\boldsymbol{\beta}_Q^\prime = \boldsymbol{\beta}_Q, \boldsymbol{\gamma}_Q^\prime = \boldsymbol{\gamma}_Q, \boldsymbol{\beta}_K^\prime = \boldsymbol{\beta}_K, \boldsymbol{\gamma}_K^\prime = \boldsymbol{\gamma}_K$, 
\begin{align}
& W_Q^\prime = \begin{bmatrix}
\mathbf{w}_{Q,1}^\top - \left(\frac{1}{d_\mathrm{model}}\sum_{d=1}^{d_\mathrm{model}} \mathbf{w}_{Q,d}^\top\right)\\
\vdots \\
\mathbf{w}_{Q,d_\mathrm{model}}^\top - \left(\frac{1}{d_\mathrm{model}}\sum_{d=1}^{d_\mathrm{model}} \mathbf{w}_{Q,d}^\top\right)
\end{bmatrix}, \\ 
& W_K^\prime = \begin{bmatrix}
\mathbf{w}_{K,1}^\top - \left(\frac{1}{d_\mathrm{model}}\sum_{d=1}^{d_\mathrm{model}} \mathbf{w}_{K,d}^\top\right)\\
\vdots \\
\mathbf{w}_{K,d_\mathrm{model}}^\top - \left(\frac{1}{d_\mathrm{model}}\sum_{d=1}^{d_\mathrm{model}} \mathbf{w}_{K,d}^\top\right)
\end{bmatrix} 
\end{align}
produces the same attention scores. The reasons are detailed below:
\begin{equation}
\begin{aligned}
&W_Q^\prime \mathbf{x}_{i_q} \\
&= \begin{bmatrix}
\mathbf{w}_{Q,1}^\top - \left(\frac{1}{d_\mathrm{model}}\sum_{d=1}^{d_\mathrm{model}} \mathbf{w}_{Q,d}^\top\right)\\
\vdots \\
\mathbf{w}_{Q,d_\mathrm{model}}^\top - \left(\frac{1}{d_\mathrm{model}}\sum_{d=1}^{d_\mathrm{model}} \mathbf{w}_{Q,d}^\top\right)
\end{bmatrix}\mathbf{x}_{i_q} \\
&= \begin{bmatrix}
\mathbf{w}_{Q,1}^\top \mathbf{x}_{i_q}\\
\vdots \\
\mathbf{w}_{Q,d_\mathrm{model}}^\top \mathbf{x}_{i_q}
\end{bmatrix} - 
\begin{bmatrix}
\frac{1}{d_\mathrm{model}}\sum_{d=1}^{d_\mathrm{model}} \mathbf{w}_{Q,d}^\top \mathbf{x}_{i_q} \\
\vdots \\
\frac{1}{d_\mathrm{model}}\sum_{d=1}^{d_\mathrm{model}} \mathbf{w}_{Q,d}^\top \mathbf{x}_{i_q} 
\end{bmatrix}\\
&= \begin{bmatrix}
\mathbf{w}_{Q,1}^\top \mathbf{x}_{i_q}\\
\vdots \\
\mathbf{w}_{Q,d_\mathrm{model}}^\top \mathbf{x}_{i_q}
\end{bmatrix} -  \mathbf{1} \mu\left(\begin{bmatrix}
\mathbf{w}_{Q,1}^\top \mathbf{x}_{i_q}\\
\vdots \\
\mathbf{w}_{Q,d_\mathrm{model}}^\top \mathbf{x}_{i_q}
\end{bmatrix}\right). 
\end{aligned}
\end{equation}
Here, since $\operatorname{RMS}\left(\mathbf{y} - \mathbf{1}\mu(\mathbf{y})\right) = \sigma(\mathbf{y})$ holds,
\begin{equation}
\begin{aligned}
& \operatorname{LN}_{\mathrm{RMS}}(\mathbf{y} - \mathbf{1}\mu(\mathbf{y})) = \operatorname{LN}(\mathbf{y})
\end{aligned}
\end{equation}
holds for any $\mathbf{y} \in \mathbb{R}^{d_\mathrm{model}}$. Therefore, 
\begin{equation}
\begin{aligned}
\operatorname{LN}_{\mathrm{RMS}}(W_Q^\prime \mathbf{x}_{i_q}) = \operatorname{LN}(W_Q \mathbf{x}_{i_q})
\end{aligned}
\end{equation}
holds. Similarly, 
\begin{equation}
\operatorname{LN}_{\mathrm{RMS}}(W_K^\prime \mathbf{x}_{i_k}) = \operatorname{LN}(W_K \mathbf{x}_{i_k})
\end{equation}
also holds, indicating that 
\begin{equation}
\begin{aligned}
& \langle\operatorname{LN}_{\mathrm{RMS}}(W_Q^\prime \mathbf{x}_{i_q}), \operatorname{LN}_{\mathrm{RMS}}(W_K^\prime \mathbf{x}_{i_k})\rangle \\
& = \langle\operatorname{LN}(W_Q \mathbf{x}_{i_k}), \operatorname{LN}(W_K \mathbf{x}_{i_q})\rangle.
\end{aligned}
\end{equation}
\end{proof}
\end{lemma}

\begin{lemma}\label{lemma: convert QK-RMSLN to QK-LN}
For any attention layer with $\texttt{QK-RMSLN}$, there exists an attention layer with $\texttt{QK-LN}$ that produces the same output for any given input.
\begin{proof}
Assume an attention layer with $\texttt{QK-RMSLN}$ parameterized by $\boldsymbol{\beta}_Q, \boldsymbol{\gamma}_Q$, $\boldsymbol{\beta}_K, \boldsymbol{\gamma}_K$, 
\begin{align}
&W_Q = \begin{bmatrix}
\mathbf{w}_{Q,1}^\top \\
\vdots \\
\mathbf{w}_{Q,d_\mathrm{model}}^\top 
\end{bmatrix}, &W_K = \begin{bmatrix}
\mathbf{w}_{K,1}^\top \\
\vdots \\
\mathbf{w}_{K,d_\mathrm{model}}^\top 
\end{bmatrix}. 
\end{align}
Then, $\texttt{QK-LN}$ parameterized by 
\begin{align}
& \boldsymbol{\beta}_Q^{\prime\prime} = \begin{bmatrix}
\boldsymbol{\beta}_Q \\
-\boldsymbol{\beta}_Q \\
\mathbf{0}
\end{bmatrix}, \boldsymbol{\gamma}_Q^{\prime\prime} = \sqrt{\frac{2}{3}}\begin{bmatrix}
\boldsymbol{\gamma}_Q \\
\boldsymbol{\gamma}_Q \\
\mathbf{1}
\end{bmatrix},\\
& \boldsymbol{\beta}_K^{\prime\prime} = \begin{bmatrix}
\boldsymbol{\beta}_K \\
\mathbf{0} \\
-\boldsymbol{\beta}_K
\end{bmatrix}, \boldsymbol{\gamma}_K^{\prime\prime} = \sqrt{\frac{2}{3}}\begin{bmatrix}
\boldsymbol{\gamma}_K \\
\mathbf{1} \\
\boldsymbol{\gamma}_K 
\end{bmatrix},\\
& W_Q^{\prime\prime} = \begin{bmatrix}
\mathbf{w}_{Q,1}^\top \\
\vdots \\
\mathbf{w}_{Q,d_\mathrm{model}}^\top \\
-\mathbf{w}_{Q,1}^\top \\
\vdots \\
-\mathbf{w}_{Q,d_\mathrm{model}}^\top \\
\mathbf{0}^\top \\
\vdots \\
\mathbf{0}^\top
\end{bmatrix}, W_K^{\prime\prime} = \begin{bmatrix}
\mathbf{w}_{K,1}^\top \\
\vdots \\
\mathbf{w}_{K,d_\mathrm{model}}^\top \\
\mathbf{0}^\top \\
\vdots \\
\mathbf{0}^\top \\
-\mathbf{w}_{K,1}^\top \\
\vdots \\
-\mathbf{w}_{K,d_\mathrm{model}}^\top
\end{bmatrix}
\end{align}
produces the same attention scores. The reasons are detailed below:
\begin{equation}
\begin{aligned}
& W_Q^{\prime\prime} \mathbf{x}_{i_q} = \begin{bmatrix}
\mathbf{w}_{Q,1}^\top\mathbf{x}_{i_q} \\
\vdots \\
\mathbf{w}_{Q,d_\mathrm{model}}^\top \mathbf{x}_{i_q}\\
-\mathbf{w}_{Q,1}^\top \mathbf{x}_{i_q} \\
\vdots \\
-\mathbf{w}_{Q,d_\mathrm{model}}^\top \mathbf{x}_{i_q}\\
0 \\
\vdots \\
0 
\end{bmatrix}.
\end{aligned}
\end{equation}
Since $\mu(W_Q^{\prime\prime} \mathbf{x}_{i_q}) = 0$, the results of applying the layer normalization and the RMS layer normalization to this vector are identical; that is, 
\begin{align}
&\begin{aligned}
& \operatorname{LN}(W_Q^{\prime\prime} \mathbf{x}_{i_q}) \\
& =\operatorname{LN}_\mathrm{RMS}(W_Q^{\prime\prime} \mathbf{x}_{i_q}) \\
& = \boldsymbol{\gamma}_Q^{\prime\prime} \odot \frac{1}{\sqrt{\frac{2\|W_Q \mathbf{x}_{i_q}\|_2^2}{3d_{\mathrm{model}}}}}\begin{bmatrix}
\mathbf{w}_{Q,1}^\top\mathbf{x}_{i_q} \\
\vdots \\
\mathbf{w}_{Q,d_\mathrm{model}}^\top \mathbf{x}_{i_q}\\
-\mathbf{w}_{Q,1}^\top \mathbf{x}_{i_q} \\
\vdots \\
-\mathbf{w}_{Q,d_\mathrm{model}}^\top \mathbf{x}_{i_q}\\
0 \\
\vdots \\
0 
\end{bmatrix} + \boldsymbol{\beta}_Q^{\prime\prime} \\
& = \begin{bmatrix}
\boldsymbol{\gamma}_Q \odot \frac{1}{\sqrt{\frac{\|W_Q \mathbf{x}_{i_q}\|_2^2}{d_{\mathrm{model}}}}}
\begin{bmatrix}
\mathbf{w}_{Q,1}^\top\mathbf{x}_{i_q} \\
\vdots \\
\mathbf{w}_{Q,d_\mathrm{model}}^\top \mathbf{x}_{i_q}
\end{bmatrix} + \boldsymbol{\beta}_Q \\
- \boldsymbol{\gamma}_Q \odot \frac{1}{\sqrt{\frac{\|W_Q \mathbf{x}_{i_q}\|_2^2}{d_{\mathrm{model}}}}}
\begin{bmatrix}
\mathbf{w}_{Q,1}^\top\mathbf{x}_{i_q} \\
\vdots \\
\mathbf{w}_{Q,d_\mathrm{model}}^\top \mathbf{x}_{i_q}
\end{bmatrix} - \boldsymbol{\beta}_Q \\
\begin{bmatrix}
0 \\
\vdots \\
0 
\end{bmatrix}
\end{bmatrix}\\
& = \begin{bmatrix}
\operatorname{LN}_{\mathrm{RMS}}(W_Q \mathbf{x}_{i_q}) \\
-\operatorname{LN}_{\mathrm{RMS}}(W_Q \mathbf{x}_{i_q}) \\
\mathbf{0}
\end{bmatrix}.
\end{aligned}
\end{align}
Similarly, 
\begin{align}
&\begin{aligned}
& \operatorname{LN}(W_K^{\prime\prime} \mathbf{x}_{i_k}) = \begin{bmatrix}
\operatorname{LN}_{\mathrm{RMS}}(W_K \mathbf{x}_{i_k}) \\
\mathbf{0} \\
-\operatorname{LN}_{\mathrm{RMS}}(W_K \mathbf{x}_{i_k}) 
\end{bmatrix},
\end{aligned}
\end{align}
indicating
\begin{equation}
\begin{aligned}
& \langle \operatorname{LN}(W_Q^{\prime\prime} \mathbf{x}_{i_q}), \operatorname{LN}(W_K^{\prime\prime} \mathbf{x}_{i_k})\rangle \\
& = \langle \operatorname{LN}_{\mathrm{RMS}}(W_Q \mathbf{x}_{i_q}), \operatorname{LN}_{\mathrm{RMS}}(W_K \mathbf{x}_{i_k})\rangle.
\end{aligned}
\end{equation}
\end{proof}
\end{lemma}

\subsection{Incorporating the QK normalization with the RMS layer normalization to our constructive proof}
We use the attention layers for two purposes in our constructive proofs: (i) used to create positional vectors $(\cos \phi(i), \sin \phi(i))$ and depth vectors $(\cos \theta(\operatorname{d}), \sin \theta(\operatorname{d}))$ and (ii) used as an approximation of hardmax attention to focus on a single token. In the following sections, we show how to incorporate QK normalization into our constructive proofs.

\subsubsection*{(i) When used to create positional and depth vectors}
When the attention layers are used to create positional vectors or depth vectors, an attention score of $a$ is assigned to the BOS token and $0$ to other tokens. We then show that this operation can be implemented also in the architecture with the QK normalization.

We omit the unnecessary dimensions of input vector $\mathbf{x}_i^{(\ell)}$ in this layer as follows:
\begin{equation}
\mathbf{x}_i^{(\ell)} = \begin{bmatrix}
\vdots \\
s_i \\
1 \\
\vdots 
\end{bmatrix}.
\end{equation}
Then, the attention layer with $\texttt{QK-RMSLN}$ parameterized by 
\begin{align}
& \boldsymbol{\beta}_Q^{(\ell)} = \mathbf{0},  \boldsymbol{\gamma}_Q^{(\ell)} = \sqrt{\frac{1}{d_\mathrm{model}}}\mathbf{1}, \\
& \boldsymbol{\beta}_K^{(\ell)} = \mathbf{0}, \boldsymbol{\gamma}_K^{(\ell)} = a \sqrt{\frac{1}{d_\mathrm{model}}}\mathbf{1}, \\
& W_Q^{(\ell)} = 
\begin{bmatrix}
\cdots & 0 & 1 & \cdots \\
 & \vdots & \vdots & 
\end{bmatrix}, \\
& W_K^{(\ell)} = 
\begin{bmatrix}
\cdots & 1 & 0 & \cdots \\
 & \vdots & \vdots & 
\end{bmatrix}
\end{align}
produces the desired attention scores. This is because
\begin{align}
& \begin{aligned}
&\operatorname{LN}_{\mathrm{RMS}}\left(W_Q^{(\ell)} \mathbf{x}_{i_q}^{(\ell)}\right) \\
&\quad = \operatorname{LN}_{\mathrm{RMS}}\left(\begin{bmatrix}
1 \\
\mathbf{0}
\end{bmatrix}\right) \\
&\quad = \sqrt{\frac{1}{d_\mathrm{model}}}\mathbf{1}\odot \begin{bmatrix}
\sqrt{d_\mathrm{model}} \\
\mathbf{0}
\end{bmatrix} + \mathbf{0}\\
&\quad = \begin{bmatrix}
1 \\
\mathbf{0}
\end{bmatrix},
\end{aligned}\\
& \begin{aligned}
& \operatorname{LN}_{\mathrm{RMS}}\left(W_K^{(\ell)} \mathbf{x}_{i_k}^{(\ell)}\right) \\
&\quad = \operatorname{LN}_{\mathrm{RMS}}\left(\begin{bmatrix}
s_{i_k} \\
\mathbf{0}
\end{bmatrix}\right) \\
&\quad = a \sqrt{\frac{1}{d_\mathrm{model}}} \mathbf{1} \odot \begin{bmatrix}
\sqrt{d_\mathrm{model}} \cdot s_{i_k} \\
\mathbf{0}
\end{bmatrix} + \mathbf{0} \\
&\quad = \begin{bmatrix}
s_{i_k} \cdot a \\
\mathbf{0}
\end{bmatrix},
\end{aligned}
\end{align}
indicating 
\begin{equation}
\begin{aligned}
&\left\langle \operatorname{LN}_{\mathrm{RMS}}\left(W_K^{(\ell)} \mathbf{x}_{i_k}^{(\ell)}\right),\operatorname{LN}_{\mathrm{RMS}}\left(W_Q^{(\ell)} \mathbf{x}_{i_q}^{(\ell)}\right) \right\rangle \\
&\quad = s_{i_k}\cdot a \\
&\quad =\begin{cases}
a & \text{if } w_i = \texttt{<bos>} \\
0 & \text{otherwise}
\end{cases}.
\end{aligned}
\end{equation}

\subsubsection*{(ii) When used as an approximation of hardmax attention}
Here, we show that we can fix the $2$-norm of the query/key vectors to prevent changes in the attention scores. In our proofs, we only use the following values: $o_i, s_i, 1, \cos(\cdot), \sin(\cdot), \operatorname{q}_i$. Except for $\operatorname{q}_i$, by adding the complementary values described below to the new dimensions of the query/key vector, we can fix the $2$-norm. $o_i$ and $s_i$ are the complementary values, so are $\cos(\cdot)$ and $\sin(\cdot)$. This is because $o_i^2 + s_i^2 = 1$ and $\cos^2(\cdot)+\sin^2(\cdot) = 1$ hold. For example, if we set the attention parameters $W_Q, W_K$ to satisfy 
\begin{align}
& W_Q \mathbf{x}_{i_q} = \begin{bmatrix}
o_{i_q} \\
\cos \phi(i_q)  \\
\mathbf{0} 
\end{bmatrix}, W_K \mathbf{x}_{i_k} = \begin{bmatrix}
s_{i_k} \\
1 \\
\mathbf{0} 
\end{bmatrix},
\end{align}
by modifying them into 
\begin{align}
&W'_Q \mathbf{x}_{i_q} = \begin{bmatrix}
o_{i_q} \\
\cos \phi(i_q)  \\
s_{i_q} \\
\sin \phi(i_q)  \\
0 \\
\mathbf{0} 
\end{bmatrix},W'_K \mathbf{x}_{i_k} = \begin{bmatrix}
s_{i_k} \\
1 \\
0 \\
0 \\
o_{i_k} \\
\mathbf{0} 
\end{bmatrix},
\end{align}
we can fix the $2$-norm of the query/key vectors. This is because 
\begin{align}
&\begin{aligned}
&\|W'_Q \mathbf{x}_{i_q}\|_2^2 \\
&= o_{i_q}^2 + s_{i_q}^2 + \cos^2\phi(i_q) + \sin^2\phi(i_q) = 2,
\end{aligned} \\
&\begin{aligned}
&\|W'_K \mathbf{x}_{i_k}\|_2^2 \\
&= s_{i_k}^2 + o_{i_k}^2 + 1^2 = 2.
\end{aligned}
\end{align}
Then, by setting $\boldsymbol{\beta}_Q = \boldsymbol{\beta}_K = \mathbf{0}, \boldsymbol{\gamma}_Q = \frac{\|W'_Q \mathbf{x}_{i_q}\|_2}{\sqrt{d_\mathrm{model}}} \mathbf{1}$ and $\boldsymbol{\gamma}_K = \frac{\|W'_K \mathbf{x}_{i_k}\|_2}{\sqrt{d_\mathrm{model}}} \mathbf{1}$, we obtain
\begin{align}
&\begin{aligned}
&\operatorname{LN}_{\mathrm{RMS}}\left(W'_Q \mathbf{x}_{i_q}\right) \\
&\quad =\boldsymbol{\gamma}_Q\odot\frac{\sqrt{d_\mathrm{model}}}{\|W'_Q \mathbf{x}_{i_q}\|_2}\left(W'_Q \mathbf{x}_{i_q}\right) \\
&\quad  = W'_Q \mathbf{x}_{i_q}, 
\end{aligned} \\
&\begin{aligned}
&\operatorname{LN}_{\mathrm{RMS}}\left(W'_K \mathbf{x}_{i_k}\right) \\
&\quad =\boldsymbol{\gamma}_K\odot\frac{\sqrt{d_\mathrm{model}}}{\|W'_K \mathbf{x}_{i_k}\|_2}\left(W'_K \mathbf{x}_{i_k}\right) \\
&\quad  = W'_K \mathbf{x}_{i_k},
\end{aligned}
\end{align}
indicating that the attention layer with the QK normalization produces the same attention scores.

In contrast, $\operatorname{q}_i$ is used in the key vector defined in Appendix \ref{app: dyck recognition with bos, subsec: fifth layer}, and it is hard to fix the $2$-norm. However, by setting $\boldsymbol{\beta}_K = 0, \boldsymbol{\gamma}_K = \frac{1}{\sqrt{d_\mathrm{model}}} \mathbf{1}$, we obtain
\begin{equation}
\begin{aligned}
&\operatorname{LN}_{\mathrm{RMS}}\left(W_K^{(5)} \mathbf{x}_{i_k}^{(5)}\right) \\
&=\operatorname{LN}_{\mathrm{RMS}}\left(\begin{bmatrix}
- \operatorname{q}_{i_k} \\
 \operatorname{q}_{0}\cdot s_{i_k} \\
\mathbf{0}
\end{bmatrix}\right) \\
&= \begin{cases}
\begin{bmatrix}
- \frac{1}{\sqrt{2}}\\
\frac{1}{\sqrt{2}}\\
\mathbf{0} 
\end{bmatrix} & \text{if } i_k = 0 \\
\begin{bmatrix}
- 1\\
0 \\
\mathbf{0} 
\end{bmatrix} & \text{if } \operatorname{q}_{i_k} > 0 \\
\begin{bmatrix}
1 \\
0 \\
\mathbf{0} 
\end{bmatrix} & \text{if } \operatorname{q}_{i_k} < 0
\end{cases},
\end{aligned}
\end{equation}
which leads to the same result.

\section{Details of Experiments}\label{app: experiments}

\subsection{Full evaluation on $\texttt{Dyck}_k$}

\subsubsection*{Setup}
The $\texttt{Dyck}_k$ and $\texttt{Shuffle-Dyck}_k$ language datasets are generated by $p_{\texttt{Dyck}_k}\left(\cdot; q, r, \boldsymbol{\pi}\right)$ parameterized with $q=0.5, r=0.9, \boldsymbol{\pi} = \frac{1}{k} \mathbf{1}$ and $p_{\texttt{Shuffle-Dyck}_k}\left(\cdot; q, r, \boldsymbol{\pi}, \boldsymbol{\overline{\pi}}\right)$ parameterized with $q=0.3, r=0.97, \boldsymbol{\pi} = \frac{1}{k}\mathbf{1}, \boldsymbol{\overline{\pi}} = \frac{1}{k} \mathbf{1}$, respectively. Compared to $\texttt{Dyck}_k$, we set the smaller value for $q$ and the larger value for $r$ in the case of $\texttt{Shuffle-Dyck}_k$ for two reasons: (i) to avoid the situation where all types remain unclosed in the later positions, making the task trivial and (ii) to prevent the generation of an excessive number of short sequences due to the small $q$.

Following \citet{yao-etal-2021-self}, we set $n_\mathrm{max} = 700$ and $d_\mathrm{model} = 30$, and we truncated the sequences longer than $n_\mathrm{max}$. We generated $100,000$ sequences as training data, with an additional $10,000$ sequences (equivalent to $10$\% of the training data) used for both validation and test datasets. Note that for the test data, we create out-of-distribution (OOD) sequences with respect to length, generating sequences up to a maximum length of $1.2 \times n_\mathrm{max}$.

We conducted experiments by varying the presence of the BOS token ($\{\texttt{BOS}, \texttt{NoBOS}\}$), the presence of positional encoding ($\{\texttt{PE}, \texttt{NoPE}\}$), the number of brackets types ($\{1,2,4,8,16\}$ for $\texttt{Dyck}_k$ and $\{2,4,8,16\}$ for $\texttt{Shuffle-Dyck}_k$), and the number of layers ($\{1,2,3,4,5,6,7,8,9,10\}$). Here, each $10$-layer model has a size of $0.05$M parameters.

We set the learning rate candidates to $\{3\mathrm{e}{-3}, 3\mathrm{e}{-4}\}$ and evaluated the performance of the model that achieved the lowest validation loss. We report the average performance over $5$ runs with different random seeds.

\subsubsection*{Metric}
Following \citet{hewitt-etal-2020-rnns}, \citet{yao-etal-2021-self}, we evaluated the model performance using the conditional probability of outputting the correct closing brackets on test data. In addition, we also reported the TV distance from the true language generation process.

The test data contains sequences whose length is up to $1.2 \times n_\mathrm{max}$. We regard tokens at position $i \leq n_\mathrm{max}$ as in-distribution (ID) data and tokens at position $n_\mathrm{max} < i \leq 1.2 \times n_\mathrm{max}$ as out-of-distribution (OOD) data, thereby we evaluate the generalization ability with respect to sequence length.

Figure \ref{fig:closed_acc_all} and \ref{fig:closed_acc_all_shuffle} show the average test accuracy of generating the correct closed bracket on $\texttt{Dyck}_k$ and $\texttt{Shuffle-Dyck}_k$, respectively. Moreover, Figure \ref{fig:tv_all} and \ref{fig:tv_all_shuffle} show the average test TV distance on $\texttt{Dyck}_k$ and $\texttt{Shuffle-Dyck}_k$, respectively.

\begin{table}
  \centering
  \begin{tabular}{lc}
    \hline
    \textbf{Hyperparameter} & \textbf{Value} \\
    \hline
    Model parameters & \\
    \hline
    Number of attention heads & $1$ \\
    Embedding dimension $d_\mathrm{model}$ & $30$ \\
    Use of bias terms & \texttt{False} \\
    \makecell{Affine Transformation in the \\$\quad$RMS layer normalization} & \texttt{True} \\
    Window size & $1,024$ \\
    Activation function & ReLU \\
    \hline
    Training parameters \\
    \hline
    Dropout rate & $0.0$ \\
    Batch size & $16$ \\
    Learning rate & $\{3\mathrm{e}{-3}, 3\mathrm{e}{-4}\}$\\
    Gradient accumulation steps & $2$ \\
    Weight decay & $0.0$ \\
    Adam parameters $(\beta_1, \beta_2)$ & $(0.9, 0.999)$ \\
    Maximum iterations & $3,000$ \\
    Warmup iterations & $0$ \\
    Learning rate decay & $\texttt{False}$ \\
    \hline
  \end{tabular}
  \caption{Hyperparameter configuration for experiments on the $\texttt{Dyck}_k$ language.}
  \label{tab: experiment, dyck, hyperparameters}
\end{table}

\begin{figure*}[ht]
  \includegraphics[width=\linewidth]{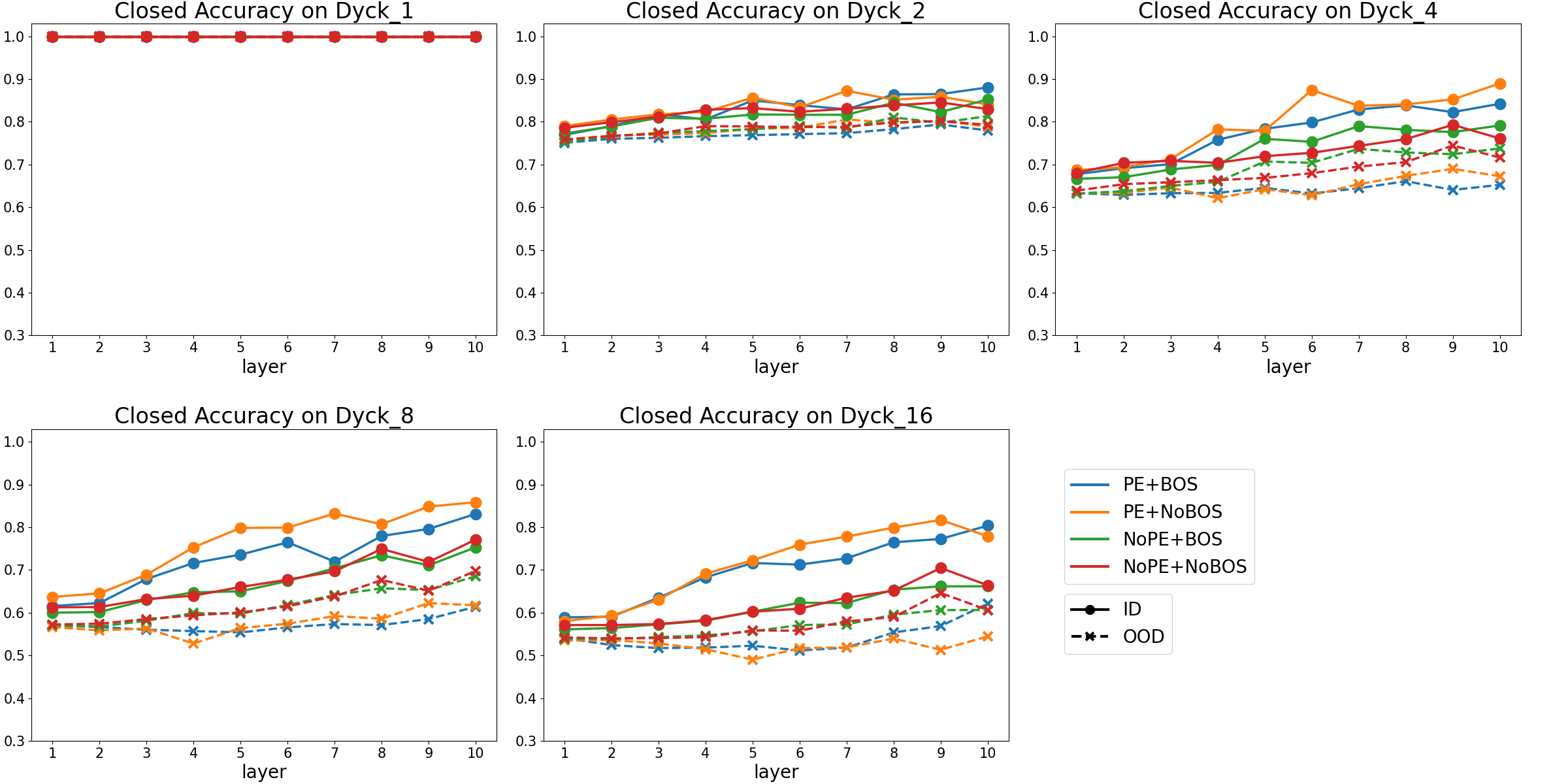}
  \caption{Test accuracy over $5$ runs of generating the correct closed brackets on $\texttt{Dyck}_k (k \in \{1 ,2 ,4 ,8, 16\})$. The solid lines represent the results for in-distribution data ($n \leq 700$), while the dashed lines represent the results for out-of-distribution data ($700 < n \leq 840$).}
  \label{fig:closed_acc_all}
\end{figure*}

\begin{figure*}[ht]
  \includegraphics[width=\linewidth]{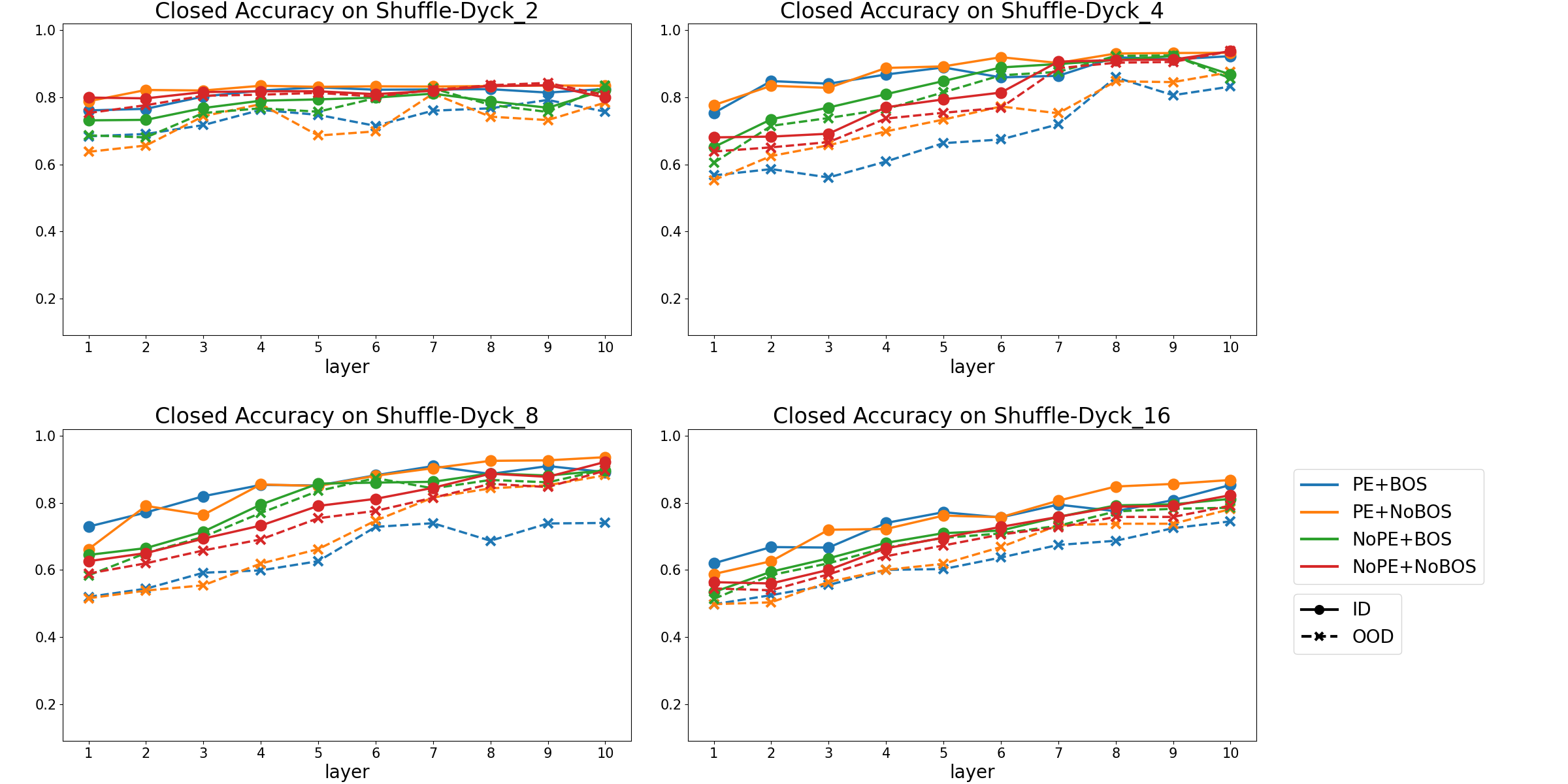}
  \caption{Test accuracy over $5$ runs of generating the correct closed brackets on $\texttt{Shuffle-Dyck}_k (k \in \{2 ,4 ,8, 16\})$. The solid lines represent the results for in-distribution data ($n \leq 700$), while the dashed lines represent the results for out-of-distribution data ($700 < n \leq 840$).}
  \label{fig:closed_acc_all_shuffle}
\end{figure*}

\begin{figure*}[ht]
  \includegraphics[width=\linewidth]{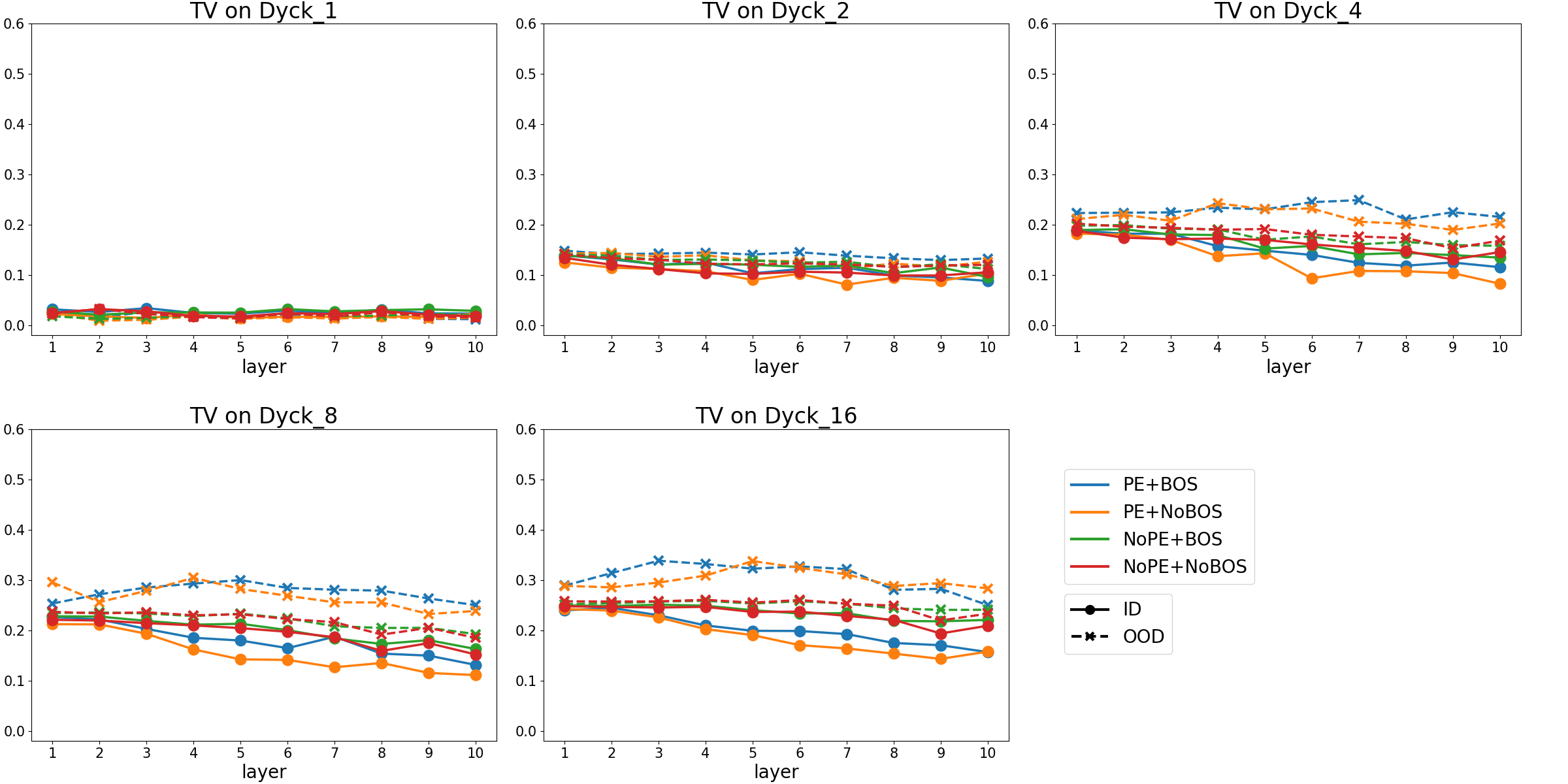}
  \caption{Average TV distance over $5$ runs on $\texttt{Dyck}_k (k \in \{1 ,2 ,4 ,8, 16\})$. The solid lines represent the results for in-distribution data ($n \leq 700$), while the dashed lines represent the results for out-of-distribution data ($700 < n \leq 840$).}
  \label{fig:tv_all}
\end{figure*}

\begin{figure*}[ht]
  \includegraphics[width=\linewidth]{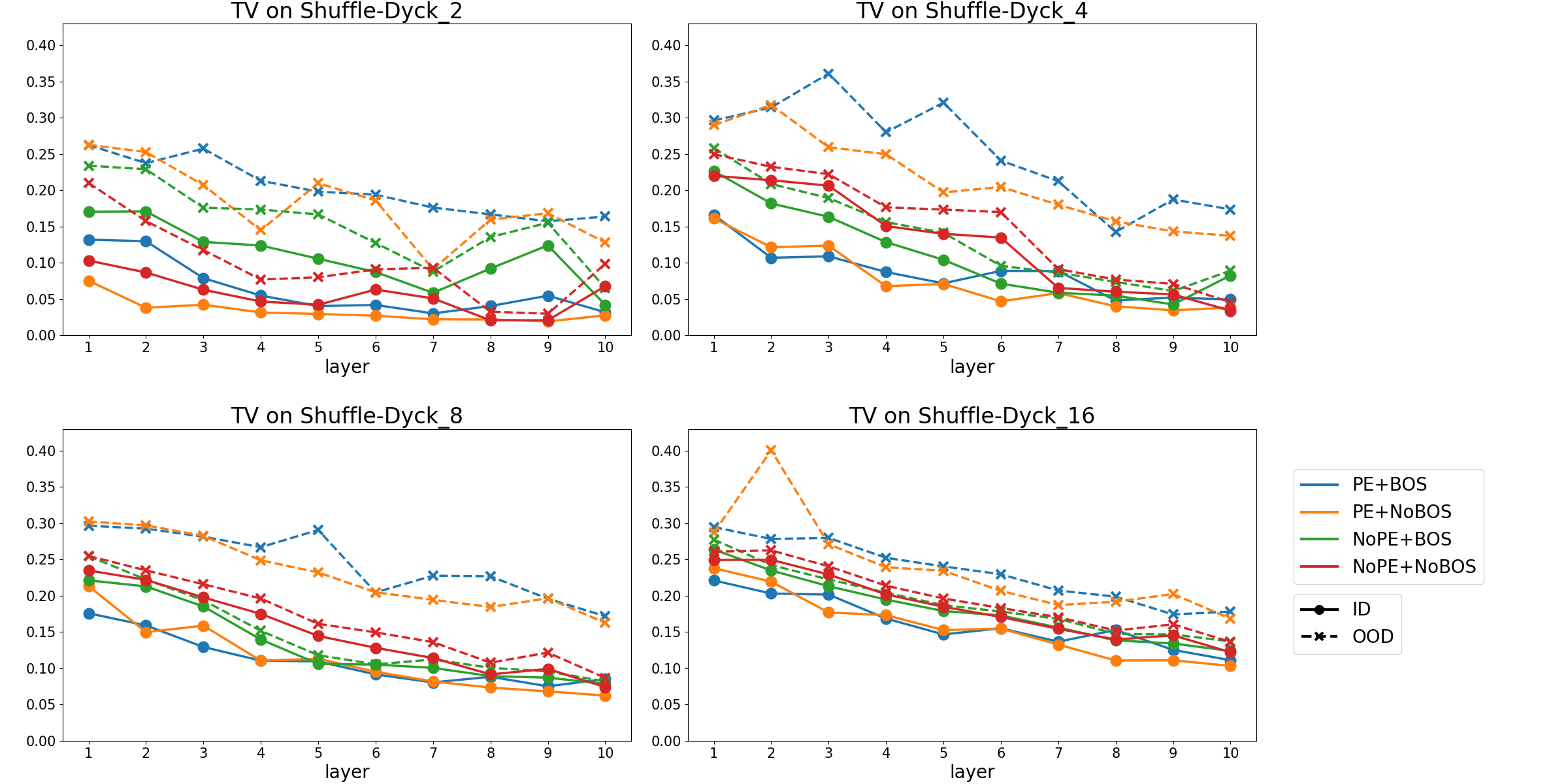}
  \caption{Average TV distance over $5$ runs on $\texttt{Shuffle-Dyck}_k (k \in \{2 ,4 ,8, 16\})$. The solid lines represent the results for in-distribution data ($n \leq 700$), while the dashed lines represent the results for out-of-distribution data ($700 < n \leq 840$).}
  \label{fig:tv_all_shuffle}
\end{figure*}

\subsection{Evaluation on natural language datasets}
In Section \ref{subsec:Experiment/Evaluation on natural language datasets}, we empirically investigated the effect of the layer normalization position on model performance using two natural language datasets, WikiText-103 \footnote{The WikiText-103 dataset is licensed under CC BY-SA 3.0, and we can freely use the content as long as we provide appropriate attribution. Our use of this dataset is consistent with the intended use. To the best of our knowledge, there is no specific step that checks whether personal information or offensive content is contained.} \citep{merity2016pointer}, a common English dataset that contains over $100$ million tokens extracted from the articles on Wikipedia, and OpenWebText \footnote{The OpenWebText is licensed under Creative Commons CC0 license, and we can freely use the content. Our use of this dataset is consistent with the intended use. To the best of our knowledge, there is no specific step that checks whether personal information or offensive content is contained.} \citep{Gokaslan2019OpenWeb}, a $30$GB of common English dataset that contains HTML pages whose URLs are shared on Reddit. Here, we provide detailed experimental settings and other experimental results.

We used the default split for WikiText-103: $103,227,021$ tokens from $28,475$ articles for training, $217,646$ tokens from $60$ articles for validation, and $245,569$ tokens from $60$ articles for test. In contrast, for OpenWebText, we used $0.5$\% of the total data for the validation set following the approach of \citet{fu2023hungry}, and similarly used $0.5$\% for the test set.

We implemented the architecture based on nanoGPT\footnote{nanoGPT(\url{https://github.com/karpathy/nanoGPT}) is licensed under MIT License, and we can freely use, copy, modify, publish, and distribute.}, which is a small version of GPT and incorporates the GPT-2 tokenizer in the tiktoken library \footnote{\url{https://github.com/openai/tiktoken}}. We add modifications to the position of the layer normalization. Regarding the hyperparameters, we used the default values except the values concerning the number of iterations: we modified the number of iterations to $20,000$, and accordingly, we also modified the number of iterations for learning-rate decay to $20,000$. Note that we adopt the QK normalization \citep{pmlr-v202-dehghani23a} to stabilize training. We use NVIDIA A100, and each experiment on WikiText-103 required approximately $40$ GPU hours, while each experiment on OpenWebText required approximately $100$ GPU hours. The values of the other hyperparameters are summarized in Table \ref{tab: experiment, hyperparameters}, and the decrease in training and validation loss is shown in Figure \ref{fig: app: experiments on natural language datasets}. 

\begin{table}
  \centering
  \begin{tabular}{lc}
    \hline
    \textbf{Hyperparameter} & \textbf{Value} \\
    \hline
    Model parameters & \\
    \hline
    Number of layers $L$ & $12$ \\
    Number of attention heads & $12$ \\
    Embedding dimension $d_\mathrm{model}$ & $768$ \\
    Use of bias terms & $\texttt{False}$ \\
    Window size & $1,024$ \\
    Activation function & gelu \\
    \hline
    Training parameters \\
    \hline
    Dropout rate & $0.0$ \\
    Batch size & $12$ \\
    Gradient accumulation steps & $40$ \\
    Learning rate & $6\mathrm{e}{-4}$ \\
    Minimum learning rate & $6\mathrm{e}{-5}$ \\
    Weight decay & $1\mathrm{e}{-1}$ \\
    Adam parameters $(\beta_1, \beta_2)$ & $(0.9, 0.95)$ \\
    Maximum iterations & $20,000$ \\
    Warmup iterations & $2,000$ \\
    Learning rate decay iterations & $20,000$ \\
    \hline
  \end{tabular}
  \caption{Hyperparameter configuration for experiments on natural language datasets.}
  \label{tab: experiment, hyperparameters}
\end{table}

\begin{figure*}[ht]
    \includegraphics[width=0.48\linewidth]{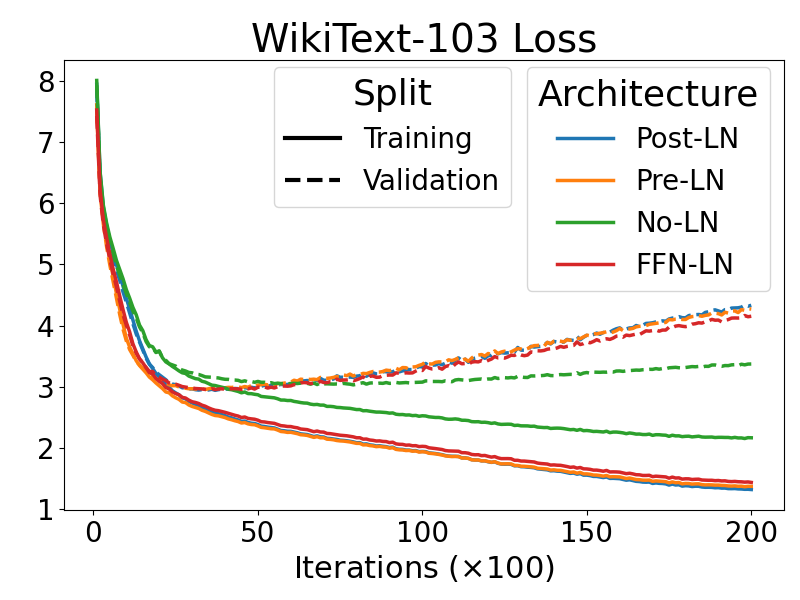} 
    \includegraphics[width=0.48\linewidth]{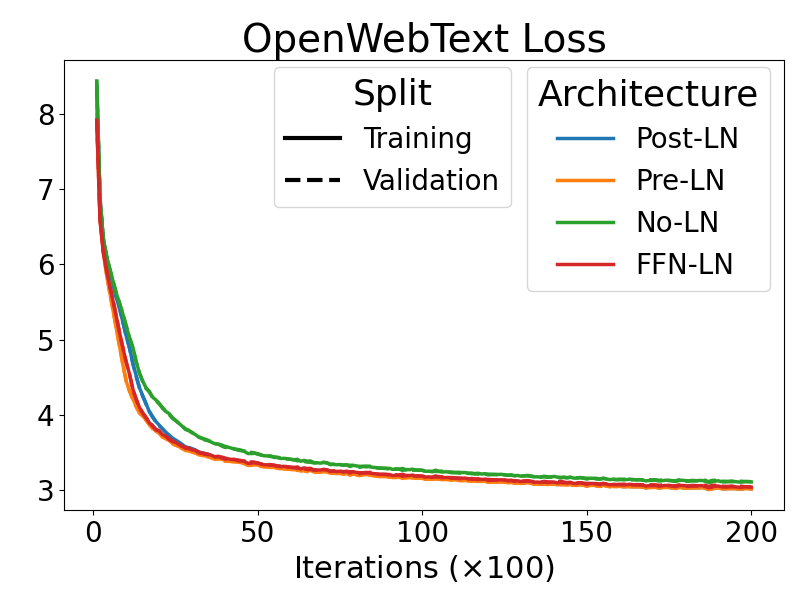}
    \caption{Results on natural language datasets. The transition of training and validation loss are reported.}
    \label{fig: app: experiments on natural language datasets}
\end{figure*}

\section{Further Discussion on Layer Normalization Position}\label{app: further discussion}

A common explanation for the reason why layer normalization leads to good performance is that layer normalization stabilizes the output distribution. Recently, some studies have investigated how the position of the layer normalization affects the model performance.

Most of the recent models such as Llama \citep{touvron2023llamaopenefficientfoundation}, Llama 2 \citep{touvron2023llama2openfoundation}, GPT-2 \citep{radford2019language}, and GPT-3 \citep{NEURIPS2020_1457c0d6} adopt $\texttt{Pre-LN}$, while the original Transformer architecurue \citep{NIPS2017_3f5ee243} and GPT \citep{radford2018improving} adopt $\texttt{Post-LN}$. There are some studies supporting that $\texttt{Pre-LN}$ outperforms $\texttt{Post-LN}$. However, there are also results indicating that $\texttt{Post-LN}$ can outperform $\texttt{Pre-LN}$ under specific conditions.

\citet{pmlr-v119-xiong20b} analyzed the layer normalization from the perspective of mean-field theory and showed that $\texttt{Pre-LN}$ provides more stable gradient after initialization compared to $\texttt{Post-LN}$. \citet{pmlr-v119-xiong20b} also empirically showed that $\texttt{Pre-LN}$, unlike $\texttt{Post-LN}$, does not require a warmup phase and significantly reduces training time.
In addition, \citet{wang-etal-2019-learning-deep} suggested that $\texttt{Post-LN}$ can have a higher risk of gradient vanishing and that in settings with a large number of layers, which are commonly seen in recent years, $\texttt{Pre-LN}$ outperforms $\texttt{Post-LN}$. In contrast, with respect to neural machine translation (NMT) task, \citet{nguyen-salazar-2019-transformers} showed that although $\texttt{Pre-LN}$ contributes to training stability and better performance in low-resource settings, $\texttt{Post-LN}$ shows superior performance in high-resource settings. Moreover, \citet{mao-etal-2023-exploring} demonstrated that for zero-shot machine translation, $\texttt{Post-LN}$ consistently outperforms $\texttt{Pre-LN}$. Furthermore, \citet{shleifer2021normformerimprovedtransformerpretraining} demonstrated that incorporating the layer normalization right before the second linear layer of the feed-forward network layer can effectively mitigate gradient explosion and vanishing, which are observed commonly in both $\texttt{Pre-LN}$ and $\texttt{Post-LN}$ setups.

Based on these results, we concluded that the optimal position of the layer normalization has not been established yet. Although the optimal position of the layer normalization remains unclear, in our experiments using the WikiText-103 and OpenWebText, we observed that the performance of $\texttt{Pre-LN}$, $\texttt{Post-LN}$, and $\texttt{FFN-LN}$ consistently outperformed $\texttt{No-LN}$. Therefore, we concluded that the architecture used in our proof $\texttt{FFN-LN}$ is competitive compared to other layer normalization positions, $\texttt{Pre-LN}$ and $\texttt{Post-LN}$.

\end{document}